%% file: paper.tex
\documentclass{article}
\include{header}

\usepackage{microtype}
\usepackage{graphicx}
\usepackage{booktabs} 

\usepackage[accepted]{icml2024}

\icmltitlerunning{\Momo{}: Momentum Models for Adaptive Learning Rates}

\usepackage[T1]{fontenc}

\begin{document}

\twocolumn[
\icmltitle{\Momo{}: Momentum Models for Adaptive Learning Rates}



\icmlsetsymbol{equal}{*}

\begin{icmlauthorlist}
\icmlauthor{Fabian Schaipp}{tum}
\icmlauthor{Ruben Ohana}{fi}
\icmlauthor{Michael Eickenberg}{fi}
\icmlauthor{Aaron Defazio}{meta}
\icmlauthor{Robert M.\ Gower}{fi}
\end{icmlauthorlist}

\icmlaffiliation{tum}{Department of Mathematics, Technical University of Munich, Munich}
\icmlaffiliation{fi}{Flatiron Institute, CCM, New York}
\icmlaffiliation{meta}{Meta AI, Fundamental AI Research (FAIR) team, New York}

\icmlcorrespondingauthor{Fabian Schaipp}{fabian.schaipp@tum.de}

\icmlkeywords{Machine Learning, ICML}

\vskip 0.3in
]



\printAffiliationsAndNotice{}  

\begin{abstract}
Training a modern machine learning architecture on a new task requires extensive learning-rate tuning, which comes at a high computational cost. Here we develop new Polyak-type adaptive learning rates that can be used on top of any momentum method, and require less tuning to perform well.
We first develop \Momo{}, a \textbf{Mo}mentum \textbf{Mo}del based adaptive learning rate for \SGDM{} (stochastic gradient descent with momentum). \Momo{} uses momentum estimates of the  losses and gradients sampled at each iteration to build a model of the loss function.
Our model makes use of any known lower bound of the loss function by using \emph{truncation}, e.g. most losses are lower-bounded by zero. 
The model is then approximately minimized at each iteration to compute the next step. We show how \Momo{} can be used in combination with any momentum-based method, and showcase this by developing \MomoAdam{}, which is \Adam{} with our new model-based adaptive learning rate.
We show that \Momo{} attains a $\mathcal{O}(1/\sqrt{K})$ convergence rate for convex problems with interpolation, needing knowledge of no problem-specific quantities other than the optimal value.
Additionally, for losses with unknown lower bounds, we develop on-the-fly estimates of a lower bound, that are incorporated in our model. 
We show that  \Momo{} and \MomoAdam{} improve over \SGDM{} and \Adam{} in terms of robustness to hyperparameter tuning for training image classifiers on \texttt{MNIST}, \texttt{CIFAR}, and \texttt{Imagenet}, for recommender systems on  \texttt{Criteo}, for a transformer model on the translation task \texttt{IWSLT14}, and for a diffusion model.
\end{abstract}

\section{Introduction}
 Training a modern production-grade large neural network can cost over $1$ million dollars in compute. For instance,  the cost for the \emph{Text-to-Text Transfer Transformer} T5-model \citep{raffel2020exploring} is estimated to be more than $1.3$ million dollars for a single run~\citep{sharir2020cost}. What makes training models highly expensive is that multiple runs are needed to tune the hyperparameters, with arguably the most important parameter being the learning rate. 
 Indeed, finding a good learning-rate schedule plays a disproportionately large role in the resulting test error of the model, with one extensive study showing that it is of similar importance as the choice of optimizer~\citep{Schmidt2021}.
 
Here, we develop  adaptive learning rates that can be used together with any momentum-based method. To showcase our method, we apply our learning rates to \SGDM{} (stochastic gradient descent with momentum) and to \Adam{}~\citep{Kingma2015}, which gives the \Momo{} and \MomoAdam{} method, respectively.
We make use of model-based stochastic optimization~\citep{Asi2019,Davis2019,Chadha2021}, and leverage that loss functions are bounded below (typically by zero) to derive our new \Momo{} (\textbf{Mo}mentum \textbf{Mo}del) adaptive learning rate.
An implementation of \Momo{} is available in \href{https://github.com/fabian-sp/MoMo}{\texttt{Pytorch}} and \href{https://github.com/google-deepmind/optax}{\texttt{optax}}.

\subsection{The Model-Based Approach}
Consider the problem
\begin{align} \label{eq:main}
\min_{x \in \R^d} \; f(x), \quad f(x):=  \EE{s \sim \cD}{ f(x,s)} ,
\end{align}
where  $f(x, s)$ is a loss function, $s$ is an input (mini-batch of data), and $x$ are the learnable parameters of a model. We assume throughout that $f(x,s) \geq 0$, which is true for most loss functions\footnote{
We choose zero as a lower bound for simplicity, but any constant lower bound could be handled.}. We also assume that $f(\cdot, s)$ is continuously differentiable for all $s\in \cD$, that there exists a solution $x^*$ to~\eqref{eq:main} and denote the optimal value by $f^*\eqdef f(x^*) \in \R$. 

In our main algorithms \Momo{} and \MomoAdam{} (\cref{alg:momo,alg:momo-adam}), we present adaptive learning rates\footnote{Here the term \emph{adaptivity} refers to a scalar learning rate that changes from one iteration to the next by using easy-to-compute quantities. 
}
for \SGDM{} and \Adam{}, respectively. 
To derive  \Momo{} and \MomoAdam{}, we use the model-based viewpoint, which is often motivated by the stochastic proximal point (\texttt{SPP}) \citep{Asi2019,Davis2019} method. At each iteration, \texttt{SPP} samples $s_k \sim \cD,$ then  trades-off minimizing $f(x,s_k)$ with not moving too far from the current iterate $x^k$.
Given a learning rate $\alpha_k>0$, this can be written as
\begin{align} \label{eq:prox-stoch}
x^{k+1} =  \argmin{x\in\R^d} f(x, s_k) +\tfrac{1}{2\alpha_k}\norm{x -x^k}^2.
\end{align}
Since this problem needs to be solved at every iteration, it needs to be fast to compute.
However, in
general~\eqref{eq:prox-stoch} is difficult to solve because  $f(x, s_k)$ can be a highly nonlinear function. 
Model-based methods
replace $f(x,s_k)$ by a simple model $m_k(x)$ of the function~\citep{Asi2019,Davis2019}, and update according to
\begin{equation}\label{eq:proxmk}
 x^{k+1} \; = \;  \argmin{x\in\R^d} m_k(x) +\tfrac{1}{2\alpha_k}\norm{x -x^k}^2.
\end{equation}
\texttt{SGD} can be formulated as a model-based method by choosing the model to be the
linearization of $f(x,s_k)$  around $x^k,$ that is
\begin{equation} \label{eq:SGDmodel} 
    m_k(x) \; = \; f(x^k,s_k) + \dotprod{\nabla f(x^k,s_k), x-x^k }. 
\end{equation} 
Using the above $m_k(x) $ in~\eqref{eq:proxmk} gives the \SGD{} 
update
$$x^{k + 1} = x^k - \alpha_k\nabla f(x^k, s_k),$$ 
see \citep{Robbins1951,Asi2019}. 

Our main insight for developing 
the \Momo{} methods is that we should build a model directly for $f(x)$, and not $f(x,s_k)$,  since our objective is to minimize $f(x)$. To this end, we develop a model $m_k(x)$ that is a good approximation of $f(x)$ when $x$ is close to $x^k$, and such that~\eqref{eq:proxmk} 
has a closed-form solution.
We use momentum estimates of past gradients and loss values to build a 
model $f(x)$. 
Finally, since the loss function is positive, we  impose that our model be positive. 

\subsection{Background and Contributions}

\paragraph{Momentum and model-based methods.}

The update formula of many stochastic methods such as \SGD{} can be interpreted by taking a proximal step with respect to a model of the objective function~\citep{Asi2019,Davis2019}. 
Independently of this, (heavy-ball) momentum \citep{Polyak1964,Sebbouh2021} is incorporated into many methods in order to boost performance.

\emph{Contributions.} Here we give a new interpretation of momentum, namely that it can be motivated as a model of the objective function
$f(x)$ by averaging sampled loss functions. This allows us to naturally combine momentum with other model-based techniques. 

\paragraph{Lower bounds and truncated models.} 
One of the main advantages of the model-based viewpoint~\citep{Asi2019,Davis2019} is that it illustrates how to use knowledge of a lower bound of the function via truncation. Methods using this truncated model are often easier to tune~\citep{meng2022a,Schaipp2023}.

\emph{Contributions.} By combining the model-based viewpoint of momentum with a truncated model we arrive at our new \Momo{} method. 
Since we are interested in loss functions, we can use zero as a lower bound estimate in many learning tasks. However, for some tasks such as training transformers, the minimal loss is often non-zero. 
If the non-zero lower bound is known, such as the entropy of language used in scaling laws~\cite{scalinglaws}, we can straightforwardly incorporate it into our model.
For unknown lower bound values 
we also develop new online estimates of a lower bound in \cref{sec:lower}. Our estimates can be applied to any stochastic momentum-based
method%
, and thus may be of independent interest. Our main influence for this technique was D-adaptation~\citep{defazio2023learningratefree} which develops an online estimate of the distance to the solution.
\paragraph{Adaptive methods.}
In practice, tuning the learning rate is intricate and computationally expensive.  \Adam{} ~\citep{Kingma2015} (or \texttt{AdamW}~\citep{Loshchilov2019}) is often easier to tune and now being used routinely to train DNNs across a variety of tasks.  This and the success of \Adam{} have incentivised the development of many new learning-rate techniques, including approaches based on coin-betting ~\citep{coin-betting}, variants of \texttt{AdaGrad}~\citep{Duchi2011,defazio2023learningratefree}, and stochastic line search~\citep{Vaswani2019}. Recent work also combines parameter-free coin betting methods with truncated models~\citep{Chen2022}.

\emph{Contributions.}
Our new adaptive learning rate can be combined with any momentum-based method, and allows for a preconditioner to be used.
For example, \Adam{} is a momentum method that makes use of a preconditioner. 
By using this viewpoint, together with a lower bound, we derive \MomoAdam{}, a variant of \Adam{} that uses our adaptive learning rates.

\paragraph{Polyak step sizes.} 
For convex, non-smooth optimization, Polyak proposed an adaptive step size using the current objective function value $f(x^k)$ and the optimal value $f^*$ \citep{Polyak1987}.
Recently,  the Polyak step size has been adapted to the stochastic setting~\citep{Berrada2019, Gower2021, Loizou2021, Orvieto2022}. 
For example, for $c,\gamma_b > 0$ \citep{Loizou2021} proposed 
\begin{align*}
    x^{k+1} = x^k - \min\Big\{\gamma_b, \tfrac{f(x^k, s_k)- \inf_z f(z, s_k)}{c\|\nabla f(x^k, s_k)\|^2}\Big\}\nabla f(x^k, s_k),
\end{align*}
called the \spsmax{} method. 
The stochastic Polyak step size is closely related to stochastic model-based proximal point methods as well as stochastic bundle methods \citep{Asi2019,Paren2022,Schaipp2023}.

\emph{Contributions.} Our proposed method \Momo{} can be seen as an extension of the Polyak step size that also incorporates momentum. This follows from the viewpoint of the Polyak step size~\citep{Berrada2019,Paren2022,Schaipp2023} as a truncated model-based method. In particular \Momo{} with no momentum is equal to \spsmax{}.

\paragraph{Numerical findings.}

We find that \Momo{} consistently improves the sensitivity with respect to hyperparameter choice as compared to \SGDM{}.
The same is true for \MomoAdam{} compared to \Adam{}. 
Our experiments cover image classification tasks on \texttt{MNIST}, \texttt{CIFAR10}, \texttt{CIFAR100} and \texttt{Imagenet}, a recommender system for the \texttt{Criteo} dataset, an encoder-decoder transformer for the translation task \texttt{IWSLT14}, and a diffusion model.

Furthermore, we find that the adaptive learning rate of \Momo{}(-\Adam{}) for some tasks automatically performs a warm-up at the beginning of training and a decay in later iterations, two techniques often used  to improve training \citep{Sun2020}.

\section{Model-Based Momentum Methods}
Let us recall the \texttt{SGD} model in~\eqref{eq:SGDmodel} which has two issues: first, it approximates a single stochastic function $f(x,s_k)$, as opposed to the full loss $f(x)$.
Second, 
this model can be negative even though our loss function is always positive.  
Here, we develop a model directly for $f(x)$, and not $f(x,s_k),$ 
which also takes into account lower bounds on the function value.
\subsection{Model-Based Viewpoint of Momentum}
Suppose we have sampled inputs $s_1,\ldots, s_k$ and past iterates $x^1, \ldots, x^k$. We can use these samples to build a  model of  $f(x)$ by averaging past function evaluations as follows
\begin{align}\label{eq:approxaverage}
f(x) = \EE{s\sim \cD}{f(x,s)}\; \approx\; \tfrac{1}{\rho_k}\sum_{j=1}^k \rho_{j,k} f(x,s_j),
\end{align}
where $\rho_{j,k}\geq 0$ and $\rho_k:=\sum_{j=1}^k\rho_{j,k}$. Thus, the $\rho_k^{-1}\rho_{j,k}$ are a discrete probability 
mass function over the previous samples. 
The issue with~\eqref{eq:approxaverage} is that it is expensive to evaluate $f(x,s_j)$ for $j=1,\ldots, k$, which we would need to do at every iteration.
Instead,
we approximate each $f(x,s_j)$ 
by linearizing $f(x,s_j)$ around $x^j$, the point it was last evaluated, that is for $j=1,\ldots,k$\\[-0.4cm]
\begin{align}\label{eq:linjapproac}
f(x,s_j)  \approx   f(x^j,s_j) + \dotprod{\nabla f(x^j,s_j), x-x^j }.
\end{align}
Using~\eqref{eq:approxaverage} and the linear approximations in~\eqref{eq:linjapproac} we can approximate $f(x)$ with the model $m_k(x)$ given by\\[-0.45cm]
\begin{align}\label{eq:approxaverage-linearized}
\begin{split}
    &\hspace{-1ex}m_k^{\text{avg}}(x)= \\
    &\hspace{1ex}\tfrac{1}{\rho_k} \sum_{j=1}^k \rho_{j,k} \big(f(x^j,s_j) +  \dotprod{\nabla f(x^j,s_j), x-x^j } \big).
\end{split}
\end{align}
If we plug in the above model $m_k^{\text{avg}}(x)$ into~\eqref{eq:proxmk}, then the resulting update is given by \\[-0.3cm]
\begin{equation} 
    x^{k+1} \; = x^k - \tfrac{\alpha_k}{\rho_k} d_k, \quad 
    d_k  \eqdef \sum_{j=1}^k\rho_{j,k} \nabla f(x^j,s_j).
\end{equation}
By appropriately choosing the $\rho_{j,k}$ (details in Section~\ref{sec:averaging}), the above method is equivalent to \SGDM{}.
This gives a new viewpoint of (heavy-ball) momentum. Next we incorporate a lower bound into this model so that, much like the loss function, it cannot become negative.

\subsection{Deriving \Momo{}}
Since we know the loss is lower-bounded by zero, we will also impose a lower bound on the model~\eqref{eq:approxaverage-linearized}. Though we could use zero, we will use an estimate $\lb{k} \geq 0 $ of the lower bound to allow for cases where $f(x^*)$ may be  far from zero.
Imposing a lower bound of $\lb{k}$  gives the following model 
\begin{align}\label{eq:posmodel} \textstyle
f(x)  \approx 
\max\Big\{m_k^{\text{avg}}(x), \lb{k}  \Big\} 
=: m_k(x).
\end{align}

For overparametrized
machine-learning
models the minimum value $f(x^*)$ is often close to zero~\citep{Ma2018,Gower2021}.
Thus,
choosing $\lb{k}=0$ in every iteration will work well (as we verify later in our experiments). For tasks where $\lb{k}=0$ is too loose of a bound, in~\cref{sec:lower} we develop an online estimate for $\lb{k}$ based on available information.
Using the model \eqref{eq:posmodel}, we can now define the proximal update
\begin{align}
\label{eq:modelbasedVI}
    x^{k+1} = \argmin{y \in \R^d} m_k(y) + \tfrac{1}{2\alpha_k}\|y-x^k\|^2.
\end{align}
Because $m_k(y)$ is a simple piece-wise linear function, the update \eqref{eq:modelbasedVI} has a closed form solution, as we show in the following lemma (proof in \cref{sec:proof-update-lemma-momo}).
\begin{restatable}{lemma}{lemupdate}[{\bf \Momo{} update}]\label{lem:update}
Let 
\begin{align}\label{eq:barf-d-gamma}
    \begin{split}
    &\hspace{-11ex}d_k  := \sum_{j=1}^k\rho_{j,k} \nabla f(x^j,s_j), \\
    &\hspace{-11ex}\bar{f}_k :=  \sum_{j=1}^k \rho_{j,k} f(x^j, s_j),  \\
    &\hspace{-11ex}\gamma_k := \sum_{j=1}^k \rho_{j,k} \iprod{\nabla f(x^j,s_j)}{x^j}. 
    \end{split}
\end{align}
Using model \eqref{eq:posmodel}, the closed form solution to~\eqref{eq:modelbasedVI} is 
\begin{align}\label{eq:momo-update}
        \tau_k &:= \min \Big\{\frac{\alpha_k}{\rho_k}, \frac{\big(\bar{f}_k  +  \dotprod{d_k,x^k} - \gamma_k - \rho_k \lb{k} \big)_+}{\|d_k\|^2} \Big\}, \nonumber\\
        x^{k+1} &= x^k - \tau_k d_k.
\end{align}
\end{restatable}

%

%
Finally, it remains to  select the averaging coefficients $\rho_{j,k}$. Here we will  use an exponentially weighted average that places more weight on recent samples. Aside from working  well in practice on countless real-world examples, exponential averaging can be motivated through the model-based interpretation. Recent iterates will most likely have gradients and loss values that are closer to the current iterate $x^k$. Thus, we place more weight on recent iterates,
i.e.\ $\rho_{j,k}$ is big for $j$ close to $k$. We give two options for exponentially weighted averaging next. 
\subsection{The Coefficients $\rho_{j,k}$: To bias or not to bias} \label{sec:averaging}
We now choose $\rho_{j,k}\geq 0$ such that we can update $\bar f_k$, $d_k$ and $\gamma_k$ in \eqref{eq:barf-d-gamma} on the fly, storing only two scalars and one vector, resulting in the same iteration complexity as \SGDM{}.

{\bf Exponentially Weighted Average.} Let $\beta\in [0,1)$. Starting with $\rho_{1,1}=1$, and for $k \geq 2$ define
\begin{align*}
    \rho_{j,k} = \begin{cases}\beta \rho_{j,k-1},\quad  &j\leq k-1, \\ 1-\beta,\quad  &j=k.\end{cases}
\end{align*} 
Then, $\rho_k=\sum_{j=1}^k \rho_{j,k} = 1$ for all $k\in \N$ and the quantities in \eqref{eq:barf-d-gamma} are exponentially weighted averages, see \cref{lem:averaging-coefficients}.
As a consequence, we can update $\bar f_k$, $d_k$ and $\gamma_k$ on the fly as given in lines~\ref{ln:fbar-and-gamma-up}--\ref{ln:dup} in \cref{alg:momo}. Combining update \eqref{eq:momo-update} and   $\rho_k=1$, we obtain \cref{alg:momo}, which we call \Momo{}.

%
\begin{algorithm}[H]
    \caption{\texttt{MoMo}:\hspace{-0.5mm} Model-based Momentum method. }
    \label{alg:momo}
    \begin{algorithmic}[1]
    \STATE {\bfseries Default settings:} \hspace{-0.5mm} $\alpha_k =1$, $\beta =0.9$, $(\lb{k})_{k\in \N} =0.$
    \STATE {\bfseries Input:} $x^1 \in \R^d$, $\beta \in [0,\;1)$, $\alpha_k>0, (\lb{k})_{k\in \N} \subset\R$ 
   \STATE {\bf Init:} $\bar{f}_0 = f(x^1,s_1), d_0 =\nabla f(x^1,s_1)$,$\gamma_0 =\dotprod{d_0, x^1}$
    \FOR{$k=1$ {\bfseries to} $K-1$}
    \STATE $ \displaystyle \bar{f}_k  \; = (1-\beta) f(x^k, s_k) +  \beta\bar{f}_{k-1}$  \label{ln:fbar-and-gamma-up} 
    \STATE $\displaystyle \gamma_k =(1-\beta)\dotprod{\nabla f(x^k,s_k), x^k} + \beta\gamma_{k-1} $
    \STATE $ \displaystyle  d_k \;=   (1-\beta) \nabla f(x^k,s_k)  + \beta d_{k-1} $  \label{ln:dup}
    \STATE $\displaystyle h_k \; = \bar{f}_k + \dotprod{d_k, x^k} - \gamma_k $
    \STATE $ \displaystyle  x^{k+1} = x^k - \min \Big\{\alpha_k, \tfrac{(h_k- \lb{k}   )_+}{\|d_k\|^2} \Big\} d_k$  \label{ln:xup}
    \ENDFOR
    \end{algorithmic}
\end{algorithm}
\begin{remark}
   {The \emph{adaptive learning rate} $\tau_k$ in \eqref{eq:momo-update}  determines the size of the step and can vary in each iteration even if $\alpha_k$ is constant. The \emph{(user-specified) learning rate} $\alpha_k$ caps the adaptive learning rate.}
\end{remark}
%
\begin{remark}[{ Complexity}]{
\Momo{} has the same order iteration complexity and memory footprint as \SGDM{}. \Momo{}  stores two additional scalars $\gamma_k$ and $\bar{f}_k$, as compared to \SGDM{}, and has two additional $\cO(d)$ inner products lines 6 and 8, and one $\cO(d)$ vector norm on line 9.}
\end{remark}

For $\beta=0$ (no momentum), we have $\gamma_k=\iprod{\nabla f(x^k,s_k)}{x^k}=\iprod{d_k}{x^k}$ and $\bar{f}_k = f(x^k,s_k).$ Consequently $h_k =f(x^k,s_k)$, and in this special case, \Momo{} is equivalent\footnote{This equivalence requires setting $\gamma_b \leftarrow \alpha_k,~c\leftarrow 1$, and assuming $\lb{k} = \inf_z f(z,s_k)$.} to \spsmax{}. 

\cref{fig:momo-model} shows how the \Momo{} model~\eqref{eq:posmodel}  approximates a convex function (left) and a non-convex function (right). 
The \Momo{} update $x_{\Momo{}}^{k+1}$ in~\cref{fig:momo-model} attains a smaller loss for both examples, as compared to the \SGDM{} update.\\[-0.68cm]

\paragraph{Averaging with Bias Correction.} \label{par:bias-correction}
Alternatively, we can choose $\rho_{j,k} = (1-\beta)\beta^{k-j}$ for $j=1,\ldots, k$, as it is used in \Adam{}~\citep{Kingma2015}. This gives $\rho_k=1-\beta^k \neq 1$. 
We discuss this choice for \Momo{} in \cref{sec:appendix-averaging-coeffs} and for consistencies sake will use it later for \MomoAdam{}. 
\begin{figure}[t]
\centering
  \begin{subfigure}[b]{0.9\columnwidth}
    \centering
    \includegraphics[width=\linewidth]{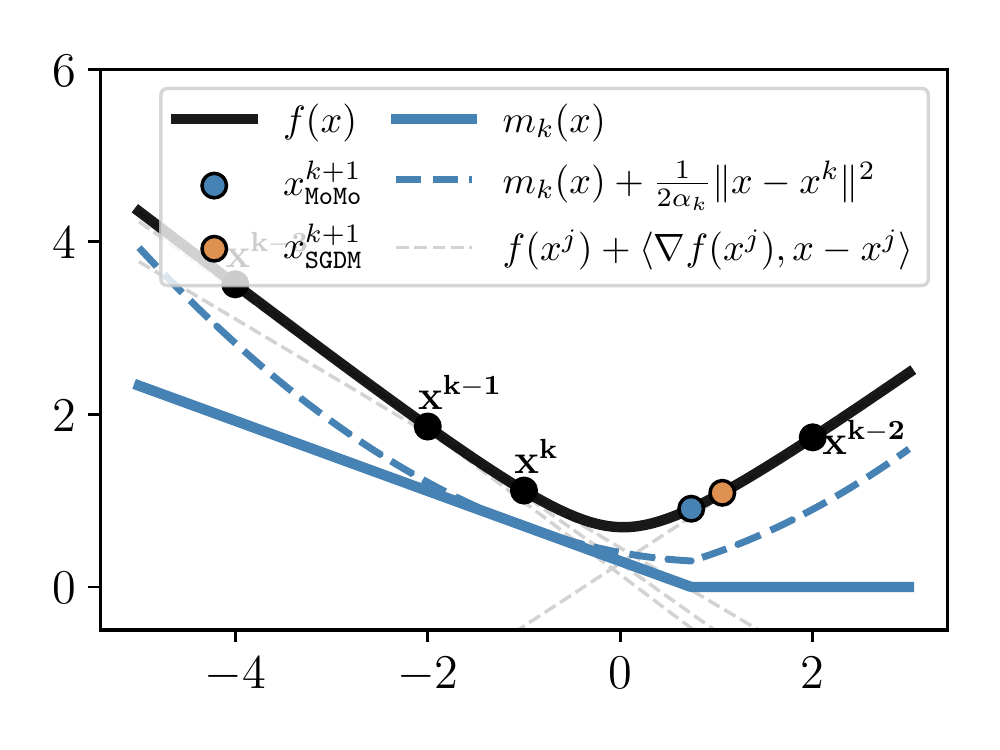}
     \caption{Convex loss}    
  \end{subfigure}
  \begin{subfigure}[b]{0.9\columnwidth}
    \centering
    \includegraphics[width=\linewidth]{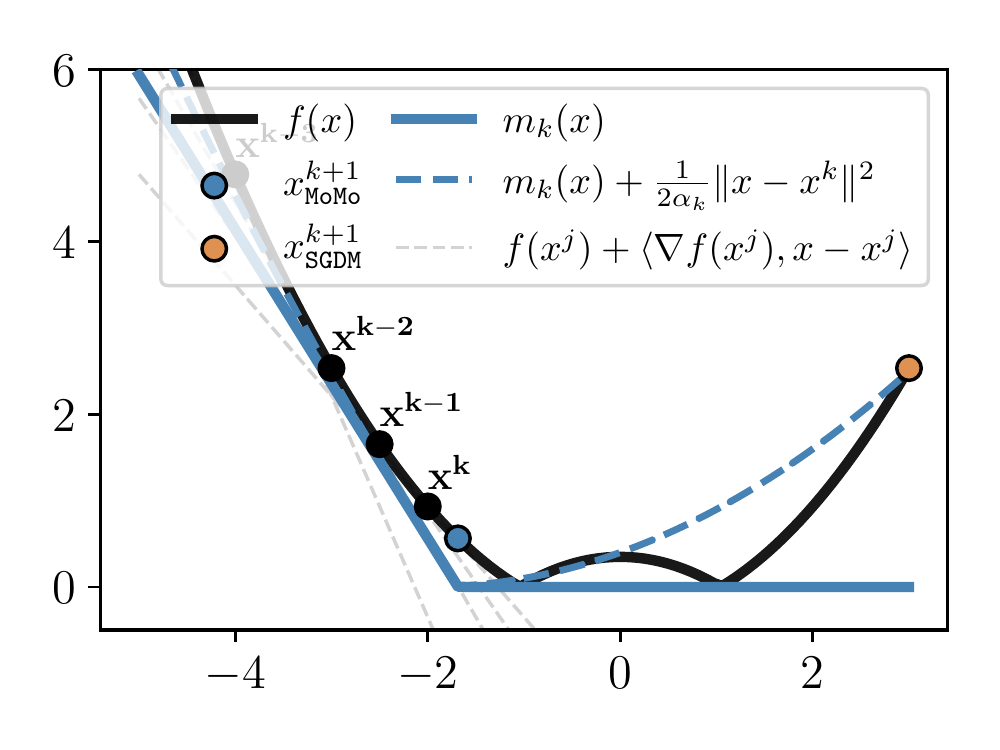}
    \caption{Non-convex loss}
  \end{subfigure}
  \caption{{\small Illustration of the \texttt{MoMo} model (blue curves) for two different loss functions with $\alpha_k =5$. Due to truncation, the new iterate of \Momo{} (blue point) is closer to the minimum than \SGDM{} (orange point).   The right plot shows how \texttt{MoMo} takes a small step when gradients are steep, whereas \SGDM{} takes a large step and ends up far from the solution.}}
  \label{fig:momo-model}
\end{figure}
%
%

\section{Weight Decay and Preconditioning}
Often weight decay is used in order to improve generalization \citep{Zhang2019}. Weight decay is equivalent to adding a squared $\ell_2$-regularization to the objective function \citep{Krogh1991}, in other words,  instead of~\eqref{eq:main} we solve
   $\min_{x\in\R^d} f(x)  + \tfrac{\color{color2}{\lambda}}{2}\norm{x}^2$,
where $f(x)$ is again the loss function.  
To include weight decay, we build a model $m_k$ for the loss $f$ and keep the $\ell_2$-regularization outside of the model. That is equation~\eqref{eq:modelbasedVI} is modified to
\begin{align}\label{eq:momow-argmin}
    x^{k + 1} = \argmin{\scriptsize y\in\mathbb R^d} m_k(y) + \tfrac{\color{color2}{\lambda}}{2}\|y\|^2 + \tfrac{1}{2\alpha_k}\|y - x^k\|^2 .
\end{align}
Finally, 
the Euclidean norm may often not be best suited. Many popular methods such as \texttt{AdaGrad} or \Adam{} can be interpreted as using a preconditioner for the proximal step. Hence, we allow for an arbitrary norm defined by a symmetric, positive definite matrix ${\color{color1}\mD_k} \in \R^{d\times d}$, i.e.\ $\norm{x}_{{\color{color1}\mD_k}}^2 := \iprod{{\color{color1}\mD_k} x}{x}.$
We can now use ${\color{color1}\mD_k}$ to change the metric within our proximal method, by updating 
\begin{align}\label{eq:adaptive-model-update-wd}
    \begin{split}
     &\hspace{-1ex}x^{k+1} =
    \argmin{y \in \R^d} m_k(y)  + \tfrac{\color{color2}{\lambda}}{2}\|y\|_{{\color{color1}\mD_k}}^2 + \tfrac{1}{2\alpha_k}\|y-x^k\|_{{\color{color1}\mD_k}}^2.
     \end{split}
\end{align}
This update~\eqref{eq:adaptive-model-update-wd} enjoys the following closed form solution  (proof in \cref{sec:proof-update-lemma-general-momo}).
\vspace{2ex}
\begin{restatable}{lemma}{lemmomogeneralupdate}
\label{lem:momo-general-update}
Using model \eqref{eq:posmodel}, the closed form solution to~\eqref{eq:adaptive-model-update-wd}
is given by 
    \begin{align}
    &\hspace{-2.0ex}\tau_k =\min \Big\{\tfrac{\alpha_k}{\rho_k}, \tfrac{\big((1+\alpha_k {\color{color2}{\lambda}})(\bar f_k-\rho_k \lb{k} - \gamma_k) + \iprod{d_k}{x^k} \big)_+}{\|d_k\|_{{\color{color1}\mD_k^{-1}}}^2} \Big\}, \label{eq:tauk-general}\\
    &x^{k+1} = \tfrac{1}{1+\alpha_k \color{color2}{\lambda}} \Big[x^k - \tau_k {\color{color1}\mD_k^{-1}}d_k \Big].      
    \label{eq:momo-general-upd}
    \end{align}
\end{restatable}
\cref{lem:momo-general-update} shows how to incorporate weight decay in \Momo{}: we replace line \ref{ln:xup} in \cref{alg:momo} by \eqref{eq:tauk-general}-\eqref{eq:momo-general-upd} with ${\color{color1}\mD_k}=\mathbf{Id}$ and $\rho_k=1$. If $\beta=0$ (no momentum) then \Momo{} with weight decay recovers \texttt{ProxSPS}, the proximal version of the stochastic Polyak step size~\citep{Schaipp2023}.\\[-0.65cm]

\paragraph{Deriving \MomoAdam{}.}
Using \cref{lem:momo-general-update} we can  obtain an \Adam{}-version of \Momo{} by defining ${\color{color1}\mD_k}$ as the diagonal preconditioner of \Adam{}. 
Let $\mathbf{1}_d$ be the $d$-dimensional vector of ones, $\mathrm{Diag}(v)$  a diagonal matrix with diagonal entries $v\in\R^d$, and   $\odot$ and $\sqrt{v}$ the elementwise multiplication and square-root operations.
Denoting $g_k = \nabla f(x^k,s_k)$, for \Adam{} we choose
\begin{align*}
    v_k &= (1-\beta_2)v_{k-1} + \beta_2 (g_k \odot g_k), \\
    {\color{color1}\mD_k} &= \mathrm{Diag}(\epsilon \mathbf{1}_d + \sqrt{v_k/(1-\beta_2)^k}),
\end{align*}
where $\beta_2 \in [0,1),~\epsilon > 0$.
Using this preconditioner with \cref{lem:momo-general-update} gives \cref{alg:momo-adam}, called \MomoAdam{}. Note that here we choose $\rho_{j,k} = (1-\beta)\beta^{k-j}$ (cf.\  \cref{sec:averaging}) which gives the standard averaging scheme of \Adam{}.
We focus on \Momo{} versions of \SGDM{} and \Adam{} because these are the two most widely used methods. However, from \cref{lem:momo-general-update} we could easily obtain a \Momo{}-version any other preconditioned momentum method (e.g.\ \texttt{Adabelief} \cite{Zhuang2020}).
\begin{algorithm}[t]
    \caption{\MomoAdam{}: Adaptive learning rates for \Adam{} }
    \label{alg:momo-adam}
    \begin{algorithmic}[1]
    \STATE {\bfseries Default settings:} $\alpha_k=10^{-2}$, $(\beta_1, \beta_2) = (0.9,0.999)$, $\epsilon=10^{-8}$
    
    \STATE {\bfseries Input:} $x^1 \in \R^d$, $\beta_1,\beta_2 \in [0,\;1)$, $\epsilon > 0$,  $\alpha_k>0$, ${\color{color2}{\lambda}} \geq 0$, and $(\lb{k})_{k\in \N} \subset\R$.
    \STATE {\bfseries Initialize:}$\bar{f}_0 = 0, d_0 =0$, $\gamma_0 =0 $, and $v_0 = 0$.
    
    \FOR{$k=1$ {\bfseries to} $K-1$}
    \STATE $ \displaystyle g_k  \; = \nabla f(x^k, s_k); \quad d_k \;=   (1-\beta_1) g_k  + \beta_1 d_{k-1} $
    
    \STATE $ \displaystyle  v_k \;= \beta_2 v_{k-1} + (1-\beta_2) (g_k \odot g_k)$  
    
     \STATE $  {\color{color1}\mD_k}\; = \mathrm{Diag}\big(\epsilon \mathbf{1}_d + \sqrt{\left. v_k \right/ (1-\beta_2^{k})}\big)  $  

    \STATE $ \displaystyle \bar{f}_k  \; = (1-\beta_1) f(x^k, s_k) +  \beta_1\bar{f}_{k-1}$
      
    \STATE $ \displaystyle\gamma_k =(1-\beta_1)\dotprod{g_k, x^k} + \beta_1\gamma_{k-1}  $ 
     
    \STATE $\displaystyle
    h_k = \tfrac{ \big((1+{\color{color2}\lambda} \alpha_k)(\bar{f}_k -\gamma_k - (1-\beta_1^{k}) \lb{k}) + \dotprod{d_k, x^k} \big)_+}{ \|d_k\|_{{\color{color1}\mD_k^{-1}}}^2} 
    $
     
    \STATE $\displaystyle
    \tau_k \;=  \min \Big\{(1-\beta_1^{k})^{-1}\alpha_k,   h_k \Big\}
    $
    
    \STATE $   x^{k+1} = \tfrac{1}{1+\alpha_k{\color{color2}{\lambda}}}\left[x^k - \tau_k {\color{color1}\mD_k^{-1}} d_k\right] \label{eq:momo-adam-xup}$

    \ENDFOR
    \end{algorithmic}
\end{algorithm}

\section{Estimating a Lower Bound} \label{sec:lower}
%
So far, we have assumed that the lower-bound estimates $(\lb{k})$ are given with $\lb{k} =0$ being the default. However, this might not be a tight estimate of $f^*$ (e.g.\ when training transformers). In such situations, we derive an online estimate of the lower bound. 
In particular, for convex functions we will derive a lower bound for an unbiased estimate of $f(x^*)$ given by\\[-0.4cm]
\begin{align}\label{eq:fstar-real}
    \bar{f}^k_* := & \tfrac{1}{\rho_k }\mbox{$\sum_{j=1}^k $}\rho_{j,k} f(x^*,s_j). 
\end{align} 
Though $\bar{f}^k_*$ is not equal to  $f(x^*)$, it is
an unbiased estimate since $\E{f(x^*,s_j)} = f(x^*)$ and hence $\E{\bar{f}^k_*} = f(x^*)$. It is also a reasonable choice since
we motivated our method using the analogous approximation of 
$f(x)$ in~\eqref{eq:approxaverage}. 
The following lemma derives an estimate $f_*^k \geq 0$ for $ \bar{f}^k_*$ given in~\eqref{eq:fstar-real} by using readily available information for any momentum-based method, such as \cref{alg:momo-adam}.
\begin{restatable}{lemma}{lemfstar}\label{lem:fstar}
    Let $f(x,s)$ be convex in $x$ for all $s\in \cal D$. 
    Consider the iterates $x^{k+1} = x^k - \tau_k \mD_k^{-1}d_k$ for $\tau_k>0$.
    Let \vspace{-0.5cm}
    \begin{align*}
        \eta_k &:= \prod_{j=2}^k\lambda_{\min} \big(\mD_{j}^{-1}  \mD_{j-1} \big) \\
        h_k &:=  \bar f_k + \iprod{d_k}{x^k} -\gamma_k.
    \end{align*} 
    It follows that $\bar{f}_*^{k} \geq  \lb{k+1}$ where \vspace{-0.4cm} 
\begin{align*} 
\hspace{-0.4cm}
   \lb{k+1}  \eqdef   
      \tfrac{1}{2 \eta_{k}\tau_k  \rho_k} \biggl(\sum_{j=1}^k 2\eta_{j}\tau_j \big(h_j  -\tfrac{1}{2} \tau_j\norm{d_j}_{\mD_j^{-1}}^2\big) \\
      -D_1^2 -  2 \sum_{j=1}^{k-1}\eta_{j}\tau_j  \rho_j  \bar{f}_*^j\biggr)
\end{align*}
where $D_1 := \norm{x^1 -x^*}_{\mD_1}.$ Bootstrapping by using $f^k_*\approx \bar{f}_*^{k-1}$ we have for $k \geq 2$ that \vspace{-5pt}
\begin{align} \label{eq:fstarconvexD-recur}
    f_*^{k+1} &\approx \;  \tfrac{1}{\rho_k} \left(h_k-\tfrac{1}{2} \tau_{k}\norm{d_{k}}_{\mD_{k}^{-1}}^2 \right).
\end{align}
\end{restatable}
\vspace{-5pt}
To simplify the discussion, consider the case without a preconditioner, 
i.e.
$\mD_k = \mathbf{Id}$, thus  $\eta_k =1$.
Note that $f_*^{k+1}$  depends on the initial distance to the solution $D_1$, which we do not know. 
Fortunately, $D_1$ does not appear in the boostrapped update~\eqref{eq:fstarconvexD-recur}, because it only appears in $f^1_*.$ We can circumvent this initial dependency by setting $f^1_* =0.$

We can now introduce \Momo{} or \MomoAdam{} with $\lb{k}$ estimated as in \eqref{eq:fstarconvexD-recur}, which we call  \Momo{}$^*$ and \MomoAdam{}$^*$.
We need one more precautionary measure: we want to avoid that the step size $\tau_k$ in~\eqref{eq:tauk-general} becomes zero, 
which occurs if 
\begin{align}\label{eq:reset-condition}
\begin{split}
 &(1+\alpha_k\lambda)\rho_k \lb{k} \geq \\
 &\hspace{4ex}(1+\alpha_k\lambda)(\bar{f}_k  - \gamma_k) + \iprod{d_k}{x^k} =: h_k^{\lambda}.
 \end{split}
\end{align}
Hence, in each iteration of \Momo{}$^*$ or \MomoAdam{}$^*$, we call the \texttt{ResetStar} routine in \cref{alg:fstar-reset} \emph{before} the update of $x^{k+1}$ that checks if this upper bound has been crossed, and if so, resets $f_*^k$ to be sufficiently small. After updating $x^{k+1}$, we update $\lb{k+1}$ with \texttt{EstimateStar} routine in \cref{alg:fstar-update}, according to \cref{lem:fstar}. 
The full algorithm of \Momo{}$^*$ is given in \cref{alg:momo-star} in the Appendix. We give an example of how the values of $\lb{k}$ converge to $f^*$ in \cref{sec:appendix-simple-fstar-example}.
%
    \begin{algorithm}
    \caption{\texttt{ResetStar}}
    \label{alg:fstar-reset}
    \begin{algorithmic}
    \STATE {\bfseries Input:}$\lb{k},~\alpha_k,\lambda,\rho_k,h_k^{\lambda} $
    \IF{\eqref{eq:reset-condition}}
    \STATE $\lb{k}=\max\big\{\tfrac{1}{2}[(1+\alpha_k\lambda)\rho_k]^{-1}h_k^{\lambda}, \lb{1} \big\}$
    \ENDIF
    \STATE {\bfseries Return} $\lb{k}$
    \end{algorithmic}
    \end{algorithm}

    \begin{algorithm}
    \caption{\texttt{EstimateStar}}
    \label{alg:fstar-update}
    \begin{algorithmic}
    \STATE {\bfseries Input:} $\bar f_k,x^k,\gamma_k,\tau_k,d_k,\mD_k,\rho_k$
    \STATE $ \displaystyle h_k = \bar f_k + \iprod{d_k}{x_k} - \gamma_k$    
    \STATE $ \displaystyle \lb{k+1} = \max\big\{\rho_k^{-1}(h_k - \tfrac12 \tau_k \|d_k\|^2_{\mD_k^{-1}}, \lb{1}\big\} $
    \STATE {\bfseries Return} $\lb{k+1}$
    \end{algorithmic}
    \end{algorithm}
\section{Convergence Analysis}\label{sec:conv-analysis}
Here we will show that \Momo{} attains a $\mathcal{O}(1/\sqrt{K})$ rate for convex problems with interpolation. 
First, if  $\lb{k}=\bar{f}^k_*$ where $\bar{f}^k_*$ is~\eqref{eq:fstar-real}, the next lemma shows that 
for any preconditioner and convex loss, the iterates of
\Momo{} do not increase the distance to a given optimal point in each step.
\begin{restatable}{lemma}{lemdescent}\label{lem:descent}
Let $f(\cdot, s)$ be convex for every $s$ and let $x^* \in \arg \min_{x\in \R^d} f(x)$. For the iterates of the general \Momo{} update (cf.\ \cref{lem:momo-general-update}) with $\lambda=0$ and $\lb{k} = \bar{f}^k_*$, it holds \\[-0.5cm]
\begin{align}\label{eqn:momo-descent}
    \norm{x^{k+1} -x^*}_{\mD_{k}}^2  \leq \norm{x^k -x^*}_{\mD_k}^2 - \tau_k(h_k-\rho_k\bar{f}^k_*)_+.
\end{align}
\end{restatable}
This monotonicity property is already remarkable as it does not hold in general for standard \SGD{} \cite{handbookgrad}.
In the following theorem, we will use it to prove convergence of \Momo{} (\cref{alg:momo}) for the convex case under interpolation, that is, assuming
\begin{align}\label{eq:interpolation}
    f(x^*,s) = \inf_x f(x,s) = f^* \quad \text{for all } s \in \mathcal{D}.
\end{align}
Interpolation  holds when there exists model parameters $x^*$ such that the loss attains its infimum on every data point. Typically this occurs when the loss is zero for every data point. 
\begin{restatable}{theorem}{thmconvex}\label{thm:convex-case}
    Let $f(\cdot, s)$ be convex for every $s$ and let $x^* \in \arg \min_{x\in \R^d} f(x)$. Assume that \eqref{eq:interpolation} holds.
    Let $(x^k)$ be the iterates of \cref{alg:momo} with $\lb{k} = f^*$, $\alpha_k=+\infty$ for all $k\in \N$ and assume that $d_k\neq0$ for all $k\in \N$. 
    Define  \vspace{-0.35cm} $$B:= \{x ~\vert~ \|x-x^*\| \leq \|x^1-x^*\|\}.$$
    Assume that there exists $G>0$ such that $\max_{x\in B} \E{\|\nabla f(x,s)\|^2} =G^2 < \infty$.
    Then, it holds
    \begin{align*}
    \min_{k=1,\dots,K}  \E{f(x^k)-f^*} \leq \frac{G\|x^1-x^*\|}{\sqrt{K}(1-\beta)}.
\end{align*}
\end{restatable}

We remark that \cref{thm:convex-case} is an unusual result, since in the non-smooth setting for \SGD{}, one needs to assume that the gradients are \emph{globally} bounded, that is, the function is globally Lipschitz continuous~\citep[Thm.\ 9.7]{handbookgrad} (alternatively, one can assume a bounded domain ~\citep{Orabona2019}). Second, the step size to obtain the best iteration complexity requires knowledge of the global Lipschitz constant, and an upper bound of the initial distance to the solution $\|x^0-x^*\|\leq R$. Both of the above is true as well for the analysis of \texttt{SGD-M} \citep[Thm.\ 10]{Defazio2021}.
For \cref{thm:convex-case}, we require neither of the two, but instead rely on interpolation and the mild assumption that the gradients are bounded \emph{on the compact set} $B$ and in expectation.

\section{Experiments}\label{sec:experiments}

Our experiments will focus on the sensitivity with respect to choice of the learning rate $\alpha_k$. \citet{Schmidt2021} showed that most optimization methods perform equally well when being tuned. For practical use, a tuning budget needs to be considered; hence we are interested in methods that require little or no tuning.
Here we investigate how using our \Momo{} adaptive learning rate can improve the stability of both \SGDM{} and \Adam{}. To do this, for each task and model, we do a learning-rate sweep for both \SGDM{}, \Adam{}, \Momo{} and \MomoAdam{} and compare the resulting validation score for each learning rate. 

For \Momo{} and \MomoAdam{}, note that the effective step size (cf.\ \eqref{eq:tauk-general}) has the form  $\tau_k=\min\{\tfrac{\alpha_k}{\rho_k},\zeta_k\}$ with
\begin{equation} \label{eq:zetak}
\zeta_k := \frac{((1+\alpha_k\lambda)(\bar{f}_k -\rho_k \lb{k} -\gamma_k) +  \dotprod{d_k,x^k})_+}{\|d_k\|_{\mD_k^{-1}}^2}.\end{equation}
We refer to \cref{alg:momo}, line \ref{ln:xup} and \cref{alg:momo-adam}, line \ref{eq:momo-adam-xup} for the exact formula for \Momo{} and \MomoAdam{} (For \Momo{} we have that  $\rho_k=1,\mD_k=\Id$).
We will refer to $\alpha_k$ as the \emph{(user-specified) learning rate} and to $\tau_k$ as the \emph{adaptive learning rate}.
\subsection{Zero as Lower Bound}\label{sec:stability-exps}
\subsubsection{Constant Learning Rate}
First, we compare the \Momo{} methods to \SGDM{} and \Adam{} for problems where zero is a good estimate of the optimal value $f^*$.
In this section, we set $\lb{k}=0$ for all $k\in \N$ for \Momo{}(-\Adam{}).

\paragraph{Models and Datasets.}
Here we consider the tasks (additional details in \cref{sec:appendix-models-datasets}): 
\begin{itemize}
    \item \texttt{ResNet110} for \texttt{CIFAR100},  \texttt{ResNet20}, \texttt{VGG16}, and \texttt{ViT}
    for \texttt{CIFAR10}
    \item \texttt{DLRM} 
    for \texttt{Criteo} Display Advertising Challenge,
    \item \texttt{MLP} for \texttt{MNIST}: two hidden layers (size $100$) and \texttt{ReLU}.
\end{itemize}
\paragraph{Parameter Settings.}
We use default choices for momentum parameter $\beta=0.9$ for \Momo{} and \SGDM{}, and $(\beta_1,\beta_2)=(0.9,0.999)$ for \MomoAdam{} and \Adam{} respectively. 
In the experiments of this section, we always report averaged values over three seeds (five for \texttt{DLRM}), and do not use weight decay ($\lambda=0$).

\paragraph{Discussion.}
We run \Momo{}, \MomoAdam{}, \Adam{} and \SGDM{}, for a fixed number of epochs (cf.\ \cref{sec:appendix-models-datasets}),  using a constant learning rate $\alpha_k =\alpha_0$. 
The plots in \cref{fig:stability_val_score} show the final training loss and validation set accuracy of each method when varying the learning rate $\alpha_0$. The training curves for the best runs can be found in \cref{fig:all_val,fig:all_loss}. For \texttt{VGG16} and \texttt{ViT} for \texttt{CIFAR10} and \texttt{MLP} for \texttt{MNIST}, the same plots can be found in \cref{sec:appendix-numerics-info}.
We observe that for small learning rates \Momo{} (\MomoAdam{}) is identical to \SGDM{} (\Adam{}). This is expected, since for small $\alpha_0$, we have $\tau_k = \alpha_0$ (see~\eqref{eq:tauk-general}).

For larger learning rates, we see that \Momo{} and \MomoAdam{} improve the training loss and validation accuracy, but \SGDM{} and \Adam{} decline in performance or fail to converge. 
Most importantly, \Momo(-\Adam{}) consistently extends the range of ``good'' learning rates by over one order of magnitude. Further, \Momo{}(-\Adam{}) achieve the overall best validation accuracy for all problems except \texttt{DLRM} and \texttt{ViT}, where the gap to the best score is minute and within the standard deviation of running multiple seeds (see \cref{table:best_scores}).

This advantage can be explained with the adaptivity of the step size of \Momo{}(-\Adam{}). In \cref{fig:resnet20_step_sizes_momo}, we plot the adaptive term $\zeta_k$~\eqref{eq:zetak} for \Momo{} on a \texttt{ResNet20}. 
For $\alpha_0\in[1,10]$, we observe that the effective learning rate $\tau_k$ is adaptive even though $\alpha_k$ is constant. We observe two phenomena: firstly, in \cref{fig:resnet20_step_sizes_momo} \Momo{} is doing an automatic learning rate decay \emph{without any user choice for a learning-rate schedule}. Secondly, in the very first iterations, \Momo{} is doing a warm-up of the learning rate as $\tau_k=\zeta_k$ starts very small, but quickly becomes large. Both dynamics of $\tau_k$ seemingly improve performance and stability. We also observe faster initial training progress of \Momo{}(-\Adam{}) (cf.\ \cref{fig:all_val,fig:all_loss}). 
We provide additional comparisons to \texttt{Adabelief}~\citep{Zhuang2020}, \texttt{Adabound}~\citep{Luo2019}, and \texttt{Lion}~\citep{Chen2023} in \cref{fig:extendend_results} in the Appendix. 
While \Momo{}(-\Adam{}) performs favourably in this additional benchmark, we stress that it would be easy to derive a \Momo{} version of any preconditioned momentum method (such as \texttt{Adabelief} or \texttt{Adabound}).
\cref{fig:extendend_results} also shows the advantage of \Momo{} when comparing to \SGDM{} with exponentially decaying  schedule for $\alpha_k$ on the \texttt{CIFAR100} experiment. 
%
\newcommand{\figsize}{0.3}
\begin{figure}
    \centering
    \begin{subfigure}[b]{0.49\columnwidth}
        \centering
        \includegraphics[width=0.99\textwidth]{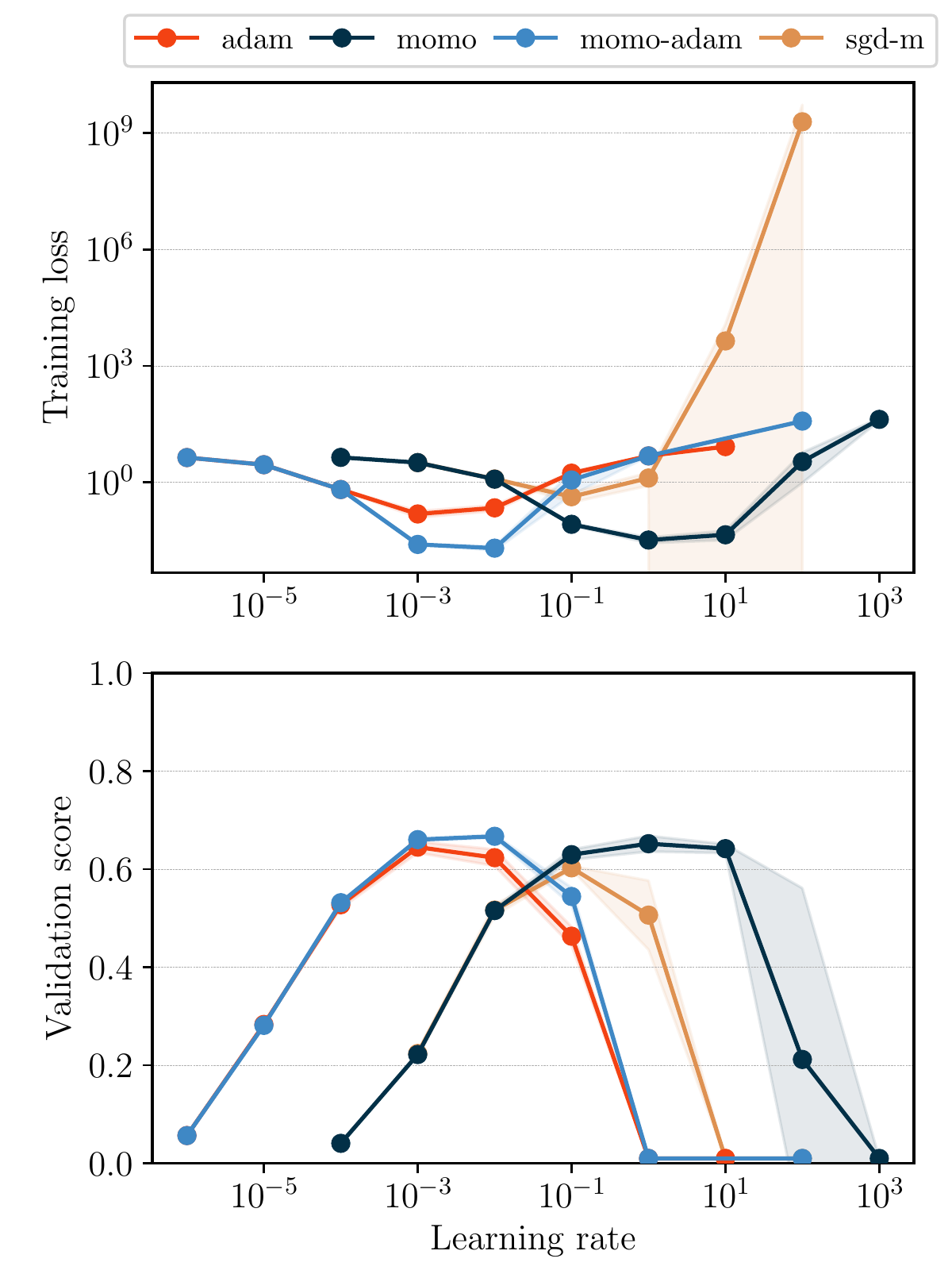}
        \caption{\texttt{ResNet110} for \texttt{CIFAR100}}  
    \end{subfigure}
    \begin{subfigure}[b]{0.49\columnwidth}
        \centering
        \includegraphics[width=0.99\textwidth]{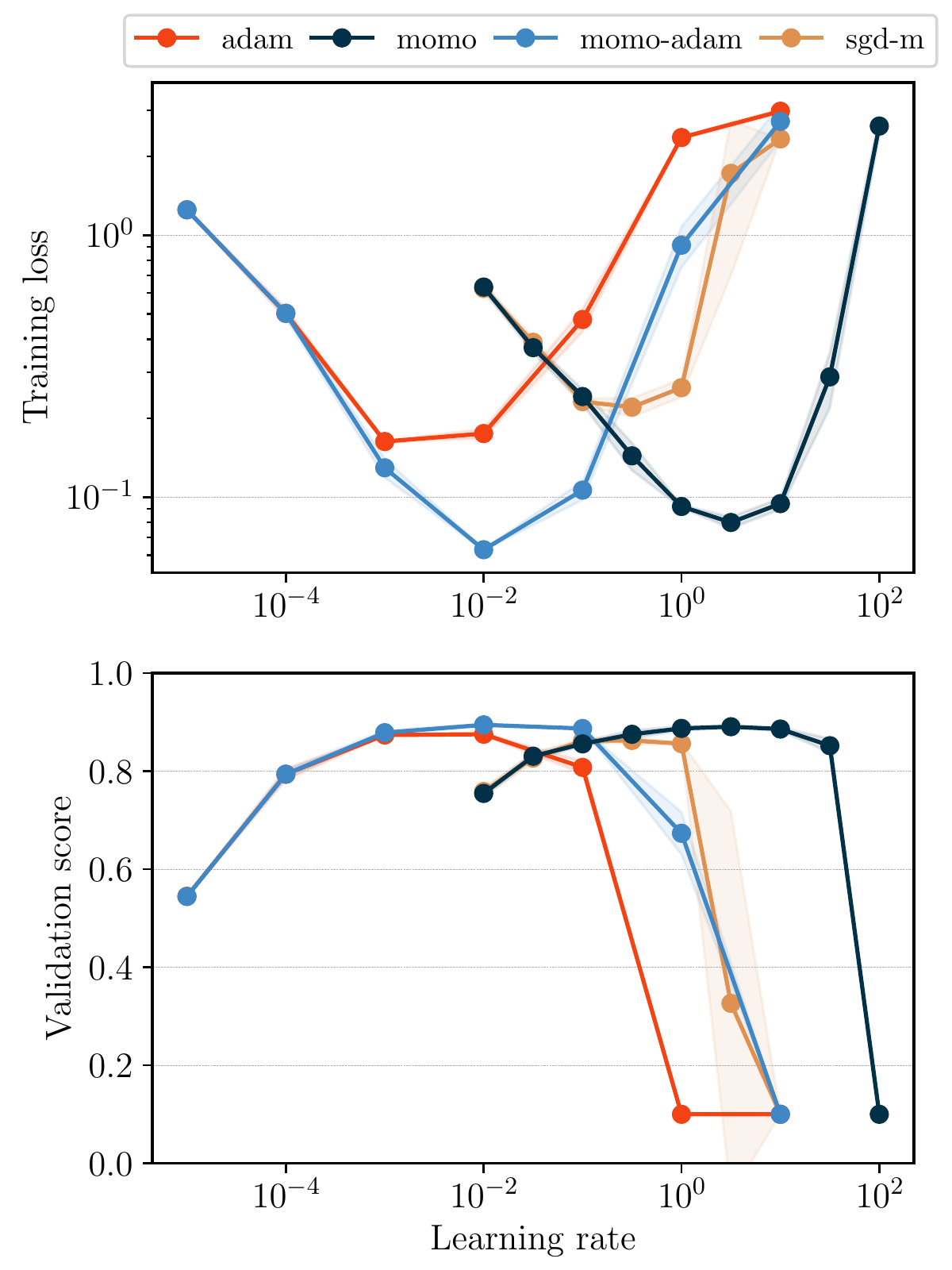}
        \caption{\texttt{ResNet20} for \texttt{CIFAR10}}    
    \end{subfigure}
    \begin{subfigure}[b]{0.98\columnwidth}
        \centering
        \includegraphics[trim={0 0 0  0},clip, width=0.495\textwidth]{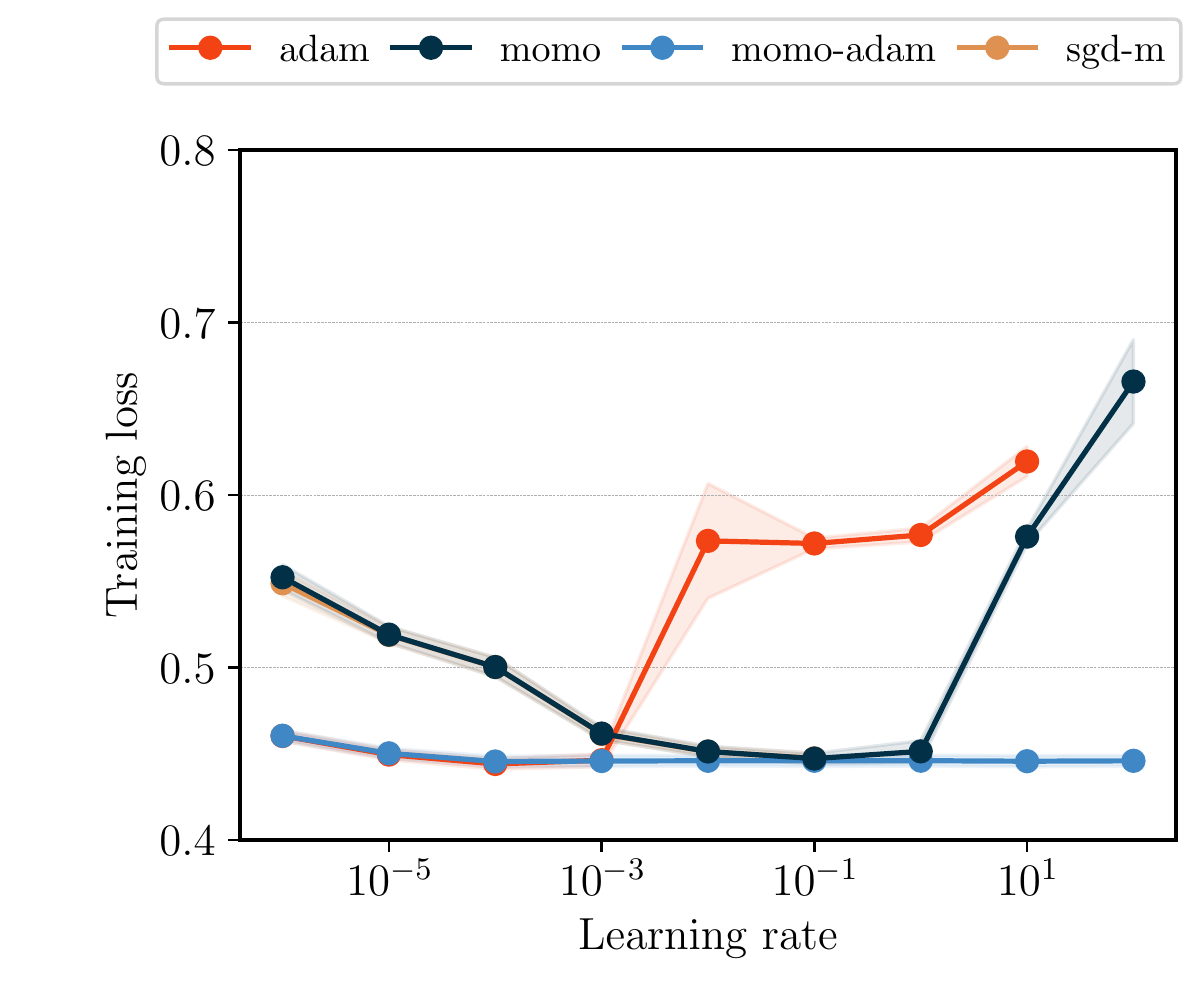} 
        \includegraphics[trim={0 0 0  0cm},clip, width=0.49\textwidth]{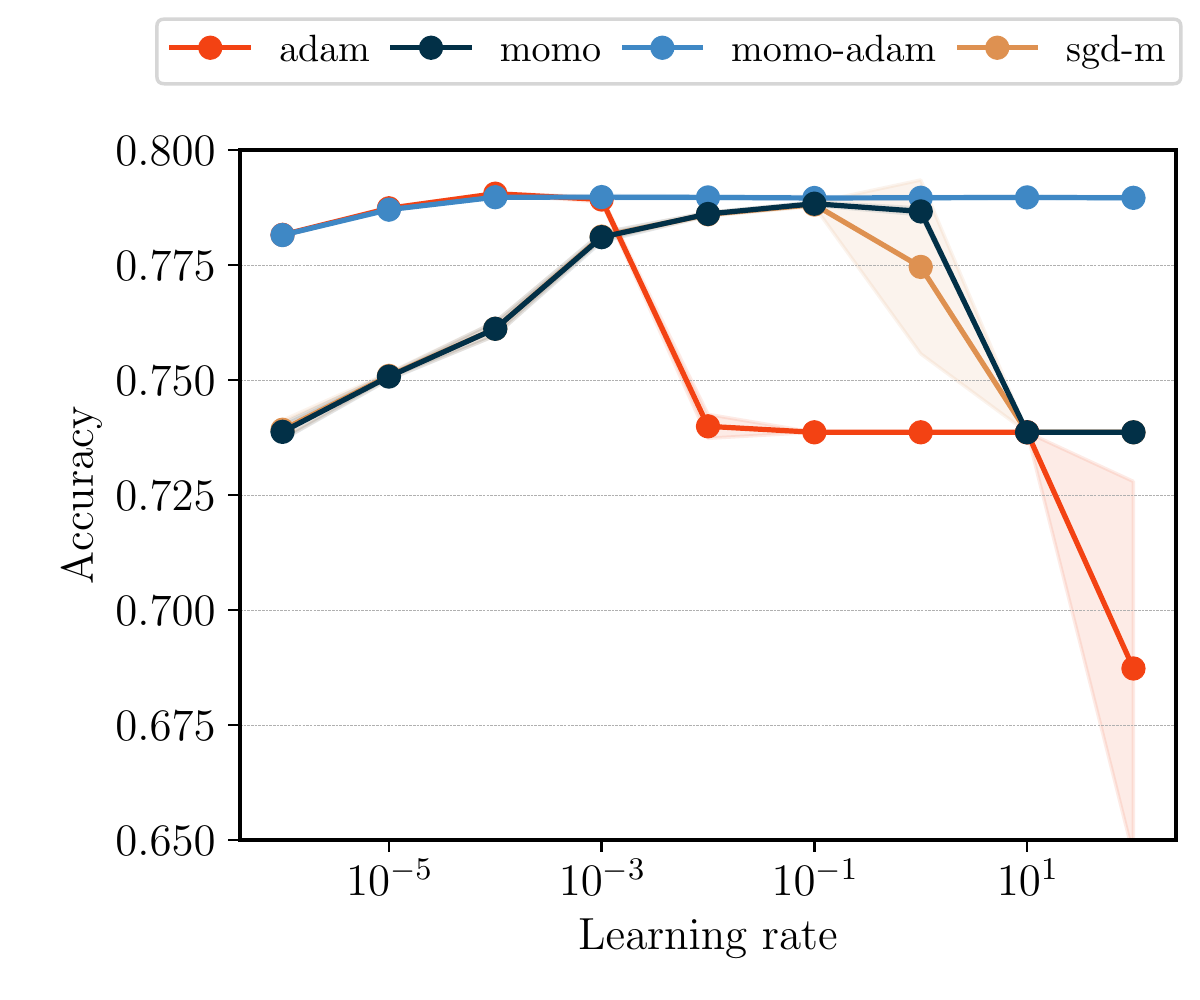}
        \caption{\texttt{DLRM} for \texttt{Criteo}}    
    \end{subfigure}
\caption{{\small Training loss and validation accuracy after a fixed number of epochs, for varying (constant) learning rate $\alpha_0$. Shaded area depicts two standard deviations.}}
\label{fig:stability_val_score}
\end{figure}
\subsubsection{Using a Learning-rate Schedule}
Here we present results for training a vision transformer (\texttt{ViT}) on \texttt{Imagenet-1k} (\cref{fig:vit_imagenet}), and a diffusion model (\texttt{UNet} architecture) on the \texttt{Smithsonian Butterflies}.

Here we only compare \MomoAdam{} to \texttt{AdamW}. This is in order to keep the computational expense within reasonable limits
and because \Adam{} is the prevalent method for these tasks. Experiment details are provided in \cref{sec:appendix-models-datasets}.

What distinguishes these experiments to previous ones, is that we use \MomoAdam{} and \texttt{AdamW} with a learning-rate schedule for $\alpha_k$. For both tasks, it is standard practice to use a warmup (from a very small value to a specified base value $\alpha_{\text{base}}$) followed by cosine decay \cite{Dosovitskiy2021}.
We sweep over $\alpha_{\text{base}}$ and investigate again sensitivity with respect to this choice.  
\begin{figure}
	\centering
	\begin{subfigure}{0.54\columnwidth}  
		\includegraphics[width=0.99\textwidth]{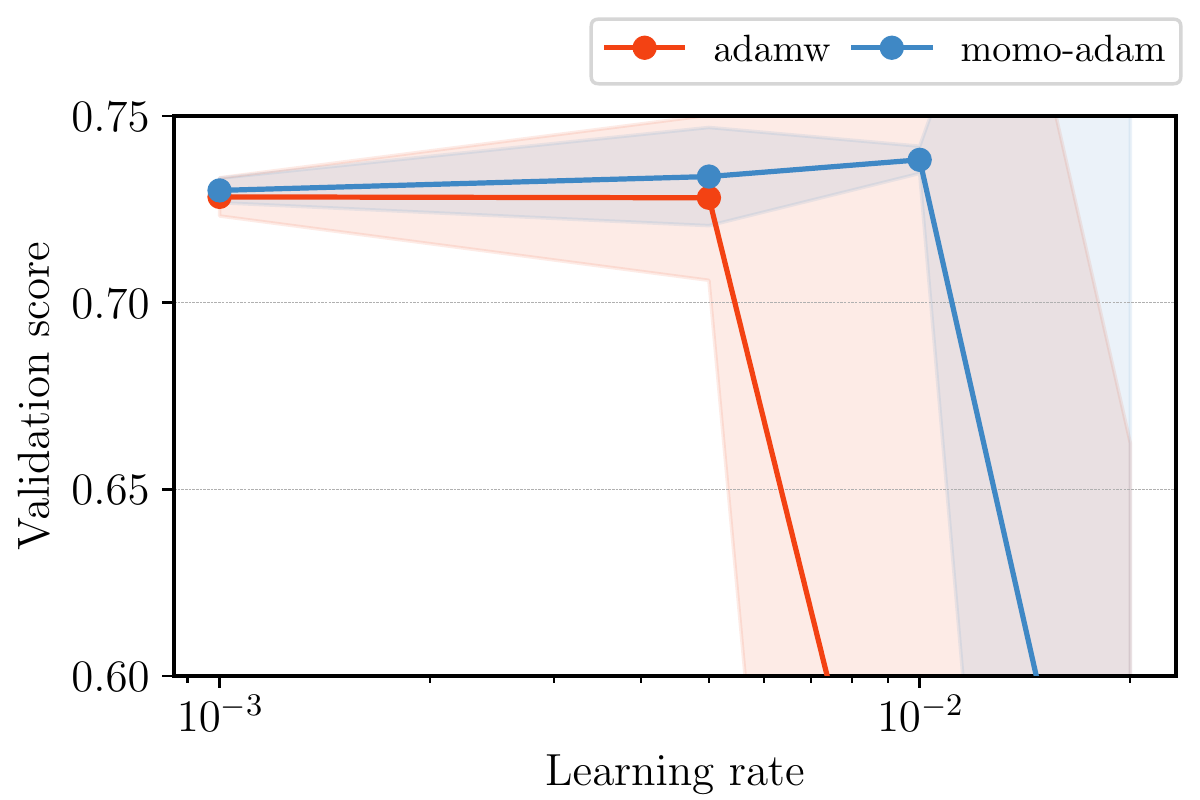}
		\label{fig:vit_imagenet_stability}
	\end{subfigure}
	\begin{subfigure}{0.42\columnwidth}
		\raisebox{0mm}{
			\includegraphics[width=0.9\textwidth]{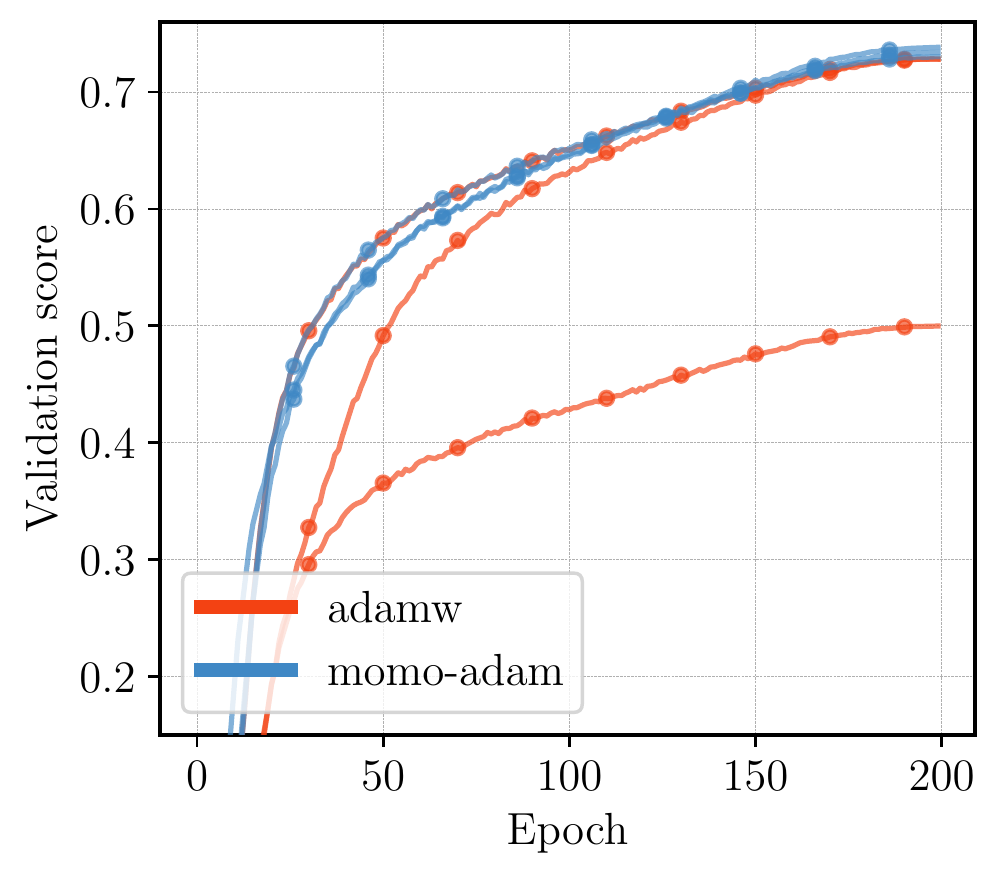}
		}
		\label{fig:vit_imagenet_training}
	\end{subfigure}
	\caption{\texttt{ViT} for \texttt{Imagenet-1k}. Left: Final validation set accuracy (top-1) for different learning-rate values $\alpha_\text{base}$. Right: Training curves for the three best values of $\alpha_\text{base}$ for both methods.}
	\label{fig:vit_imagenet}
\end{figure}
From \cref{fig:vit_imagenet}, for the \texttt{ViT} training, we see that \MomoAdam{} (i) works on a larger range of base learning-rate values and (ii) reaches a higher accuracy for the best value of $\alpha_{\text{base}}$ (here 0.01).

The results for the diffusion model are similar and presented in \cref{fig:unet_smithsonian,fig:unet_smithsonian_samples} in the Appendix. To summarize, we reach the same conclusion as previously: \MomoAdam{} is easier to tune as it works for a wider range of learning rates, stabilizes training, and it can improve the best accuracy. 

For all of the above tasks, the (training) loss converges to values below $0.5$. Next, we consider two problems where the final training loss is significantly above zero. In such situations, we find that the \Momo{} methods with $\lb{k}=0$ are less likely to make use of the adaptive term $\zeta_k$. As a consequence, \Momo{} with $\lb{k}=0$ will yield little or no improvement. To see improvement,  we employ the online estimation of a lower bound for \Momo{} given in \cref{lem:fstar}. %
%
%
\subsection{Online Lower Bound Estimation}
We now consider image classification on \texttt{Imagenet32/-1k} and a transformer for German-to-English translation. For both problems, the optimal value $f^*$ is far away from zero and hence we use \Momo{} with a known estimate of $f^*$ or with the online estimation developed in \cref{sec:lower}. Details on models and datasets are listed in \cref{sec:appendix-models-datasets}.
\paragraph{Imagenet.}
We train a \texttt{ResNet18} for \texttt{Imagenet32}
and give the resulting validation accuracy in \cref{fig:sensitivity_imagenet32} for weight decay $\lambda = 0$. We show results for $\lambda = 10^{-4}$ and for \texttt{Imagenet-1k} in the Appendix in \cref{fig:additional-imagenet}. 
We run \Momo{}(-\Adam{}) first with constant lower bound $\lb{k}=0$ and an \emph{oracle} value $\lb{k}=0.9$. Further, we run \Momo{}(-\Adam{})$^*$ (indicated by the suffix \emph{-star} in the plots),
(cf.\ \cref{alg:momo-star}). 
We compare to \SGDM{} and \texttt{AdamW}. For all methods, we use a constant learning rate $\alpha_k=\alpha_0$ and vary the value of $\alpha_0$.

First we observe in \cref{fig:sensitivity_imagenet32}  that setting $\lb{k}=0$ leads to similar performance as the baseline method (in particular it is never worse). Next, observe that the tighter lower bound $\lb{k}=0.9$ leads to improvement for all learning rates. Finally, the online estimated lower bound widens the range of learning rate with good accuracy by an order of magnitude and leads to small improvements in top accuracy.
%
%

\paragraph{Transformer for German-to-English Translation.}
We consider the task of neural machine translation from German to English by training an encoder-decoder transformer architecture \citep{Vaswani2017} on the \texttt{IWSLT14} dataset. We run two settings, namely dropout of $0.1$ and $0.3$. 
We fine-tune the hyperparameters of the baseline \texttt{AdamW}: for the learning-rate schedule $\alpha_k$, we use a linear warm-up of 4000 iterations from zero to a given base value $\alpha_\text{base}$ followed by an inverse square-root decay (cf.\ \cref{fig:stability_translation} for an example curve and the adaptive step sizes). All other parameter settings are given in \cref{sec:appendix-models-datasets}. \MomoAdam{}$^*$ uses the same hyperparameter settings as \texttt{AdamW}.

\cref{fig:stability_translation} shows the BLEU score after 60 epochs when varying the base learning rate $\alpha_\text{base}$: \MomoAdam{}$^*$ is on par or better than \texttt{AdamW} on the full range of initial learning rates and for both dropout values, with only a small improvement for larger values of $\alpha_\text{base}$. 
While this improvement is not as substantial as for previous examples, we remark that for this particular task we compare to a fine-tuned configuration of \texttt{AdamW}.
\begin{figure}
    \centering
    \begin{subfigure}[b]{0.98\columnwidth} 
        \centering
        \includegraphics[width=0.99\textwidth]{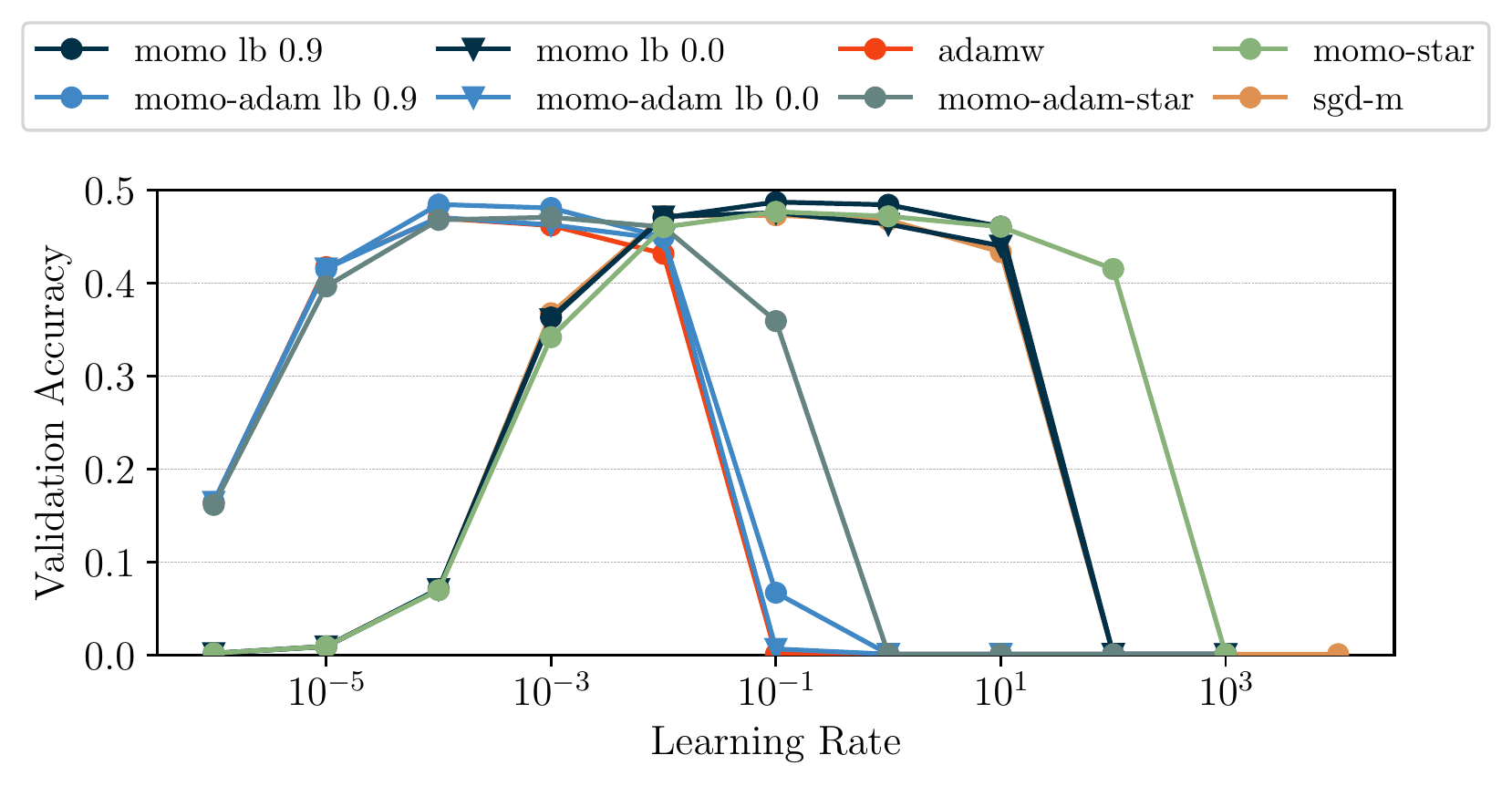}
        \caption{\texttt{ResNet18} for \texttt{Imagenet32}}
        \label{fig:sensitivity_imagenet32}
    \end{subfigure}
    \begin{subfigure}[b]{0.98\columnwidth} 
        \centering
        \includegraphics[width=0.55\textwidth]{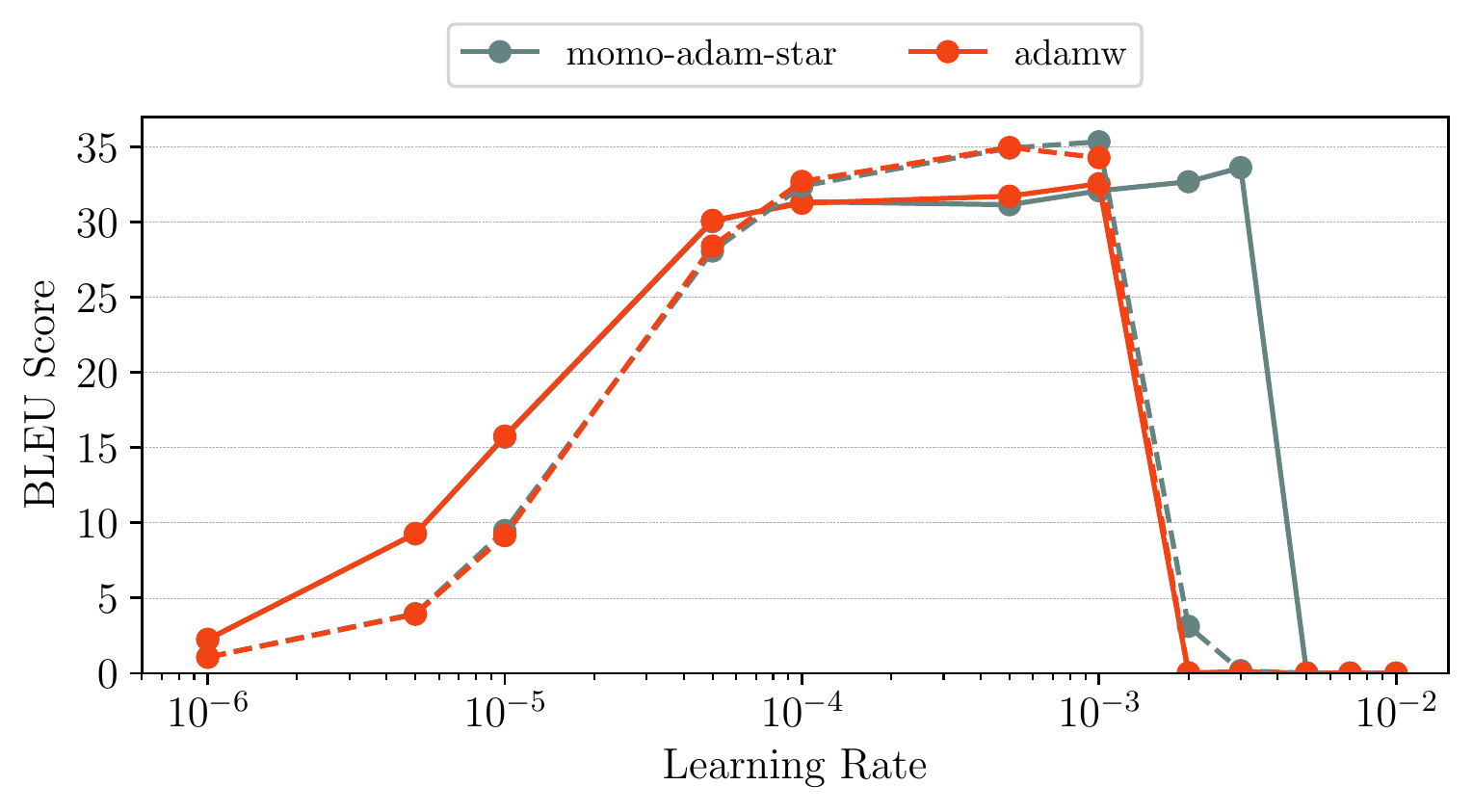}
        \includegraphics[width=0.4\textwidth]{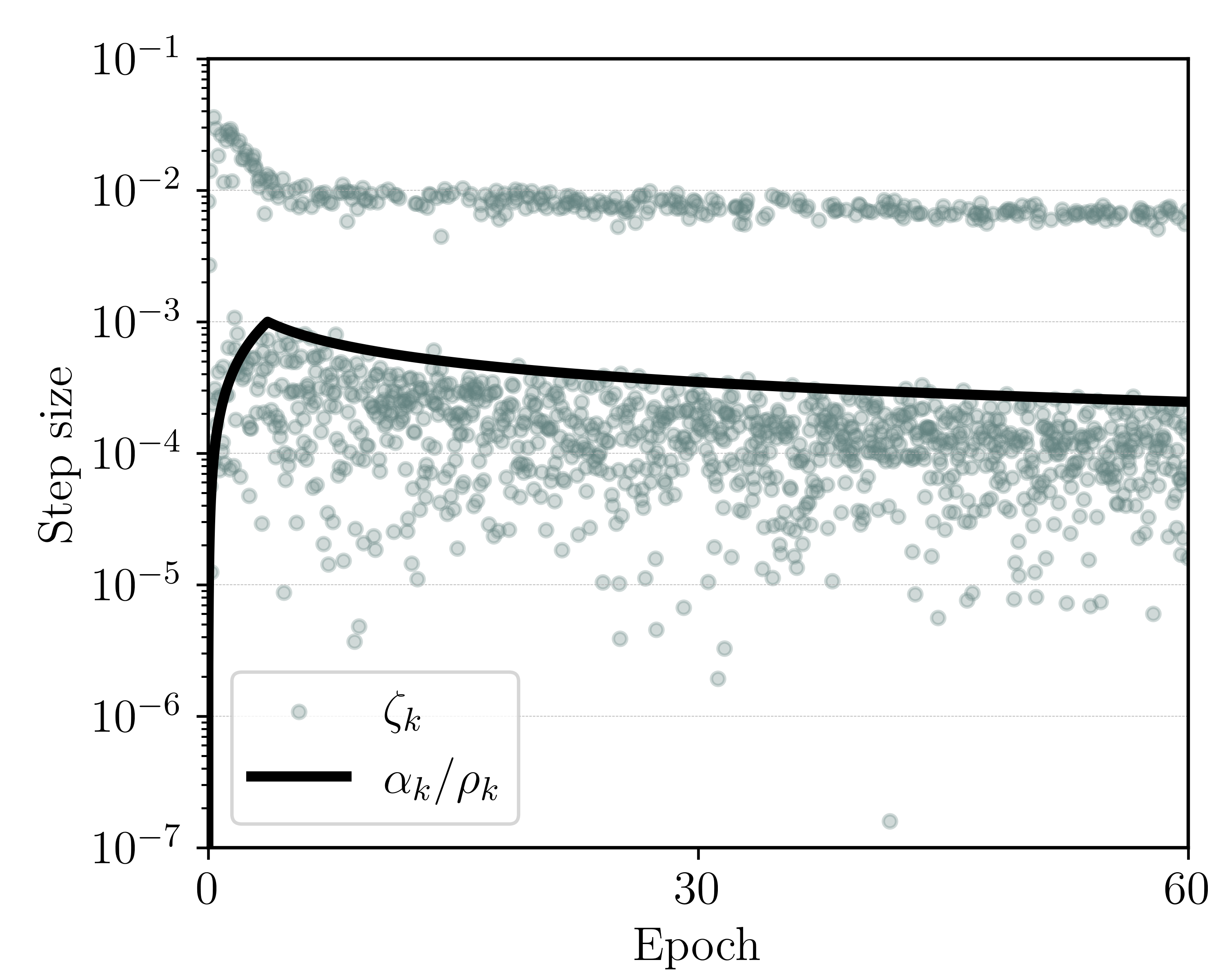}
        \caption{Encoder-Decoder Transformer for \texttt{IWSLT14}}    
        \label{fig:stability_translation}
    \end{subfigure}    
\caption{{\small Validation accuracy over a range of learning rates $\alpha_0$. (a) \texttt{Imagenet32} without weight decay ($\lambda=0$).
(b) Left: \texttt{IWSLT14} translation task with dropout 0.1 (\textbf{plain}) or 0.3 (\textbf{dashed}). 
Right: Learning rate schedule (\textbf{black}) and adaptive step sizes (\textbf{grey dots}) of \MomoAdam{}$^*$ for $\alpha_\text{base}=10^{-3}$.}}
\label{fig:stability_plots}
\end{figure}
\section{Conclusion}
We present \Momo{} and \MomoAdam{}, adaptive learning rates for \SGDM{} and \Adam{}. The main conceptual insight is that momentum can be used to build a model of the loss by averaging a stream of loss function values and gradients. Combined with truncating this average at a known lower bound of the loss, we obtain the \Momo{} algorithms. This technique can be applied to all momentum type methods, including  variants of \Adam{}.

We show examples where incorporating \Momo{} into \SGDM{} and \Adam{} significantly reduces the sensitivity to learning rate choice.
This can be particularly helpful for practitioners who look for good out-of-the-box optimization performance for new tasks.

\section*{Acknowledgements}

The computations in this work were, in part, run at facilities supported by the Scientific Computing Core at the Flatiron Institute, a division of the Simons Foundation.

All models and datasets for this project were utilized by authors at the Flatiron Institute and the Technical University of Munich.

\section*{Impact Statement}
This paper presents work whose goal is to advance the field of Machine Learning. There are many potential societal consequences of our work, none which we feel must be specifically highlighted here.

\bibliography{lib}
\bibliographystyle{icml2024}

\newpage
\appendix
\onecolumn
\input{appendix}


\end{document}

%% file: header.tex
\usepackage{layout}
\usepackage{multirow, makecell}
\usepackage{xcolor}
\usepackage{amsfonts}
\usepackage{amsmath}
\usepackage{amsthm}
\usepackage{amssymb} 

\usepackage{nicefrac}
\usepackage{mathtools}

\usepackage{chngcntr} 

\usepackage[hidelinks]{hyperref}
\usepackage[capitalize,noabbrev]{cleveref}

\usepackage{natbib}

\usepackage{mdframed} 
\usepackage{thmtools}
\usepackage{thm-restate}

\definecolor{shadecolor}{gray}{0.90}
\declaretheoremstyle[
headfont=\normalfont\bfseries,
notefont=\mdseries, notebraces={(}{)},
bodyfont=\normalfont,
postheadspace=1pt,
spaceabove=4pt,
mdframed={
  skipabove=8pt,
  skipbelow=8pt,
  hidealllines=true,
  backgroundcolor={shadecolor},
  innerleftmargin=4pt,
  innerrightmargin=4pt}
]{shaded}

\declaretheorem[style=shaded,within=section]{definition}

\declaretheorem[style=shaded,sibling=definition]{lemma}
\declaretheorem[style=shaded,sibling=definition]{remark}

\usepackage{color}
\usepackage{graphicx}
\graphicspath{{./figures/}}  

\usepackage{enumitem}
\usepackage{subcaption} 
\usepackage{caption}

\usepackage{url}
\usepackage{cite}

\definecolor{kleinblue}{RGB}{0, 47, 167}

\definecolor{color2}{rgb}{0.89, 0.35, 0.13}
\definecolor{color1}{RGB}{33, 131, 128}

\hypersetup{
	colorlinks = True,
	linkcolor = kleinblue,
	citecolor=kleinblue,
	urlcolor=kleinblue
} 

\newcommand{\R}{\mathbb{R}} 
\newcommand{\N}{\mathbb{N}} 

\newcommand{\cD}{{\cal D}}

\newcommand{\cO}{{\cal O}}

\newcommand{\mD}{{\bf D}}

\usepackage[colorinlistoftodos,bordercolor=orange,backgroundcolor=orange!20,linecolor=orange,textsize=scriptsize]{todonotes}

\definecolor{old}{rgb}{0.82, 0.1, 0.26}
\definecolor{revise}{rgb}{0.53, 0.66, 0.42}

\newcommand{\eqdef}{\coloneqq} 
\newcommand{\iprod}[2]{\langle #1, #2 \rangle}

\newcommand{\dotprod}[1]{\left< #1\right>} 
\newcommand{\norm}[1]{ \left\| #1 \right\|}      

\newcommand{\argmin}[1]{\underset{#1}{\argminin} \;}      
\DeclareMathOperator{\argminin}{argmin}  

\newcommand{\Id}{\mathbf{Id}}

\newcommand{\E}[1]{\mathbb{E}\left[#1\right] } 
\newcommand{\EE}[2]{\mathbb{E}_{#1}\left[#2\right] }

\newcommand{\Momo}{{\texttt{MoMo}}}
\newcommand{\MomoAdam}{{\texttt{MoMo-Adam}}}
\newcommand{\Adam}{{\texttt{Adam}}}
\newcommand{\SGD}{{\texttt{SGD}}}
\newcommand{\SGDM}{{\texttt{SGD-M}}}
\newcommand{\spsmax}{\texttt{SPS}\textsubscript{max}}

\newcommand{\lb}[1]{f_*^{#1}}

%% file: appendix.tex
\counterwithin{figure}{section}
\section{Implementation details}\label{sec:appendix-implementation-details}
\subsection{Notes on the Averaging Coefficients}\label{sec:appendix-averaging-coeffs}
\vspace{0.5ex}
\begin{lemma}\label{lem:averaging-coefficients}
    Let $\beta\in[0,1)$. Let $\rho_{1,1}=1$, and for $k \geq 2$ let
    \begin{align*}
        \rho_{j,k} = \begin{cases}\beta \rho_{j,k-1},\quad  &j\leq k-1, \\ 1-\beta,\quad  &j=k.\end{cases}
    \end{align*}
    Then, $\sum_{j=1}^{k} \rho_{j,k} =1$ holds for all $k\in \N$.
    Further, for an arbitrary sequence $(u_j)_{j\in \N} \subset \R^m$, $m\in \N$, consider the weighted sum \[\bar u_k := \sum_{j=1}^k \rho_{j,k} u_j.\]
    Then, if $\bar u_0:= u_1$ it holds $\bar u_k = (1-\beta)  u_k + \beta\bar u_{k-1}$ for all $k\in \N$.
\end{lemma}
\begin{proof}
    We prove that $\sum_{j=1}^{k} \rho_{j,k} =1$ holds for all $k\in \N$ by induction. For the base case $k=1$, we have $\rho_{1,1} =1$ by definition. Assuming that  $\sum_{j=1}^{k-1} \rho_{j,k-1} =1$, we have
    \[ \sum_{j=1}^k \rho_{j,k} = \rho_{k,k}+ \sum_{j=1}^{k-1} \rho_{j,k} = 1-\beta + \beta \sum_{j=1}^{k-1} \rho_{j,k-1} = 
    1-\beta + \beta =1.\]
    Consequently, we have $\bar u_1 =\rho_{11} u_1 = u_1$, and for $ k \geq 2$,
    \begin{align*}
        \bar u_k  &= \sum_{j=1}^k \rho_{j,k} u_j = (1-\beta) u_k + \sum_{j=1}^{k-1} \beta\rho_{j,k-1} u_j = (1-\beta) u_k + \beta \sum_{j=1}^{k-1}\rho_{j,k-1} u_j \\
        &= (1-\beta)  u_k + \beta\bar u_{k-1}.
    \end{align*}
\end{proof}
For the choice of $\rho_{j,k}$ in \cref{lem:averaging-coefficients}, unrolling the recursion, for $k\geq 2$ we obtain the explicit formula
\begin{align}\label{eq:rho-formula-expavg}
    \rho_{j,k} = \begin{cases} (1-\beta)\beta^{k-j},~ &j\geq 2\\ \beta^{k-1},~ &j=1. \end{cases}
\end{align}

\paragraph{Averaging with Bias Correction.}
Chosing $\rho_{j,k} = (1-\beta)\beta^{k-j}$, we have $\rho_{j,k}=\beta \rho_{j,k-1}$, and $\rho_{k,k} = 1-\beta$. Hence, we can update $\bar f_k = (1-\beta) f(x^k,s_k) + \beta \bar f_{k-1}$ and analogously for $d_k,\gamma_k$.
However, this choice does not satisfy $\sum_{j=1}^k \rho_{j,k} =1$. Indeed using the geometric series gives
$$
   \rho_k=(1-\beta) \sum_{j=0}^{k-1} \beta^j = 1-\beta^k.
$$
This fact motivates scaling by the factor of $1-\beta^k$ which was termed \textit{debiasing} in \texttt{Adam}. 
This alternative averaging scheme leads to a variant of \Momo{} with bias correction, presented in \cref{alg:momo_debiased}. As the two presented choices of $\rho_{j,k}$ are very similar, we do not expect major differences in their performance (cf.\ \cref{rem:momo-vs-momo-bias}). 
\begin{algorithm}[h]
    \caption{\texttt{MoMo-Bias}: Model-based Momentum with bias correction. \\ Defaults settings $\beta =0.9$. }
    \label{alg:momo_debiased}
    \begin{algorithmic}[1]
    \STATE {\bfseries Input:}$x^1 \in \R^d$, $\beta \in [0,\;1)$, $\alpha_k>0$, $(\lb{k})_{k\in \N} \subset\R$.
    \STATE {\bfseries Initialize:} $\bar{f}_0 = 0, d_0 =0$ and $\gamma_0 =0 $.
    \FOR{$k=1$ {\bfseries to} $K-1$}
        \STATE $ \displaystyle \bar{f}_k  \; = (1-\beta) f(x^k, s_k) +  \beta\bar{f}_{k-1}  $  
        
      \STATE $ \displaystyle  d_k \;=   (1-\beta) \nabla f(x^k,s_k)  + \beta d_{k-1} $  \label{ln:dup-bias}
         
    \STATE $ \displaystyle \gamma_k =(1-\beta)\dotprod{\nabla f(x^k,s_k), x^k} + \beta\gamma_{k-1} $
     
      \STATE $ \displaystyle  x^{k+1} = x^k - \min \Big\{\frac{\alpha_k}{1-\beta^k}, \frac{(\bar{f}_k - (1-\beta^k)\lb{k} + \dotprod{d_k, x^k} - \gamma_k  )_+}{\|d_k\|^2} \Big\} d_k. \label{ln:xup-bias}$ 
    \ENDFOR
    \end{algorithmic}
\end{algorithm}
\begin{remark}\label{rem:momo-vs-momo-bias}
\cref{alg:momo_debiased} differs from \cref{alg:momo} only in two steps: first, the quantities $\bar f_0,~d_0,~\gamma_0$ are initialized at zero. Secondly, we use $\frac{\alpha_k}{1-\beta^k}$ instead of $\alpha_k$ and $(1-\beta^k)\lb{k} $ instead of $\lb{k} $ in line \eqref{ln:xup-bias}. As $\beta\in[0,1)$, for late iteration number $k$, we can expect that both methods behave very similarly.
\end{remark}
\subsection{Comparison of \MomoAdam{} to \texttt{AdamW}}
\cref{alg:momo-adam} naturally compares to \texttt{AdamW} \citep{Loshchilov2019}. Note that the update of \texttt{AdamW} (in the notation of \cref{alg:momo-adam}) can be written as
\begin{align*}
    x^{k+1} = (1-\alpha_k\lambda)x^k - \frac{\alpha_k}{1-\beta_1^k} \mD_k^{-1}d_k,
\end{align*}
Compared to \cref{alg:momo-adam}, line \ref{eq:momo-adam-xup}, the weight decay of \texttt{AdamW} is not done dividing the whole expression by $\frac{1}{1+\alpha_k\lambda}$, but instead multiplying only $x^k$ with $1-\alpha_k\lambda$. This is a first-order Taylor approximation \citep{Zhuang2022}: for $\alpha$ small it holds $\frac{1}{1+\alpha\lambda} \approx 1-\alpha\lambda$ and $\frac{\alpha}{1+\alpha\lambda} \approx \alpha$.
If we would want to adapt this approximation, we could replace line \ref{eq:momo-adam-xup} with
\begin{align} \label{eq:momo-adam-xup-approx}
    x^{k+1} = (1-\lambda \alpha_k) x^k - \min \Big\{\frac{\alpha_k}{1-\beta_1^k}, \frac{\big((1+\lambda \alpha_k)(\bar{f}_k -(1-\beta_1^k)\lb{k} - \gamma_k)  + \dotprod{d_k, x^k} \big)_+}{\|d_k\|_{\mD_k^{-1}}^2} \Big\} \mD_k^{-1} d_k.
\end{align}
However, the results of \citep{Zhuang2022} suggest that this approximation has almost no impact on the empirical performance.
\subsection{\Momo{}$^*$}
Here we give the complete pseudocode for \Momo{}$^*$, that is the \Momo{} method that uses the estimator for $\lb{k}$ given in \cref{lem:fstar}.
\begin{algorithm}[h]
    \caption{\Momo{}$^*$: Adaptive learning rates and online estimation of $f^*$. }
    \label{alg:momo-star}
    \begin{algorithmic}[1]
    \STATE {\bfseries Input:} $x^1 \in \R^d$, $\beta \in [0,\;1)$, $\alpha_k>0, \lb{1} \subset\R$.
    \STATE {\bfseries Initialize:} $\bar{f}_0 = f(x^1,s_1), d_0 =\nabla f(x^1,s_1)$ and $\gamma_0 =\dotprod{d_0, x^1} $
    \FOR{$k=1$ {\bfseries to} $K-1$}
    \STATE $ \displaystyle \bar{f}_k  \; = (1-\beta) f(x^k, s_k) +  \beta\bar{f}_{k-1}$
    \STATE $\gamma_k =(1-\beta)\dotprod{\nabla f(x^k,s_k), x^k} + \beta\gamma_{k-1}$ 
    \STATE $ \displaystyle  d_k \;=   (1-\beta) \nabla f(x^k,s_k)  + \beta d_{k-1} $  
    \STATE $\displaystyle \lb{k}=\texttt{ResetStar}()$
    \STATE $ \displaystyle  x^{k+1} = x^k - \min \Big\{\alpha_k, \frac{(\bar{f}_k - \lb{k} + \dotprod{d_k, x^k} - \gamma_k  )_+}{\|d_k\|^2} \Big\} d_k$ 
    \STATE $\displaystyle \lb{k+1}=\texttt{EstimateStar}().$
    \ENDFOR
    \RETURN $x^K$
    \end{algorithmic}
\end{algorithm}
%
\clearpage
\section{Auxiliary Lemmas}
\vspace{3ex}
\begin{lemma} \label{lem:euclidean-truncated-update}
Let $y_0,a\in \R^p$ with $a\neq0$ and $c\in \R$. Let $\beta > 0$. The solution to 
\begin{align} \label{prob:max-of-lin-update}
    y^+ = \arg \min_y \underbracket{\Big(c+\langle a, y-y_0\rangle\Big)_+}_{=:h(y)}  + \frac{1}{2\beta} \|y-y_0\|^2
\end{align}
is given by 
\[y^+ = y_0 - \underbracket{\min \big\{\beta, \frac{(c)_+}{\|a\|^2}\big\}}_{=:\tau}a.\]
Moreover we have $h(y^+) = \big(c-\tau \|a\|^2\big)_+$ and
\begin{align}\label{eqn:max-of-lin-val}
    h(y^+) =
    c-\tau \|a\|^2, \quad &\text{if } c\geq 0.
\end{align}
\end{lemma}
\begin{proof}
Clearly, the objective of \eqref{prob:max-of-lin-update} is strongly convex and therefore there exists a unique solution. The (necessary and sufficient) first-order optimality condition is given by
\begin{align}\label{eqn:foc-update}
    0=ta+\beta^{-1}(y^+-y_0), \quad t\in \partial(\cdot)_+(c+\langle a, y^+-y_0\rangle).
\end{align}
We distinguish three cases:
\begin{enumerate}[label=(P\arabic*)]
    \item Suppose $c< 0$. Then, $y_0$ satisfies \eqref{eqn:foc-update} with $t=0$ and hence $y^+=y_0$. In this case $\tau=0$ and $h(y^+)=0=(c)_+$. 
    \item Let $\bar y := y_0-\beta a$ and assume $c+\iprod{a}{\bar y-y_0} > 0 \iff c - \beta \|a\|^2 > 0 \iff \frac{c}{\|a\|^2} > \beta$.  Then $\bar y$ satisfies \eqref{eqn:foc-update} with $t=1$ and hence $y^+=\bar y$. As $\beta>0$, hence $c> 0$ and $\tau=\beta$. As $h(y^+)= c+\iprod{a}{y^+-y_0} =c- \beta \|a\|^2 $, equation \eqref{eqn:max-of-lin-val} holds.
    \item If neither $c< 0$ nor $\frac{c}{\|a\|^2} > \beta$ hold, then it must hold $c+\iprod{a}{y^+-y_0} = 0$. Then, the optimality condition is $0=ta+\beta^{-1}(y^+-y_0)$ for some $t\in[0,1]$. Hence, $y^+ = y_0 - t\beta a$ and $c+\iprod{a}{y^+-y_0} = c- t\beta \|a\|^2 = 0 \iff t=\frac{c}{\beta\|a\|^2}$. As $c\geq 0$ we have $t\geq 0$ and  $\frac{c}{\|a\|^2}\leq \beta$ implies $t\leq 1$. Hence, $\tau=\frac{c}{\|a\|^2}$ and $c-\tau\|a\|^2= c-c=0$, so \eqref{eqn:max-of-lin-val} holds.
\end{enumerate}
\end{proof}

\begin{lemma} \label{lem:max-of-lin-update-norm-we}
Let $y_0,~a\in \R^p$ with $a\neq0$ and $c\in \R$. Let $\mathbf{D} \in \R^{p\times p}$ be a symmetric, positive definite matrix. The solution to 
\begin{align} \label{prob:max-of-lin-update-norm-we}
    y^+ = \argmin{y \in \R^p} \underbracket{\Big(c+\langle a, y-y_0\rangle\Big)_+}_{=:h(y)}  +~ \frac{1}{2\alpha} \|y-y_0\|_\mathbf{D}^2 + \frac{\lambda}{2} \| y\|_\mathbf{D}^2
\end{align}
is given by 
\[y^+ =\frac{1}{1+\lambda\alpha} \Bigg[ y_0 -\underbracket{\min \Big\{\alpha, \frac{\big((1+\lambda\alpha) c-\lambda\alpha \dotprod{a,y_0}\big)_+}{\|a\|_{\mathbf{D}^{-1}}^2}\Big\}}_{=:\tau}  \mathbf{D}^{-1}a \Bigg].\]
Furthermore
\[h(y^+) = \Big(c- \frac{\lambda\alpha}{1+\lambda\alpha}\langle a, y_0\rangle -\frac{\tau}{1+\lambda\alpha} \norm{a}_{\mathbf{D}^{-1}}^2\Big)_+.\]
\end{lemma}
\begin{proof}
First we complete the squares as follows 
\begin{align*}
    \frac{\lambda}{2}\|y\|_{\mathbf{D}}^2 + \frac{1}{2\alpha}\|y-y_0\|_{\mathbf{D}}^2 &= \frac{1}{2\alpha}\|y\|^2_{(1+\lambda\alpha)\mathbf{D}} - \frac{1}{\alpha}\iprod{y}{\mathbf{D} y_0} + \mathrm{cst}.(y)  \\
    &= \frac{1}{2\alpha}\|y\|^2_{(1+\lambda\alpha)\mathbf{D}} - \frac{1}{\alpha}\iprod{y}{(1+\lambda\alpha)\mathbf{D}\tfrac{y_0}{1+\lambda\alpha}} + \mathrm{cst}.(y)  \\
    &=\frac{1}{2\alpha} \|y-\tfrac{1}{1+\lambda\alpha}y_0\|^2_{(1+\lambda\alpha)\mathbf{D}} + \mathrm{cst}.(y),
\end{align*}
where $\mathrm{cst}.(y)$ denotes terms that are constant in $y$.
Using the above, \eqref{prob:max-of-lin-update-norm-we} is equivalent to
\begin{align*}
    y^+ &= \argmin{y \in \R^p} h(y) + \frac{1}{2\alpha}\|y-\tfrac{1}{1+\lambda\alpha}y_0\|_{(1+\lambda\alpha)\mathbf{D}}^2 \\
    & =\argmin{y \in \R^p} \left( c + \iprod{a}{ y-\tfrac{1}{1+\lambda\alpha}y_0} +\left(\tfrac{1}{1+\lambda\alpha}-1\right) \dotprod{a,y_0}\right)_+  + \frac{1}{2\alpha}\|y-\tfrac{1}{1+\lambda\alpha}y_0\|_{(1+\lambda\alpha)\mathbf{D}}^2.
\end{align*}
Let $\hat{c} := c+ \left(\tfrac{1}{1+\lambda\alpha}-1\right) \dotprod{a,y_0}
= c-\tfrac{\lambda\alpha}{1+\lambda\alpha} \dotprod{a,y_0}.$
With this definition, problem \eqref{prob:max-of-lin-update-norm-we} is equivalent to
\begin{align*}
    y^+ 
    & =\argmin{y \in \R^p} \left( \hat{c} + \iprod{a}{ y-\tfrac{1}{1+\lambda\alpha}y_0} \right)_+  + \frac{1}{2\alpha}\|y-\tfrac{1}{1+\lambda\alpha}y_0\|_{(1+\lambda\alpha)\mathbf{D}}^2.
\end{align*}
Changing variables with $z^+=\mathbf{D}^{1/2}y^+$, $z = \mathbf{D}^{1/2}y$, and $z_0 = \mathbf{D}^{1/2} y_0$ gives
\begin{align*}
    z^+ 
    & =\argmin{z \in \R^p} \left( \hat{c} + \iprod{\mathbf{D}^{-1/2}a}{ z-\tfrac{1}{1+\lambda\alpha}z_0} \right)_+  + \frac{(1+\lambda\alpha)}{2\alpha}\|z-\tfrac{1}{1+\lambda\alpha}z_0\|^2.
\end{align*}
Applying \cref{lem:euclidean-truncated-update} with 
$y_0 \leftarrow \tfrac{1}{1+\lambda\alpha} z_0,\; c \leftarrow \hat{c},\; a \leftarrow \mathbf{D}^{-1/2}a, \beta \leftarrow \frac{\alpha}{1+\lambda \alpha}$
gives
\[z^+ = \frac{1}{1+\lambda\alpha} z_0 - \min \big\{\frac{\alpha}{1+\lambda \alpha}, \frac{(\hat{c})_+}{\|a\|_{\mathbf{D}^{-1}}^2}\big\}\mathbf{D}^{-1/2}a.\]
Changing variables back using $y^+= \mathbf{D}^{-1/2}z^+$, substituting $\hat{c} 
= c-\tfrac{\lambda\alpha}{1+\lambda\alpha} \dotprod{a,y_0}$ and re-arranging the above gives
\begin{align}
    y^{+} &=  \frac{1}{1+\lambda\alpha} y_0 -\min \Big\{\frac{\alpha}{1+\lambda \alpha}, \frac{\big(c-\tfrac{\lambda\alpha}{1+\lambda\alpha} \dotprod{a,y_0}\big)_+}{\|a\|_{\mathbf{D}^{-1}}^2}\Big\}\mathbf{D}^{-1}a \nonumber \\
    &= \frac{1}{1+\lambda\alpha} \Bigg[ y_0 -\min \Big\{\alpha, \frac{\big((1+\lambda\alpha) c-\lambda\alpha \dotprod{a,y_0}\big)_+}{\|a\|_{\mathbf{D}^{-1}}^2}\Big\}\mathbf{D}^{-1}a \Bigg].
\end{align}
\end{proof}

\section{Missing Proofs} \label{sec:missingproofs}

\subsection{Proof of \Cref{lem:update} } \label{sec:proof-update-lemma-momo}

\vspace{2ex}
\lemupdate*
\begin{proof}
Recall problem \eqref{eq:modelbasedVI} given by
\begin{align*}
    x^{k+1} = \argmin{y \in \R^d} m_k(y) + \frac{1}{2\alpha_k}\|y-x^k\|^2.
\end{align*}
Introducing
\begin{align}\label{eq:def-hk}
    h_k \eqdef \sum_{j=1}^k \rho_{j,k} [f(x^j, s_j) + \iprod{\nabla f(x^j,s_j)}{x^k-x^j}] = \bar f_k + \iprod{d_k}{x^k} - \gamma_k,
\end{align}
we have that
\begin{align}\label{eq:model-simplified}
m_k(y) = \max\Big\{\rho_k^{-1}(h_k+ \iprod{d_k}{y-x^k}), \lb{k} \Big\} = \Big(\rho_k^{-1}(h_k  + \iprod{d_k}{y-x^k} )-\lb{k}\Big)_+ +  \lb{k}. 
\end{align}
Using \eqref{eq:model-simplified}, dropping the constant term $\lb{k} $, and multiplying with $\rho_k$, problem \eqref{eq:modelbasedVI} is equivalent to 
\begin{align*}
    x^{k+1} = \argmin{y \in \R^d} \Big(h_k  + \iprod{d_k}{y-x^k}- \rho_k\lb{k}\Big)_+ + \frac{\rho_k}{2\alpha_k}\|y-x^k\|^2.
\end{align*}
Applying \cref{lem:euclidean-truncated-update} with $\beta\leftarrow \rho_k^{-1}\alpha_k$, $c\leftarrow h_k - \rho_k\lb{k} $, $a\leftarrow d_k$ and $y_0 \leftarrow x^k$ gives the result.
\end{proof}

\subsection{Proof of \Cref{lem:momo-general-update} } \label{sec:proof-update-lemma-general-momo}

\vspace{2ex}
\lemmomogeneralupdate*
\begin{proof}
Recall problem \eqref{eq:adaptive-model-update-wd} given by
\begin{align*}
    x^{k+1} = \argmin{y \in \R^d} m_k(y)  + \frac{1}{2\alpha_k}\|y-x^k\|_{\mD_k}^2+ \frac{\lambda}{2}\|y\|_{\mD_k}^2.
\end{align*}
We use again \eqref{eq:model-simplified}. Dropping the constant term $\lb{k} $, and multiplying with $\rho_k$, problem \eqref{eq:adaptive-model-update-wd} is equivalent to
\begin{align*}
    x^{k+1} = \argmin{y \in \R^d} \Big(h_k  + \iprod{d_k}{y-x^k}- \rho_k\lb{k}\Big)_+ + \frac{\rho_k}{2\alpha_k}\|y-x^k\|_{\mD_k}^2+ \frac{\rho_k \lambda}{2}\|y\|_{\mD_k}^2.
\end{align*}
Now applying \cref{lem:max-of-lin-update-norm-we}
with $y_0 \leftarrow x^k$, $a\leftarrow d_k$, $c\leftarrow h_k-\rho_k\lb{k} $, $\lambda \leftarrow \rho_k\lambda $, $\alpha \leftarrow \rho_k^{-1}\alpha_k$ and $\mathbf{D}\leftarrow \mD_k$, we obtain the result.
\end{proof}

\clearpage
\section{Estimating a Lower Bound: Proofs and Alternatives} 

\subsection{Proof of \Cref{lem:fstar} }\label{sec:proof-lower-bound}

\vspace{2ex}
\lemfstar*

\begin{proof}
Switching the index $k\to j$ in the update rule, we have
\[        x^{j+1} = x^j - \tau_j \mD_j^{-1}d_j ,\]
where $\tau_j$ is the step size. Subtracting $x^*$ from both sides, taking norms and expanding the squares we have that
\begin{align}   
\norm{x^{j+1} -x^*}_{\mD_j}^2 &= \norm{x^j -x^*}_{\mD_j}^2- 2\tau_j \dotprod{d_j,x^j -x^*} +\tau_j^2\norm{d_j}_{\mD_j^{-1}}^2. \label{eq:zo84ho8hz} 
\end{align}
Now let $\delta_{j+1} := \lambda_{\min} \big(\mD_{j+1}^{-1}  \mD_{j} \big)$ and  note that for every vector $v\in\R^d$ we have that
\begin{align}
 \delta_{j+1}\left\Vert v\right\Vert ^{2}_{\mD_{j+1}}    \leq  \left\Vert v\right\Vert ^{2}_{\mD_j}.\label{eq:ten;px9hp9xt}
\end{align}
Indeed this follows since
\begin{align*}
    \left\Vert v\right\Vert ^{2}_{\mD_j} &= v^\top  \mD_j v =   v^\top \mD_{j+1}^{1/2} \big(\mD_{j+1}^{-1/2}  \mD_{j} \mD_{j+1}^{-1/2} \big) \mD_{j+1}^{1/2}v \\
    &\geq \lambda_{\min} \big(\mD_{j+1}^{-1}  \mD_{j} \big) \norm{v}_{\mD_{j+1}}^2 =  \delta_{j+1} \norm{v}_{\mD_{j+1}}^2.  
\end{align*}

For simplicity, denote $\nabla f_l = \nabla f(x^l,s_l), f_l = f(x^l,s_l)$. We have that
\begin{align}
   \dotprod{d_j,x^j-x^*} &= \sum_{l=1}^j \rho_{l,j} \dotprod{\nabla f_l,x^j-x^*}\nonumber \\
    & = \sum_{l=1}^j \rho_{l,j}\left( \dotprod{\nabla f_l,x^j-x^l} +\dotprod{\nabla f_l,x^l-x^*}\right) \nonumber \\
    &\geq  \sum_{l=1}^j \rho_{l,j}\left( \dotprod{\nabla f_l,x^j-x^l} + f_l-f(x^*,s_l)\right)\nonumber &\mbox{(by convexity of $f(\cdot,s)$)} \\ 
    & =   \bar{f}_j+\dotprod{d_j,x^j} -\gamma_j -\sum_{l=1}^j \rho_{l,j} f(x^*,s_l) \; = \; h_j- \rho_{j} \bar{f}_*^j  . \label{eq:dkdotx-xstar}
\end{align}
Using~\eqref{eq:ten;px9hp9xt} together with~\eqref{eq:dkdotx-xstar} in~\eqref{eq:zo84ho8hz} gives
\begin{align}   
\delta_{j+1} \norm{x^{j+1} -x^*}_{\mD_{j+1}}^2 & \leq \norm{x^{j+1} -x^*}_{\mD_j}^2 \nonumber \\
&= \norm{x^j -x^*}_{\mD_j}^2- 2\tau_j \dotprod{d_j,x^j -x^*} +\tau_j^2\norm{d_j}_{\mD_j^{-1}}^2 \nonumber \\
& \leq \norm{x^j -x^*}_{\mD_j}^2- 2\tau_j (h_j- \rho_j \bar{f}_*^j)+\tau_j^2\norm{d_j}_{\mD_j^{-1}}^2. 
\end{align}
Now we will perform a weighted telescoping. We will multiply the above by $\eta_j>0$ such that $\delta_{j+1} \eta_j = \eta_{j+1} ,$ thus $\eta_{j} = \eta_1 \prod_{l=2}^j\delta_{l} .$ Thus multiplying through by $\eta_{j} $ we have that
\begin{align}   
\eta_{j+1} \norm{x^{j+1} -x^*}_{\mD_{j+1}}^2 
& \leq \eta_{j}\norm{x^j -x^*}_{\mD_j}^2- 2\eta_{j}\tau_j (h_j- \rho_j \bar{f}_*^j)+\eta_{j}\tau_j^2\norm{d_j}_{\mD_j^{-1}}^2. \nonumber
\end{align}
Summing up from $j=1, \ldots, k$ and telescoping we have that
\begin{align}   
0 & \leq \eta_{k+1} \norm{x^{k+1} -x^*}_{\mD_{k+1}}^2 \nonumber\\
& \leq \eta_{1}\norm{x^1 -x^*}_{\mD_1}^2- 2\sum_{j=1}^k\eta_{j}\tau_j (h_j- \rho_j \bar{f}_*^j)+\sum_{j=1}^k\eta_{j}\tau_j^2 \norm{d_j}_{\mD_j^{-1}}^2. \label{eq:oz;89he9jze}
\end{align}
Re-arranging the above, choosing $\eta_1 = 1$ and isolating $\bar{f}_*^k$ gives
\begin{align*}
  2 \eta_{k}\tau_k  \rho_k  \bar{f}_*^k & \geq  2\sum_{j=1}^k\eta_{j}\tau_j h_j -\norm{x^1 -x^*}_{\mD_1}^2 -\sum_{j=1}^k\eta_{j}\tau_j^2\norm{d_j}_{\mD_j^{-1}}^2 -  2 \sum_{j=1}^{k-1}\eta_{j}\tau_j  \rho_j  \bar{f}_*^j.
\end{align*}
Dividing through by $2 \eta_{k}\tau_k  \rho_k$ gives
the main result.
Finally the recurrence follows since, for $k \geq 2$ we have that
\begin{align*}
 &\lb{k+1} =  \frac{2\sum_{j=1}^{k}\eta_{j}\tau_j h_j -\norm{x^1 -x^*}_{\mD_1}^2 -\sum_{j=1}^{k}\eta_{j}\tau_j^2\norm{d_j}_{\mD_j^{-1}}^2 -  2 \sum_{j=1}^{k-1}\eta_{j}\tau_j  \rho_j  \bar{f}_*^j}{2 \eta_{k}\tau_{k}  \rho_{k} } \\
   & =\frac{\eta_{k-1}\tau_{k-1}  \rho_{k-1}}{\eta_{k}\tau_{k}  \rho_{k}} \underbracket{ \frac{2\sum_{j=1}^{k-1}\eta_{j}\tau_j h_j -\norm{x^1 -x^*}_{\mD_1}^2 -\sum_{j=1}^{k-1}\eta_{j}\tau_j^2\norm{d_j}_{\mD_j^{-1}}^2 -  2 \sum_{j=1}^{k-2}\eta_{j}\tau_j  \rho_j  \bar{f}_*^j}{2 \eta_{k-1}\tau_{k-1}  \rho_{k-1} }}_{=\lb{k}} \\
   & \hspace{2ex}\; +\frac{\eta_{k-1}\tau_{k-1}  \rho_{k-1}}{\eta_{k}\tau_{k}  \rho_{k}} \frac{2\eta_{k}\tau_{k} h_{k}  -\eta_{k}\tau_{k}^2\norm{d_{k}}_{\mD_{k}^{-1}}^2 -  2 \eta_{k-1}\tau_{k-1}  \rho_{k-1}  \bar{f}_*^{k-1}}{2 \eta_{k-1}\tau_{k-1}  \rho_{k-1} }\\
   &= \frac{2\eta_{k-1}\tau_{k-1}  \rho_{k-1}(\lb{k}-\bar{f}_*^{k-1})  -\eta_{k}\tau_{k}^2\norm{d_{k}}_{\mD_{k}^{-1}}^2    +2\eta_{k}\tau_{k} h_{k}   }{2\eta_{k}\tau_{k}  \rho_{k}}.
\end{align*}

Now bootstrapping by using $\lb{k}\approx \bar f_*^{k-1}$ gives the result.
\end{proof}

\subsection{The Max Lower Bound}
Here we derive an alternative estimate for the lower bound that does not require bootstrapping, contrary to \cref{lem:fstar}.

\begin{lemma} \label{lem:fstarmax}
        Let $f(x,s)$ be convex in $x$ for every sample $s$. Furthermore let $x^* \in  \argmin{x\in \R^d} f(x)$. 
        Consider the iterates $x^{k+1} = x^k - \tau_k \mD_k^{-1}d_k$ with $\tau_k>0$,
    and let
    \begin{align*}
      \eta_k = \prod_{j=2}^k\lambda_{\min} \big(\mD_{j}^{-1}  \mD_{j-1} \big), \;
    \bar{f}_*^k :=  \frac{1}{\rho_k }\sum_{j=1}^k \rho_{j,k} f(x^*,s_j), \;
    h_k =  \bar f_k + \iprod{d_k}{x^k} -\gamma_k.
    \end{align*}
It follows that 
    \begin{align} \label{eq:fstarconvexD-old}
    \max_{j=1,\ldots, k}\bar{f}_*^j \geq  \lb{k+1}  \eqdef \frac{2\sum_{j=1}^k\eta_j\tau_j h_j -\norm{x^1 -x^*}^2 -\sum_{j=1}^k\eta_j\tau_{j}^2\norm{d_j}_{\mD_j^{-1}}^2 }{2 \sum_{j=1}^k\eta_j\tau_j  \rho_j}. 
    \end{align}
Furthermore we have the recurrence 
\begin{align} \label{eq:fstarconvexD-recur-old}
    \lb{k+1} & = \; 
    \frac{\lb{k} \sum_{j=1}^{k-1}\eta_j\tau_j\rho_j +\eta_{k}\tau_{k} \left( h_{k}- \frac{1}{2}\tau_{k}\norm{d_{k}}_{\mD_{k}^{-1}}^2\right)} {\sum_{j=1}^{k}\eta_j\tau_j\rho_j}  .
\end{align}
In particular when $\mD_k = \Id$ for every $k$, then we have that $\eta_k =1$ for all $k$.
\end{lemma}
\begin{proof}
From step~\eqref{eq:oz;89he9jze} and  re-arranging we have that
\begin{align*}
  2 \big(\max_{j=1,\ldots, k}\bar{f}_*^j \big)\big( \sum_{j=1}^k\eta_{j}\tau_j  \rho_j\big)   \; &\geq   2 \big( \sum_{j=1}^k\eta_{j}\tau_j  \rho_j\big)  \bar{f}_*^j  \\
  &\geq 2\sum_{j=1}^k\eta_{j}\tau_j h_j -\norm{x^1 -x^*}_{\mD_1}^2 -\sum_{j=1}^k\eta_{j}\tau_j^2\norm{d_j}_{\mD_j^{-1}}^2.
\end{align*}
If we now assume that $\bar{f}_*^j \approx f(x^*)$ (or upper bounding $\bar{f}_*^j$ by a constant) then by substituting in $f(x^*)$,
dividing through by $\big( \sum_{j=1}^k\eta_{j}\tau_j  \rho_j\big) $ gives the estimate
\[
\max_{j=1,\ldots, k}\bar{f}_*^j \geq  \lb{k+1}  \eqdef \frac{2\sum_{j=1}^k\eta_j\tau_j h_j -\norm{x^1 -x^*}^2 -\sum_{j=1}^k\eta_j\tau_{j}^2\norm{d_j}_{\mD_j^{-1}}^2 }{2 \sum_{j=1}^k\eta_j\tau_j  \rho_j}. 
\]
Finally the recurrence follows since 
\begin{align*}
        \lb{k+1} &=  \frac{2\sum_{j=1}^{k}\eta_j\tau_j h_j -\norm{x^1 -x^*}_{\mD_1}^2 -\sum_{j=1}^{k}\eta_j\tau_j^2\norm{d_j}_{\mD_j^{-1}}^2 }{2 \sum_{j=1}^{k}\eta_j\tau_j  \rho_j}\\
        &=  \frac{\sum_{j=1}^{k-1}\eta_j\tau_j \rho_j} {\sum_{j=1}^{k}\eta_j\tau_j \rho_j} \frac{2\sum_{j=1}^{k-1}\eta_j\tau_j  h_j-\norm{x^1 -x^*}_{\mD_1}^2- \sum_{j=1}^{k-1}\eta_j\tau_j^2\norm{d_j}_{\mD_j^{-1}}^2}{2\sum_{j=1}^{k-1}\eta_j\tau_j \rho_j} 
        \\ &\quad  + \frac{2\eta_{k}\tau_{k}  h_{k}-\eta_{k}\tau_{k}^2\norm{d_{k}}_{\mD_{k}^{-1}}^2 }{2\sum_{j=1}^{k}\eta_j\tau_j\rho_j} \\
        &= \frac{ \lb{k} \sum_{j=1}^{k-1}\eta_j\tau_j\rho_j +\eta_{k}\tau_{k}\left( h_{k}- \frac{1}{2}\tau_{k}\norm{d_{k}}_{\mD_{k}^{-1}}^2\right)} {\sum_{j=1}^{k}\eta_j\tau_j\rho_j}. 
\end{align*}

\end{proof}


\section{Proofs for Convergence Analysis}\label{sec:proof-conv-analysis}

Here we give another motivation for a variant of \Momo{} through convexity.  We discovered this interpretation of \Momo{} after reading the concurrent
work~\citep{Wang2023}.  

For this alternative derivation of \Momo{}, first let $\tau_k \geq 0$ be a free parameter, and consider a general momentum method with a preconditioner given by
\begin{align}\label{eq:SGD-M}
    d_k = \sum_{j=1}^k \rho_{j,k} \nabla f(x^j,s_j),\quad
    x^{k+1} = x^k - \tau_k \mD_k^{-1} d_k.
\end{align}
We can now view $x^{k+1}$ as a function of $\tau_k$, that is $x^{k+1}(\tau_k)$. Ideally we would like to choose $\tau_k$ so that $x^{k+1}$ is as close as possible to the optimum solution $x^*$, that is to minimize  $\norm{x^{k+1}(\tau_k)-x^*}_{\mD_k}^2$ in $\tau_k.$ This is general not possible because we do not know $x^*.$ But if we assume that $f(\cdot, s)$ is a convex function,  then 
we can minimize an upper bound of $\norm{x^{k+1}(\tau_k)-x^*}_{\mD_k}^2$ with respect to $\tau_k$. As we show next, this gives the adaptive term in the learning rate of \Momo{} if $\lb{k} = \bar{f}^k_*$.

\begin{restatable}{lemma}{lemminDstar}\label{lem:minDstar}
Let $f(\cdot, s)$ be convex for every $s$.  Let $h_k  :=  \bar f_k + \iprod{d_k}{x^k} -\gamma_k$ where $d_k, \bar{f}_k, $ and $\gamma_k$ are defined in~\eqref{eq:barf-d-gamma}. Consider the iterates given by~\eqref{eq:SGD-M} and let $x^* \in \arg \min_{x\in \R^d} f(x)$. Then, we have the upper bound
\begin{align}   \label{eq:minDconvex}
 \norm{x^{k+1} -x^*}_{\mD_{k}}^2 
& \leq \norm{x^k -x^*}_{\mD_k}^2- 2\tau_k (h_k-\rho_k\bar{f}^k_*)+\tau_k^2\norm{d_k}_{\mD_k^{-1}}^2.
\end{align}
The minimum of the right-hand side of \eqref{eq:minDconvex}, over the set $\tau_k \in \R_{\geq 0}$, is attained at
\begin{equation}\label{eq:taukopt}
    \bar \tau_k \;=\; \frac{(h_k- \rho_k\bar{f}^k_*)_+}{\norm{d_k}_{\mD_k^{-1}}^2}.
\end{equation}
\end{restatable}
\begin{proof}
Subtracting $x^*$ from both sides, taking norms and expanding the squares gives
\begin{align}   
\norm{x^{k+1} -x^*}_{\mD_k}^2 &= \norm{x^k -x^*}_{\mD_k}^2- 2\tau_k \dotprod{d_k,x^k -x^*} +\tau_k^2\norm{d_k}_{\mD_k^{-1}}^2. \label{eq:zo84ho8hz11} 
\end{align}
Denote  $\nabla f_j := \nabla f(x^j,s_j),~ f_j := f(x^j,s_j)$. Now using that
\begin{align}
   \dotprod{d_k,x^k-x^*} &= \sum_{j=1}^k \rho_{j,k} \dotprod{\nabla f_j,x^k-x^*}\nonumber \\
    & = \sum_{j=1}^k \rho_{j,k}\left( \dotprod{\nabla f_j,x^k-x^j} +\dotprod{\nabla f_j,x^j-x^*}\right) \nonumber \\
    &\geq  \sum_{j=1}^k \rho_{j,k}\left( \dotprod{\nabla f_j,x^k-x^j} + f_j-f(x^*,s_j)\right)\nonumber &\mbox{(by convexity of $f(\cdot,s_j)$)} \\ 
    & =  \dotprod{d_k,x^k} -\gamma_k +\sum_{j=1}^k \rho_{j,k} (f_j -f(x^*,s_j)) = h_k - \rho_k \bar{f}^k_*.  \label{eq:dkdotx-xstar11}
\end{align}

Using~\eqref{eq:dkdotx-xstar11} in~\eqref{eq:zo84ho8hz11} gives
\begin{align*}   
 \norm{x^{k+1} -x^*}_{\mD_{k}}^2 & = \norm{x^k -x^*}_{\mD_k}^2- 2\tau_k \dotprod{d_k,x^k -x^*} +\tau_k^2\norm{d_k}_{\mD_k^{-1}}^2 \nonumber \\
& \leq \norm{x^k -x^*}_{\mD_k}^2- 2\tau_k (h_k-\rho_k\bar{f}^k_*)+\tau_k^2\norm{d_k}_{\mD_k^{-1}}^2.
\end{align*}
If we now minimize the right-hand side of the above in $\tau_k$, but restricted to $\tau_k\geq 0$, we arrive at~\eqref{eq:taukopt}.     
\end{proof}

Inequality \eqref{eq:minDconvex} holds for any choice of $\tau_k\geq 0$ in \eqref{eq:SGD-M}, in particular for $\tau_k = \min\{\frac{\alpha_k}{\rho_k}, \frac{(h_k- \rho_k\bar{f}^k_*)_+}{\norm{d_k}_{\mD_k^{-1}}^2}\}$. This choice for $\tau_k$ is equal to \Momo{} for $\lambda=0$ and $\lb{k} = \bar{f}^k_*$. As a consequence, we we can prove a descent lemma for \Momo{}.

\lemdescent*
\begin{proof}   
We again denote $h_k  =  \bar f_k + \iprod{d_k}{x^k} -\gamma_k$.
First, assume $\tau_k= \frac{(h_k- \rho_k\bar{f}^k_*)_+}{\norm{d_k}_{\mD_k^{-1}}^2}$. Inserting this $\tau_k$ back in \eqref{eq:minDconvex} we have that 
\begin{align}   
 \norm{x^{k+1} -x^*}_{\mD_{k}}^2 
& \leq \norm{x^k -x^*}_{\mD_k}^2- 2\frac{(h_k-\rho_k\bar{f}^k_*)_+}{\norm{d_k}_{\mD_k^{-1}}^2}(h_k-\rho_k\bar{f}^k_*)+\frac{(h_k-\rho_k\bar{f}^k_*)_+^2}{\norm{d_k}_{\mD_k^{-1}}^2} \nonumber \\
& = \norm{x^k -x^*}_{\mD_k}^2- \frac{(h_k-\rho_k\bar{f}^k_*)_+^2}{\norm{d_k}_{\mD_k^{-1}}^2} \nonumber\\ 
&= \norm{x^k -x^*}_{\mD_k}^2- \tau_k (h_k-\rho_k\bar{f}^k_*)_+. \label{eqn:descent-adaptive}
\end{align}
Here we used that $a(a)_+ = (a)_+^2$ for all $a\in \R$.

If we have $\tau_k = \frac{\alpha_k}{\rho_k}$, then from \eqref{eq:minDconvex} we get
\begin{align}   
 \norm{x^{k+1} -x^*}_{\mD_{k}}^2 & \leq\norm{x^k -x^*}_{\mD_k}^2 + \frac{\alpha_k}{\rho_k}\big[-2 (h_k-\rho_k\bar{f}^k_*)+\frac{\alpha_k}{\rho_k}\norm{d_k}_{\mD_k^{-1}}^2\big].
\end{align}
In this case $\frac{\alpha_k}{\rho_k} \leq \frac{(h_k- \rho_k\bar{f}^k_*)_+}{\norm{d_k}_{\mD_k^{-1}}^2}$ and hence $\frac{\alpha_k}{\rho_k} \norm{d_k}_{\mD_k^{-1}}^2  \leq (h_k- \rho_k\bar{f}^k_*)_+$. 
Further, it must hold $(h_k- \rho_k\bar{f}^k_*)= (h_k- \rho_k\bar{f}^k_*)_+$ as $\alpha_k>0$. We get
\begin{align}   
 \norm{x^{k+1} -x^*}_{\mD_{k}}^2 & \leq\norm{x^k -x^*}_{\mD_k}^2 - \frac{\alpha_k}{\rho_k}(h_k-\rho_k\bar{f}^k_*)_+ \nonumber \\
 &= \norm{x^k -x^*}_{\mD_k}^2 - \tau_k(h_k-\rho_k\bar{f}^k_*)_+ \quad (\tau_k=\frac{\alpha_k}{\rho_k}). \label{eqn:descent-alpha}
\end{align}
Now, if $\tau_k = \min\{\frac{\alpha_k}{\rho_k}, \frac{(h_k- \rho_k\bar{f}^k_*)_+}{\norm{d_k}_{\mD_k^{-1}}^2}\}$, either \eqref{eqn:descent-adaptive} or \eqref{eqn:descent-alpha} is true, and hence we have
\begin{align*}
    \norm{x^{k+1} -x^*}_{\mD_{k}}^2  \leq \norm{x^k -x^*}_{\mD_k}^2 - \tau_k(h_k-\rho_k\bar{f}^k_*)_+.
\end{align*}
\end{proof}
Recall the interpolation condition \eqref{eq:interpolation}, given by
\begin{align*}
    f(x^*,s) = \inf_x f(x,s) = f^* \quad \text{for all } s \in \mathcal{D}.
\end{align*}
%
\thmconvex*
%
\begin{proof}
    Recall that \cref{alg:momo} is \cref{lem:momo-general-update} with $\rho_k=1,~\mD_k= \Id$ and $\lambda=0$.
    The key quantity is $h_k = \bar f_k + \iprod{d_k}{x^k} - \gamma_k$. 
    Let us denote $g_k = \nabla f(x^k,s_k)$.
    Further, denote with $\mathcal{F}_k$ the $\sigma$-algebra generated by $\{s_1,\dots,s_{k-1}\}$.
    
    \textbf{Step 1.} We first show by induction that $h_k - f^* \geq 0$ for all $k\in \N$. For $k=1$ we have $h_1 = f(x^1, s_1) \geq f^*$ due to \eqref{eq:interpolation}. 
    Now, for $k\geq 2$, assume that $h_{k-1} - f^* \geq 0$. Rewrite as
\begin{align*}
    h_k &= \beta \big[\bar f_{k-1} + \iprod{d_{k-1}}{x^k} - \gamma_{k-1}\big] + (1-\beta) \big[f(x^k,s_k) + \iprod{g_k}{x^k} - \iprod{g_k}{x^k} \big] \\
    &=\beta \big[\bar f_{k-1} + \iprod{d_{k-1}}{x^{k-1}}  - \gamma_{k-1} + \iprod{d_{k-1}}{x^k-x^{k-1}}\big] + (1-\beta) f(x^k,s_k) \\
    &= \beta h_{k-1} + \beta \iprod{d_{k-1}}{x^k- x^{k-1}}  +(1-\beta) f(x^k,s_k).
\end{align*}
Using the update rule $x^k = x^{k-1} - \tau_{k-1} d_{k-1}$  in the above gives
 \begin{align}\label{eqn:hk-equality}
    h_k = \beta (h_{k-1} -\tau_{k-1} \|d_{k-1}\|^2)  +(1-\beta) f(x^k,s_k).
\end{align}
Recall that $\tau_k =\frac{(h_k-\lb{k})_+}{\|d_k\|^2}$ due to $\alpha_k=+ \infty$. 
Hence, 
$$\tau_{k-1} \|d_{k-1}\|^2 = (h_{k-1} - \lb{k-1})_+ = (h_{k-1} - f^*)_+ = h_{k-1} - f^*$$ where the last equality is the induction hypothesis. Re-arranging the above we get
\begin{align}
    h_{k-1} -\tau_{k-1} \|d_{k-1}\|^2 = f^*.
\end{align}
Plugging this equality into \eqref{eqn:hk-equality} gives
\begin{align*}
    h_k = \beta f^* + (1-\beta) f(x^k,s_k) \geq f^*,
\end{align*}
due to $\beta\in[0,1)$ and $f(x^k,s_k)\geq f^*$. This completes the induction, and we have further shown that 
\begin{align}\label{eqn:lower-bound-hk}
    h_k - f^* = (1-\beta) \big( f(x^k,s_k) - f^*\big).
\end{align} 

\textbf{Step 2.} Due to \eqref{eq:interpolation} and $\rho_k=1$, it holds $\bar f_k^*=f^*=\lb{k}$. Hence, the assumptions of \cref{lem:descent} are satisfied and we can apply \eqref{eqn:momo-descent}, which implies in particular that the iterates $(x^k)$ are almost surely contained in the bounded set $B$. 
By assumption, we conclude that $\E{\|g_j\|^2 ~\vert~ \mathcal{F}_k } \leq G^2$ for all $j \leq k$. Using Jensen for the discrete probability measure induced by $\rho_{j,k}$, we have
\begin{align*}
    \|d_k\|^2 = \|\sum_{j=1}^k \rho_{j,k} g_j\|^2 \leq \sum_{j=1}^k \rho_{j,k} \|g_j\|^2.
\end{align*}
Thus, we conclude for the conditional expectation that $\E{\|d_k\|^2~\vert~ \mathcal{F}_k} \leq G^2$. 
By Step 1, we have $\tau_k = \frac{h_k-f^*}{\|d_k\|^2}$.
We will use next that $(x,y)\mapsto x^2/y$ is convex for $x\in\R,y>0$. 
From \eqref{eqn:descent-adaptive} and applying conditional expectation, we have
\begin{align*}
    \E{\|x^{k+1} - x^*\|^2~\vert~ \mathcal{F}_k} &\leq \|x^{k} - x^*\|^2 - \E{\frac{(h_k - f^*)^2}{\|d_k\|^2}~\vert~ \mathcal{F}_k} \\
    &\leq \|x^{k} - x^*\|^2 - \frac{\E{h_k - f^*~\vert~ \mathcal{F}_k}^2}{\E{\|d_k\|^2~\vert~ \mathcal{F}_k}} \\
    &\overset{\eqref{eqn:lower-bound-hk}}{=} \|x^{k} - x^*\|^2 - \frac{(1-\beta)^2\E{f(x^k,s_k)- f^*~\vert~ \mathcal{F}_k}^2}{\E{\|d_k\|^2~\vert~ \mathcal{F}_k}}\\
    &\leq \|x^{k} - x^*\|^2 - \frac{(1-\beta)^2(f(x^k) - f^*)^2}{G^2}.
\end{align*}

\textbf{Step 3.} Taking full expectation, using the law of total expectation, suming over $k=1,\dots,K$, dividing by $K$ and re-arranging gives
\begin{align}\label{eq:zelfhzlo8hefz}
    \frac{1}{K}\sum_{k=1}^K \E{(f(x^k)-f^*)^2} \leq \frac{G^2\|x^1-x^*\|^2}{K(1-\beta)^2}.
\end{align}
Now, due to Jensen's inequality we have $\E{(f(x^k)-f^*)^2} \geq \E{f(x^k)-f^*}^2$ and because the square-root is concave, it holds
\begin{align*}
    \frac{1}{K}\sum_{k=1}^K \E{f(x^k)-f^*} \leq \sqrt{\frac{1}{K}\sum_{k=1}^K \E{f(x^k)-f^*}^2}.
\end{align*}
Using the above together with~\eqref{eq:zelfhzlo8hefz}, we obtain
\begin{align*}
    \min_{k=1,\dots,K}  \E{f(x^k)-f^*} \leq \frac{1}{K}\sum_{k=1}^K \E{f(x^k)-f^*} \leq \frac{G\|x^1-x^*\|}{\sqrt{K}(1-\beta)}.
\end{align*}
\end{proof}

The above result is basically identical to \citep[Thm.\ C.1]{Loizou2021}, but also allowing for momentum. We make two remarks: the best constant is clearly achieved by $\beta=0$, i.e.\ no momentum. While empirically, momentum helps in most cases, we can not show a theoretical improvement at this time. Second, we do not need to assume bounded gradient norms as done in \citep{Loizou2021}, because this follows from the descent property \cref{lem:descent}. However, this improvement could be achieved analogously for the the proof of \citep{Loizou2021} based on the techniques we developed here.
%

\section{Additional Information on Experiments}\label{sec:appendix-numerics-info}
\subsection{Additional Plots}

\begin{figure}[h]
    \centering
    \begin{subfigure}{0.325\textwidth}  
        \includegraphics[width=0.99\textwidth]{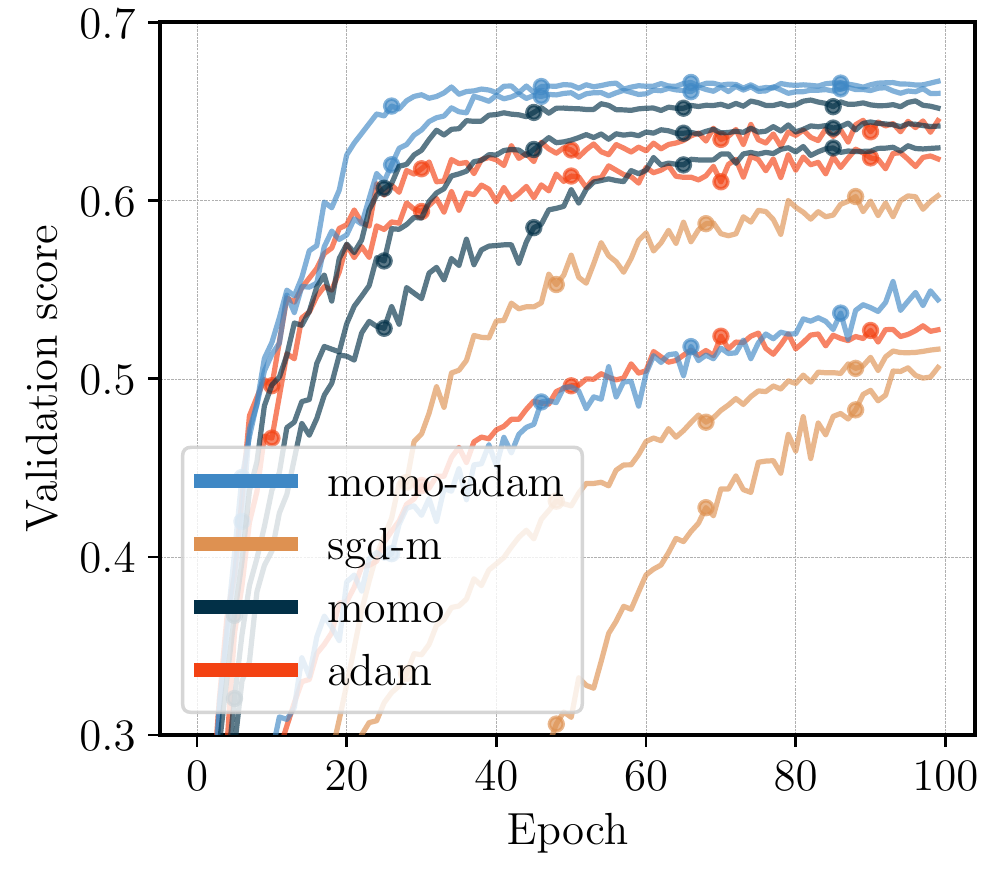}
        \caption{\texttt{ResNet110} for \texttt{CIFAR100}}
        \label{fig:resnet110_all_val}
    \end{subfigure}
    \begin{subfigure}{0.325\textwidth}  
        \includegraphics[width=0.99\textwidth]{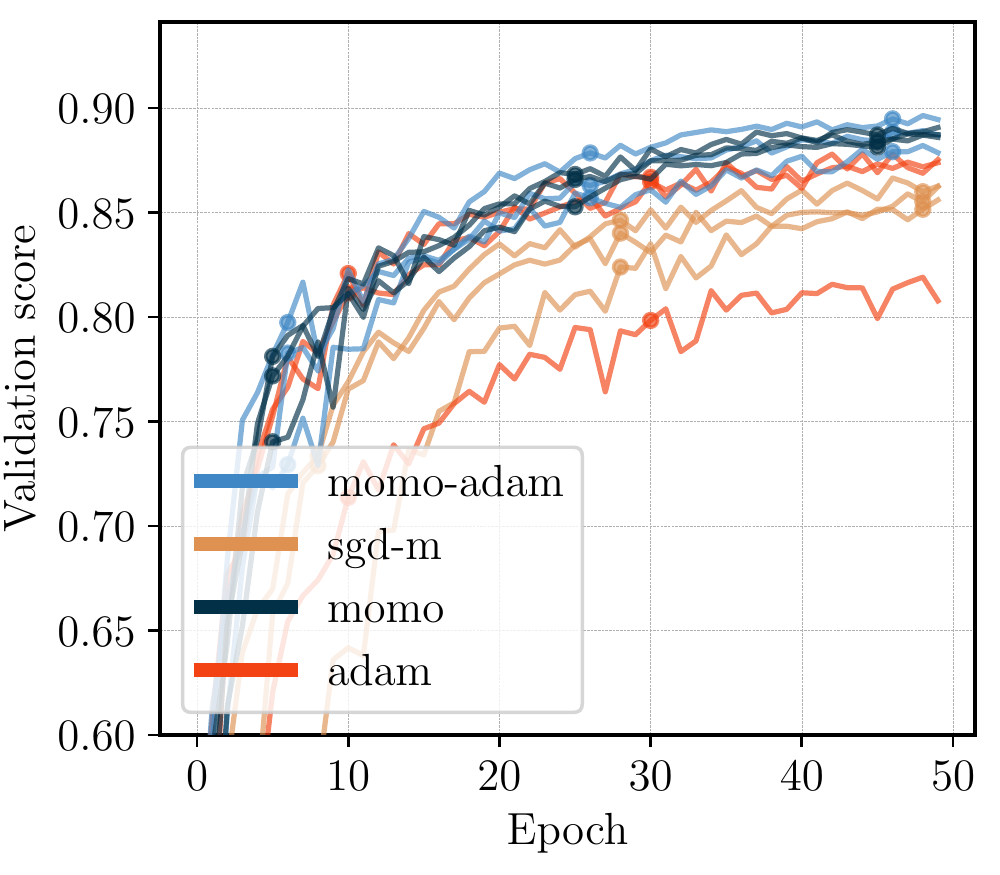}
        \caption{\texttt{ResNet20} for \texttt{CIFAR10}}
        \label{fig:resnet20_all_val}
    \end{subfigure}
    \begin{subfigure}{0.325\textwidth}  
        \includegraphics[width=0.99\textwidth]{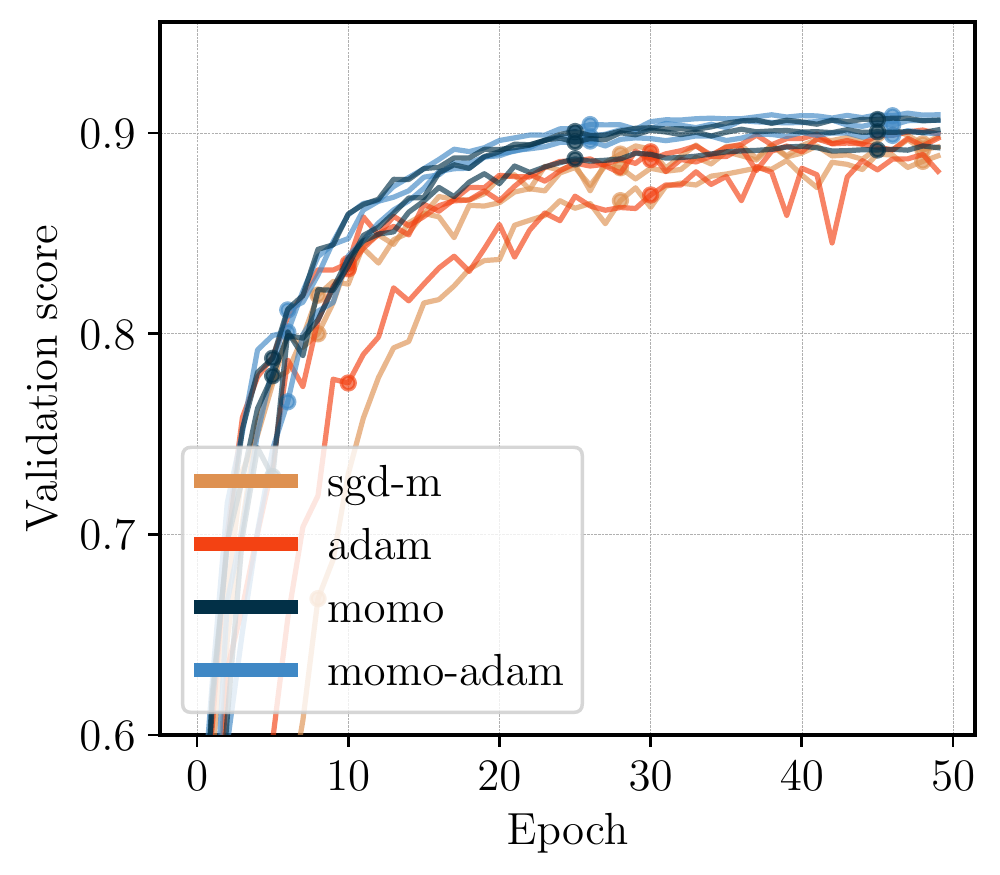}
        \caption{\texttt{VGG16} for \texttt{CIFAR10}}
        \label{fig:vgg16_all_val}
    \end{subfigure}
\caption{Validation score over training, we plot, for each method, the three choices of $\alpha_0$  that lead to the best validation score (compare to \cref{fig:stability_val_score,fig:stability_appendix}).}
\label{fig:all_val}
\end{figure}
\begin{figure}[h]
    \centering
    \begin{subfigure}{0.325\textwidth}  
        \includegraphics[width=0.99\textwidth]{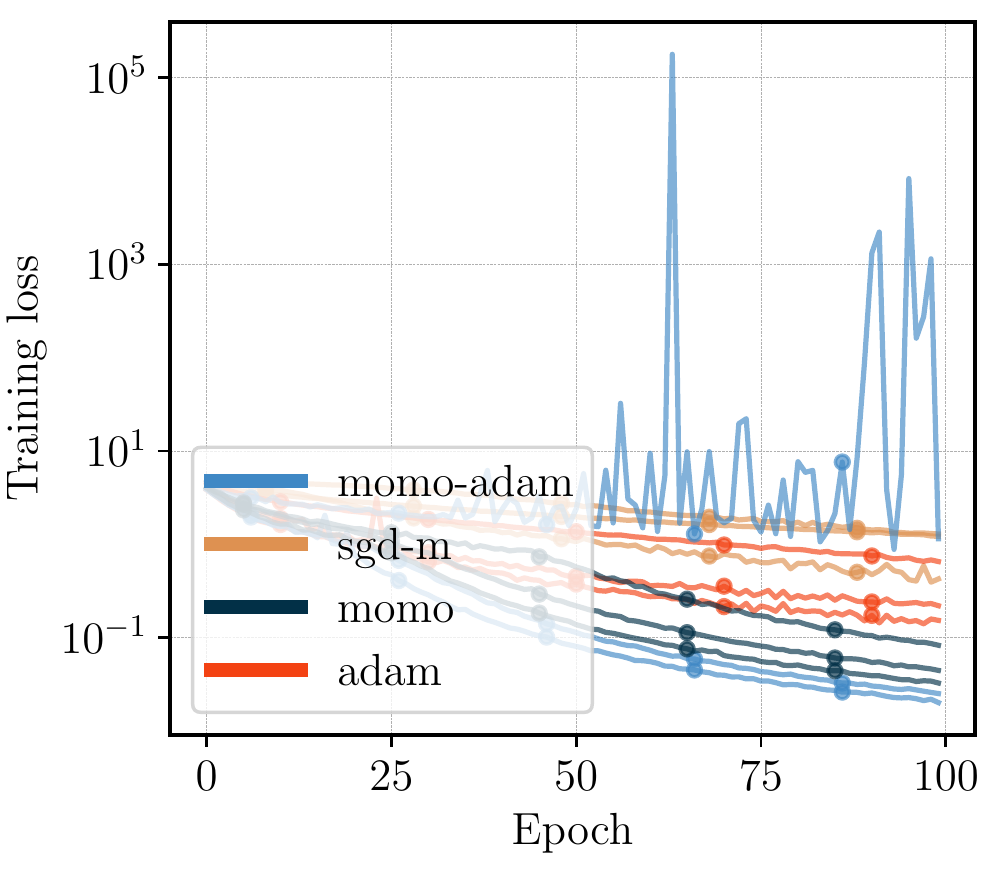}
        \caption{\texttt{ResNet110} for \texttt{CIFAR100}}
        \label{fig:resnet110_all_loss}
    \end{subfigure}
    \begin{subfigure}{0.325\textwidth}  
        \includegraphics[width=0.99\textwidth]{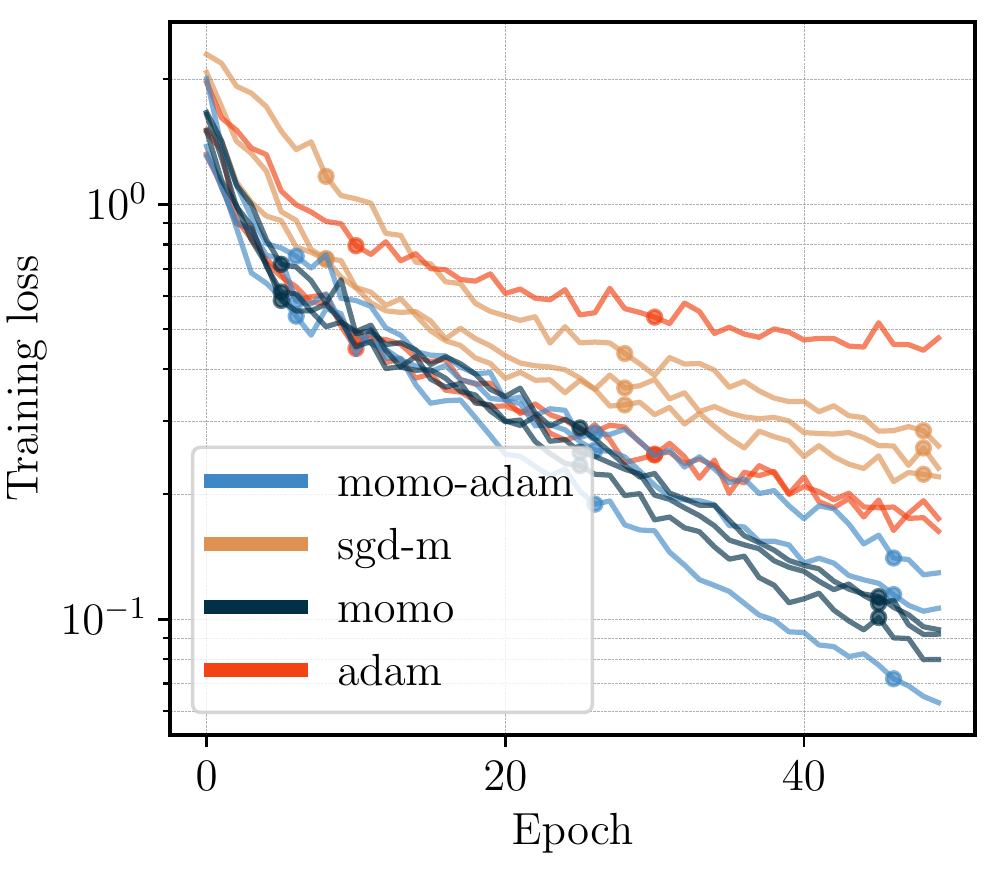}
        \caption{\texttt{ResNet20} for \texttt{CIFAR10}}
        \label{fig:resnet20_all_loss}
    \end{subfigure}
    \begin{subfigure}{0.325\textwidth}  
        \includegraphics[width=0.99\textwidth]{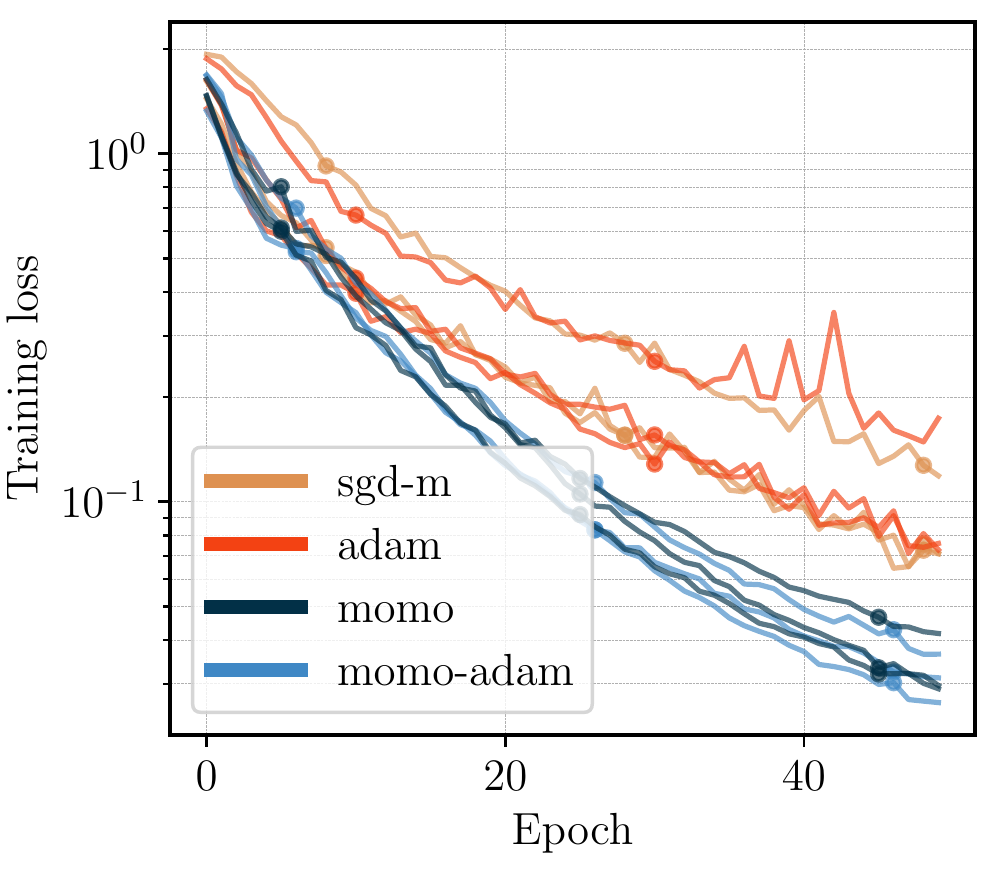}
        \caption{\texttt{VGG16} for \texttt{CIFAR10}}
        \label{fig:vgg16_all_loss}
    \end{subfigure}
\caption{Training loss over training, we plot, for each method, the three choices of $\alpha_0$  that lead to the best validation score.}
\label{fig:all_loss}
\end{figure}
\begin{figure}[h]
    \centering
    \begin{subfigure}[b]{0.32\textwidth}  
        \centering 
        \includegraphics[width=0.99\textwidth]{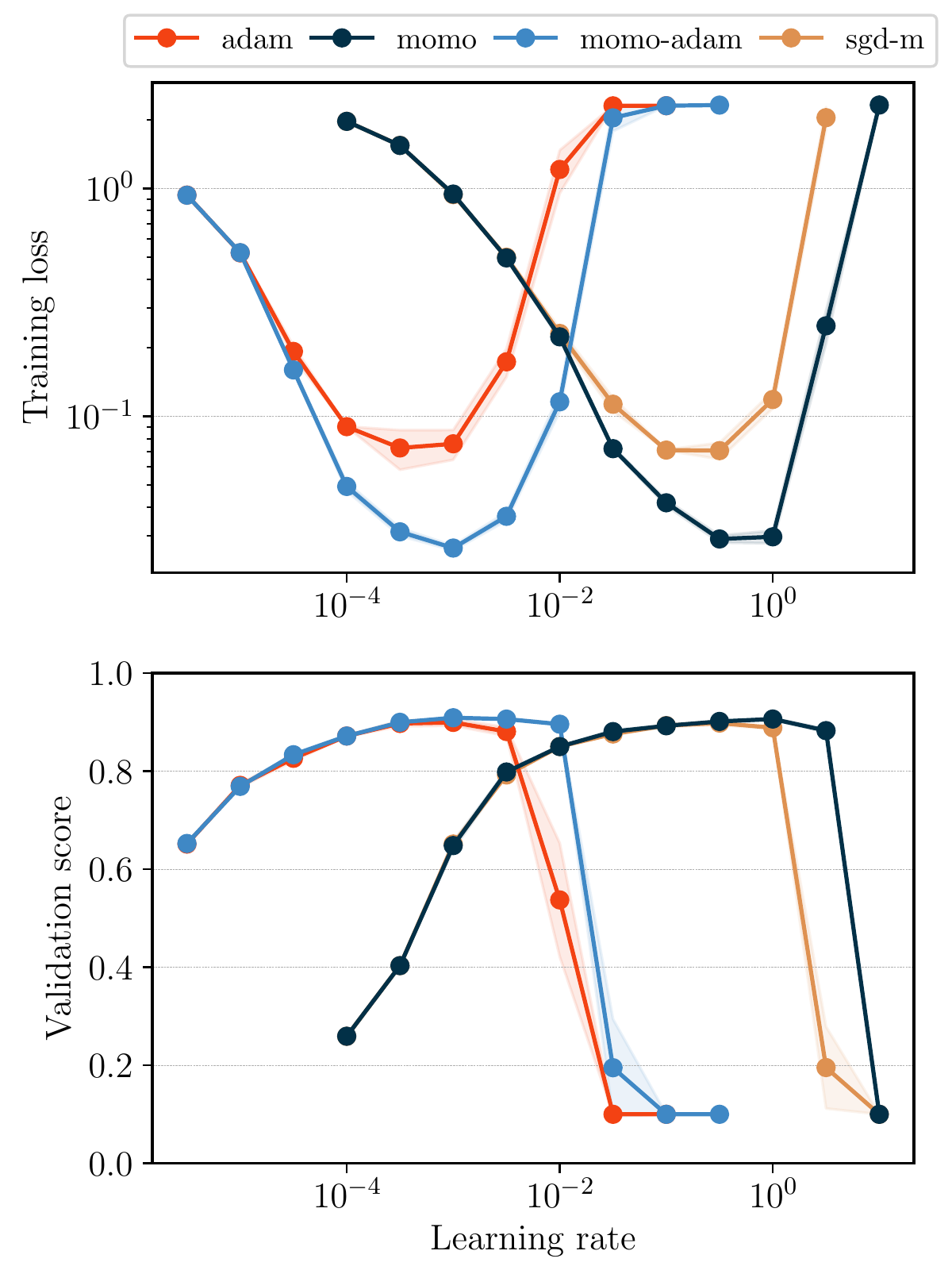}
        \caption{\texttt{VGG16} for \texttt{CIFAR10}}
    \end{subfigure}
    \begin{subfigure}[b]{0.32\textwidth}  
        \centering 
        \includegraphics[width=0.99\textwidth]{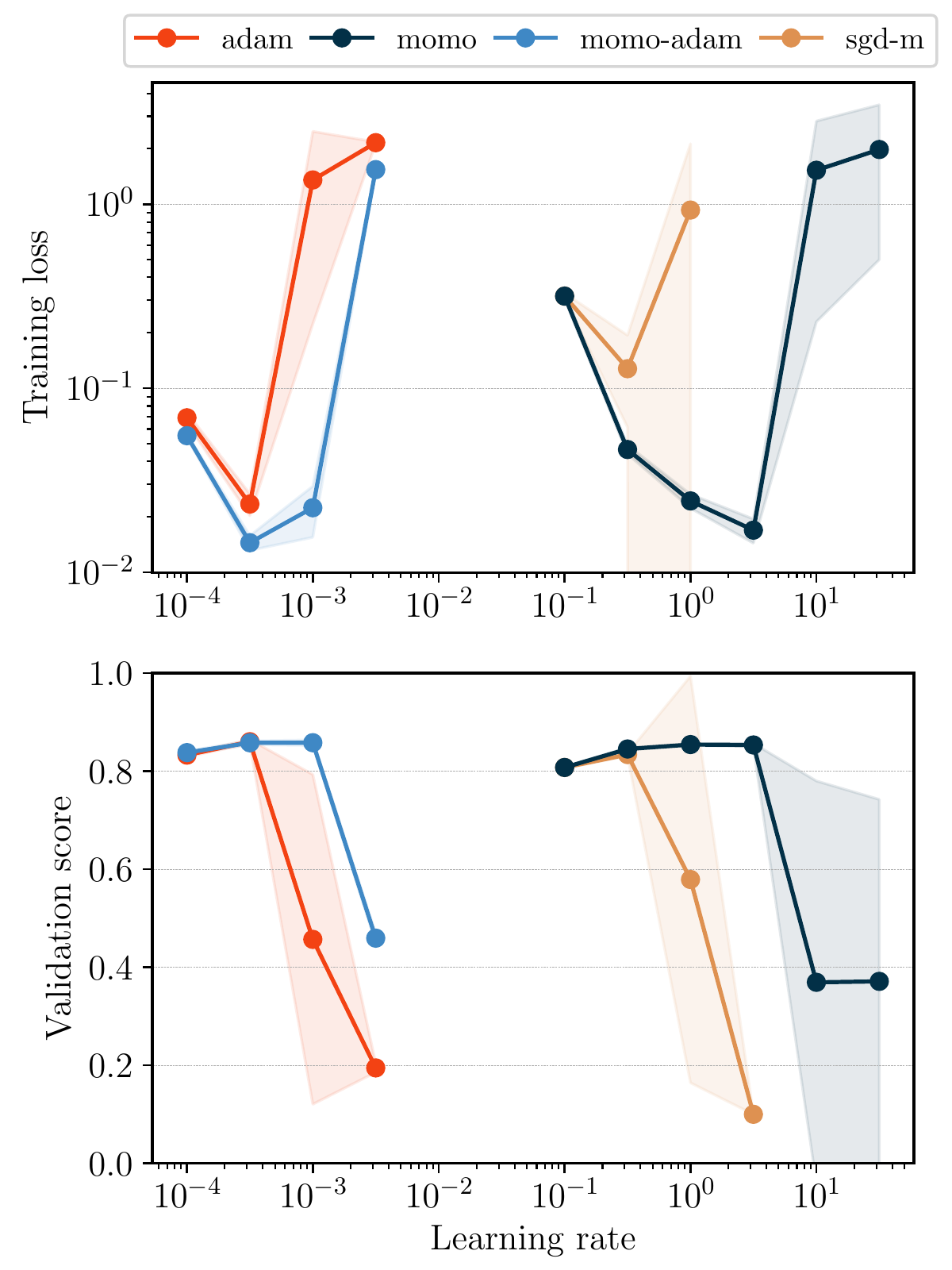}
        \caption{\texttt{ViT} for \texttt{CIFAR10}}
    \end{subfigure}
    \begin{subfigure}[b]{0.32\textwidth}  
        \centering 
        \includegraphics[width=0.99\textwidth]{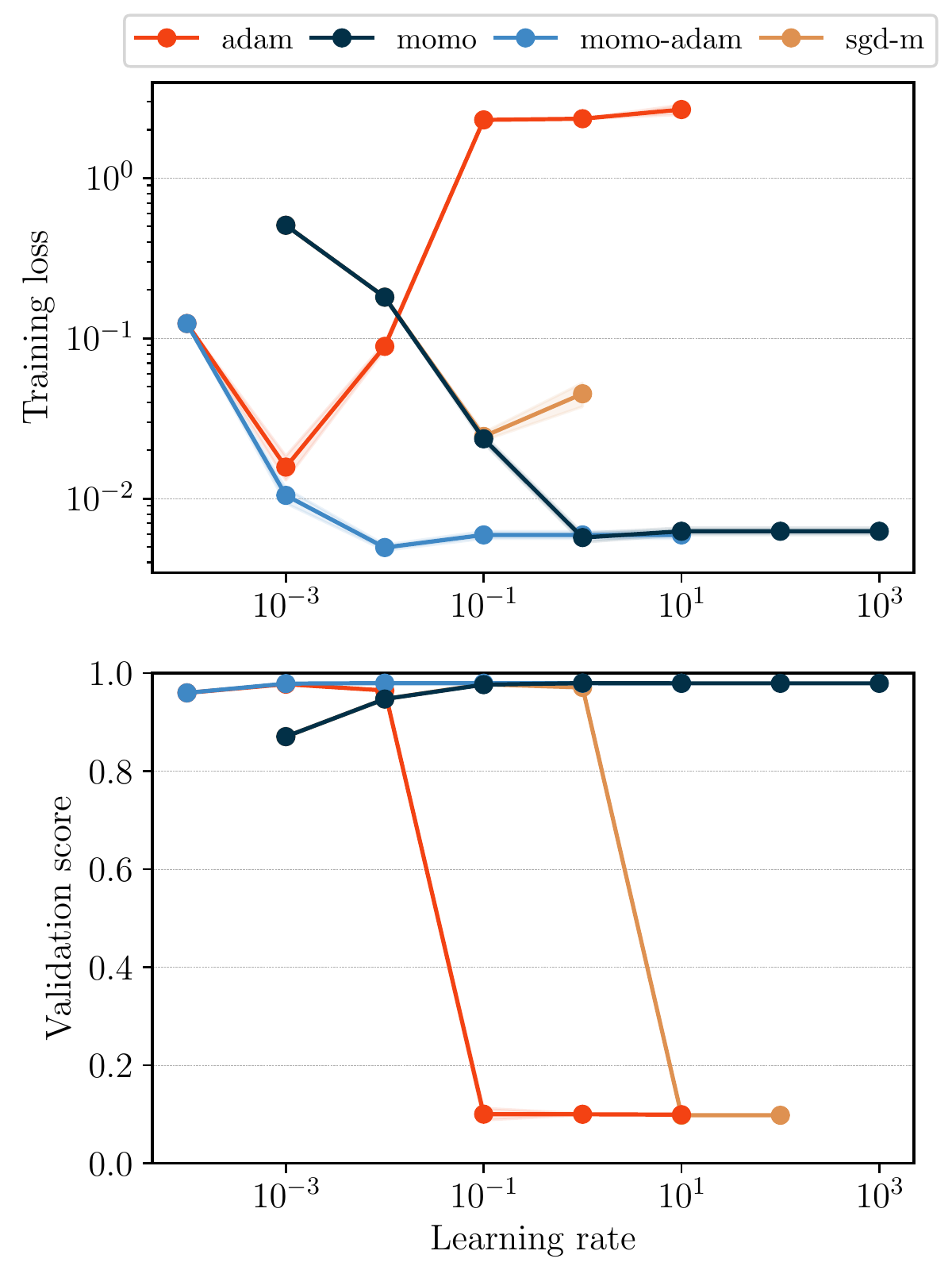}
        \caption{\texttt{MLP} for \texttt{MNIST}}
    \end{subfigure}
\caption{Training loss (top row) and validation accuracy (bottom row) after a fixed number of epochs, for varying (constant) learning rate $\alpha_0$.}
\label{fig:stability_appendix}
\end{figure}
%
%
%
\begin{figure}[h]
    \centering
    \begin{subfigure}[b]{0.58\textwidth}  
        \includegraphics[width=0.99\textwidth]{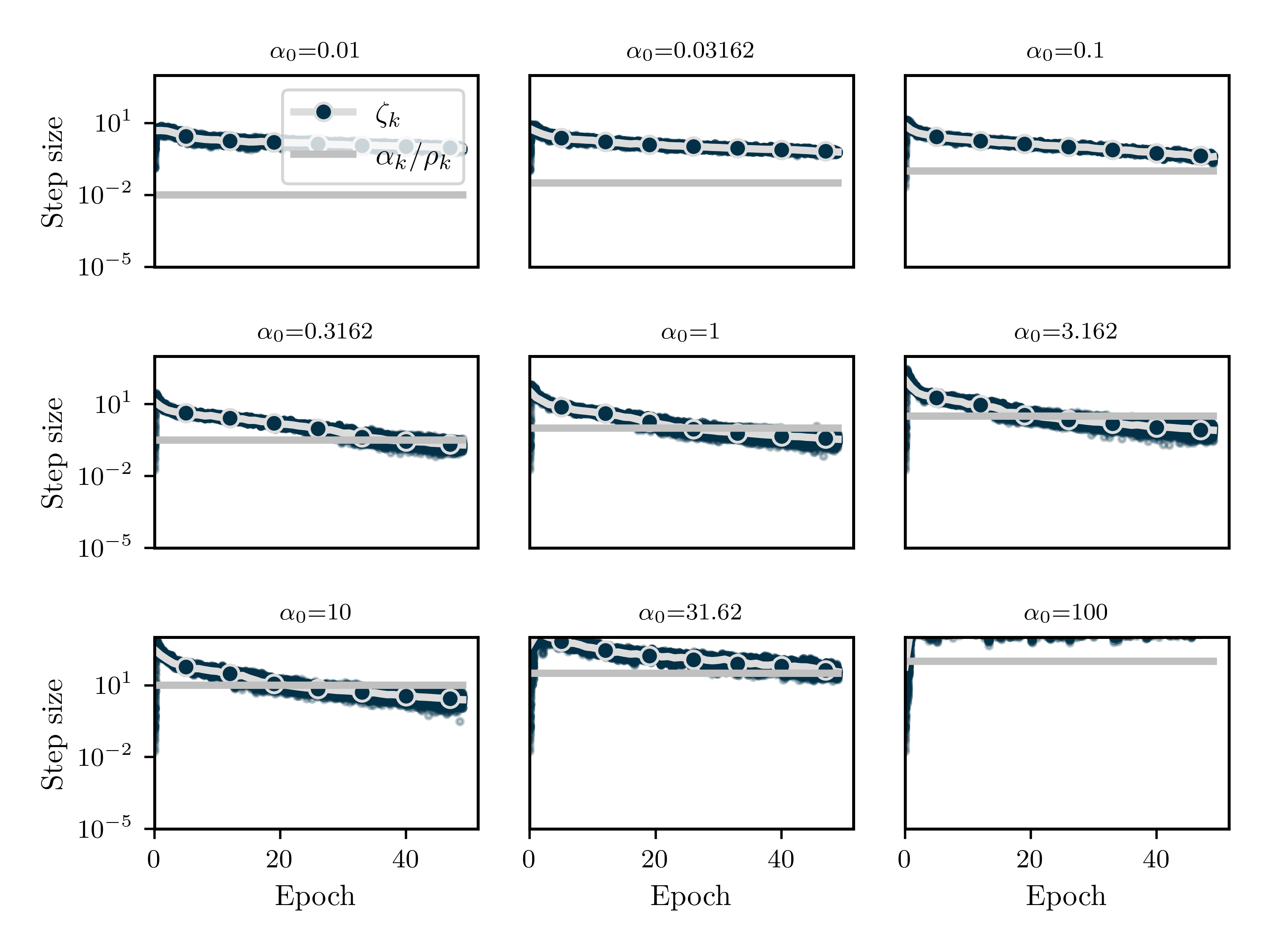}
        \caption{\Momo{}}
        \label{fig:resnet20_step_sizes_momo}
    \end{subfigure}
    \begin{subfigure}[b]{0.38\textwidth}
        \includegraphics[width=0.99\textwidth]{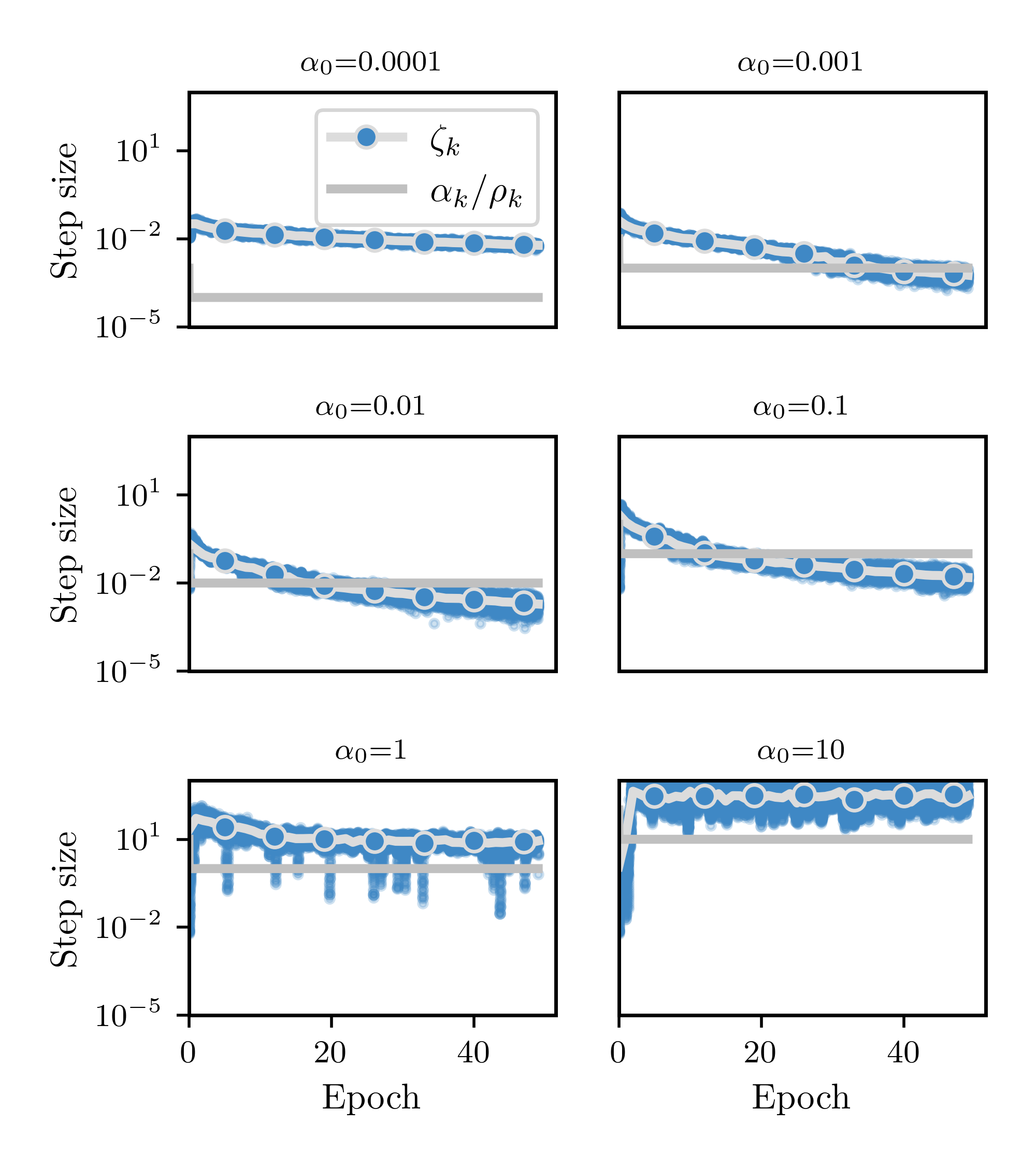}
        \caption{\MomoAdam{}}
        \label{fig:resnet20_step_sizes_momo_adam}
    \end{subfigure}
\caption{\texttt{ResNet20} for \texttt{CIFAR10}. Adaptive learning rate of \Momo{} (left) and \MomoAdam{} (right). The colored dots represent the term $\zeta_k$ in each iteration. The grey line represents the user-specified learning rate $\alpha_k/\rho_k$ (note that $\rho_k=1$ for \Momo{} and $\rho_k\approx 1$ except for the first few iterations in \MomoAdam{}). The minimum of the grey line and the dots is the adaptive learning rate $\tau_k=\min\{\frac{\alpha_k}{\rho_k}, \zeta_k\}$ in each iteration. The silver line with colored markers is the median over the values of $\zeta_k$ in each epoch.}
\label{fig:resnet20_step_sizes}
\end{figure}

\begin{table}
\caption{Validation score (with one standard deviation) for the best learning rate choice for each method among the ones displayed in \cref{sec:experiments}. Symbol ``$^*$" indicates usage of online lower bound, otherwise \texttt{MoMo}(\texttt{-Adam}) used with $\lb{k}=0$. Bold indicates the best method (for experiments with multiple seeds, we only mark in bold if the advantage is outside of standard deviation).}
\resizebox{\textwidth}{!}{
    \begin{tabular}{|c|c|c|c|c|}
        \hline
        & \texttt{MoMo} & \texttt{MoMo-Adam} & \texttt{SGD-M} & \texttt{Adam(W)} \\
        \hline
        \texttt{ResNet110} for \texttt{CIFAR100} & 65.21$~\pm 1.61$ & \textbf{66.71$~\pm 0.31$} & 60.28$~\pm 0.36$ & 64.5$~\pm 1.14$ \\
        \hline
        \texttt{ResNet20} for \texttt{CIFAR10} & 89.07$~\pm 0.2$ & \textbf{89.45$~\pm 0.17$} & 86.27$~\pm 0.67$ & 87.54$~\pm 0.26$ \\
        \hline
        \texttt{ViT} for \texttt{CIFAR10} & 85.43$~\pm 0.19$ & 85.81$~\pm 0.57$ & 83.39$~\pm 0.28$ & 86.02$~\pm 0.44$ \\
        \hline
        \texttt{VGG16} for \texttt{CIFAR10} & 90.64$~\pm 0.18$ & \textbf{90.9$~\pm 0.17$} & 89.81$~\pm 0.43$ & 89.95$~\pm 0.67$ \\
        \hline
        \texttt{MLP} for \texttt{MNIST} & \textbf{97.97$~\pm 0.08$} & 97.96$~\pm 0.12$ & 97.73$~\pm 0.12$ & 97.75$~\pm 0.06$ \\
        \hline
        \texttt{DLRM} for \texttt{Criteo}&  78.83 $~\pm 0.038$ & 78.98 $~\pm 0.036$  & 78.81 $~\pm 0.041$ & \textbf{79.05} $~\pm 0.014$ \\
        \hline
        \texttt{ResNet18 for Imagenet32} & \textbf{47.66}$^*$ & 47.54$^*$ & 47.38 & 46.98\\
        \hline
        \texttt{ResNet18 for Imagenet-1k} & \textbf{69.68} & N/A & 69.57 & N/A\\
        \hline
        \texttt{IWSLT14} (dp 0.1)& N/A & \textbf{33.63}$^*$ & N/A & 32.56 \\
        \hline
        \texttt{IWSLT14} (dp 0.3)& N/A & \textbf{35.34}$^*$ & N/A & 34.97 \\
        \hline
        \texttt{ViT} for \texttt{Imagenet-1k}& N/A & \textbf{73.83$~\pm 0.36$}& N/A & 72.83$~\pm 0.51$ \\
			\hline
    \end{tabular}}
\label{table:best_scores}
\end{table}
\begin{figure}[h]
    \centering
    \begin{subfigure}[b]{0.55\textwidth} 
        \centering
        \includegraphics[width=0.99\textwidth]{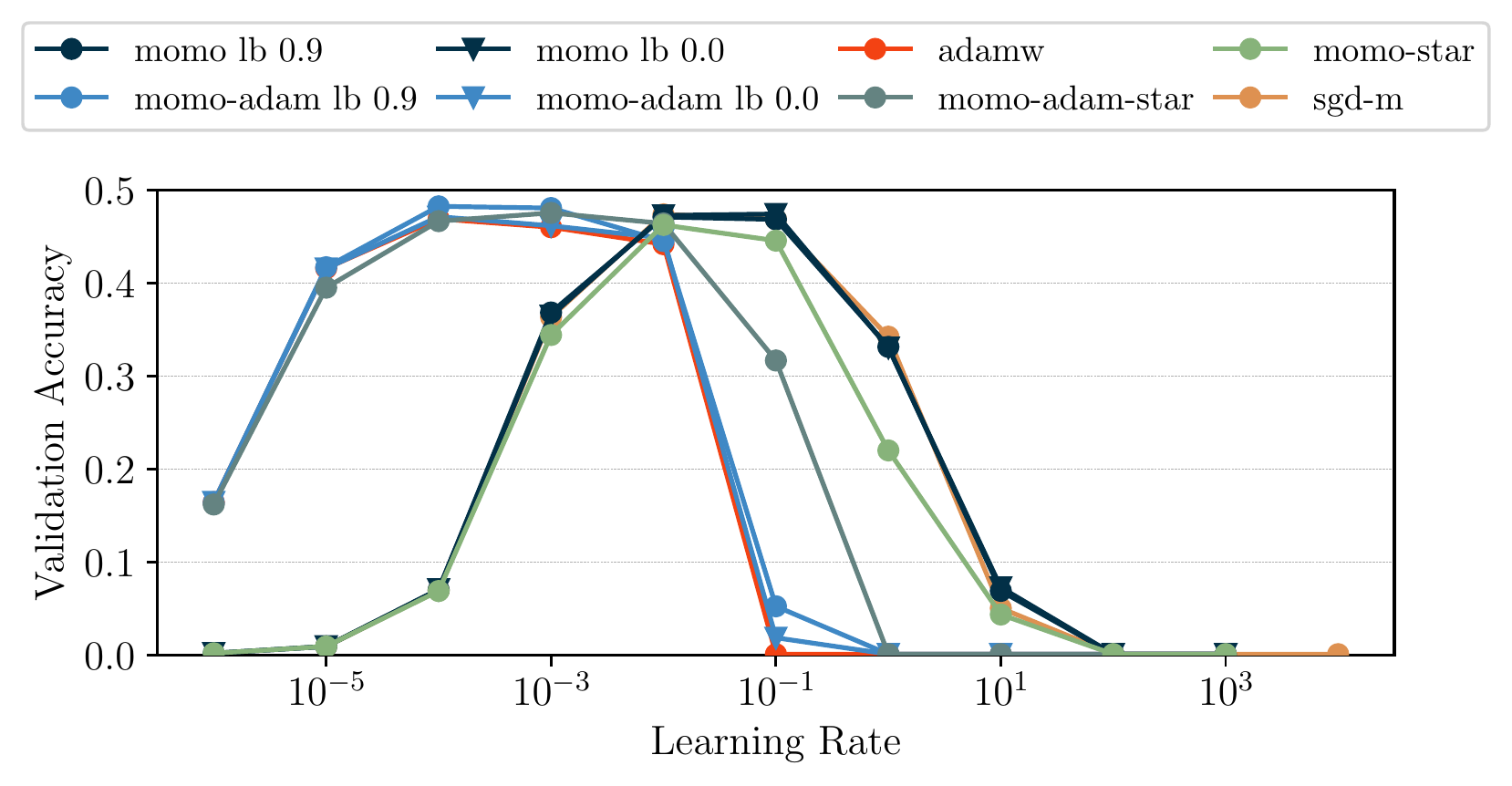}
        \caption{\texttt{ResNet18} for \texttt{Imagenet32}}
        \label{fig:sensitivity_imagenet_wd1e-4}
    \end{subfigure}
    \begin{subfigure}[b]{0.42\textwidth}
        \centering
        \includegraphics[width=0.99\textwidth]{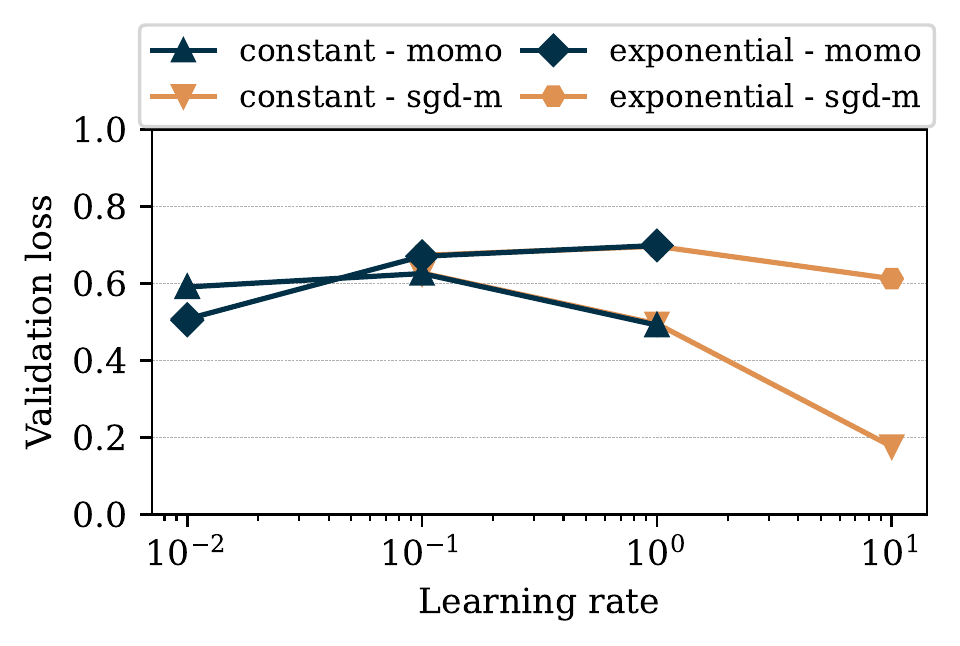}
        \caption{\texttt{ResNet18} for \texttt{Imagenet-1k}}
    \label{fig:sensitivity_imagenet_1k}
    \end{subfigure}
    \caption{Left: Validation accuracy of a \texttt{ResNet18} for \texttt{Imagenet32} with weight decay $\lambda=10^{-4}$. 
    Right: Validation accuracy of a \texttt{ResNet18} for \texttt{Imagenet-1k}, with standard exponential learning rate schedule (decay factor 10 at epochs 30 and 60) and constant learning rate schedule.}
    \label{fig:additional-imagenet}
\end{figure}
%
%
\subsection{Experimental Setup of \cref{sec:stability-exps}}\label{sec:appendix-info-stability}
We set the momentum parameter $\beta=0.9$ for \Momo{} and \SGDM{}, and $(\beta_1,\beta_2)=(0.9,0.999)$ for \MomoAdam{} and \Adam{} respectively. We do not use weight decay, i.e.\ $\lambda=0$.

For \SGDM{} we set the dampening parameter (in \texttt{Pytorch}) equal to the momentum parameter $0.9$. Like this, \SGDM{} does an exponentially-weighted average of past gradients and hence is comparable to \Momo{} for identical learning rate and momentum. Setting $\texttt{dampening}=0.9$ is equivalent to running with $\texttt{dampening}=0$ and a ten times smaller learning rate. For all other hyperparameters we use the \texttt{Pytorch} default values for \Adam{} and \SGDM{} (unless explicitly stated otherwise).

%
\subsection{Models and Datasets} \label{sec:appendix-models-datasets}
%
\begin{description}[listparindent=0pt,leftmargin=1.0em,]
\item[\texttt{ResNet} for \texttt{CIFAR}] \hfill \citep{He2016}~\\[.5em]
Used for \texttt{ResNet20} for \texttt{CIFAR10} and \texttt{ResNet110} for \texttt{CIFAR100}. We adapt the last layer output size to $\{10,100\}$ according to the used dataset. We run 50 epochs for \texttt{ResNet20} and $100$ epochs for \texttt{ResNet110}, both with batch size 128.

\vspace{1ex}
\begin{tabular}{@{}ll@{}}
Model & \small \url{https://github.com/akamaster/pytorch_resnet_cifar10/blob/master/resnet.py}
\end{tabular}

\item[\texttt{VGG16} for \texttt{CIFAR10}] \hfill \citep{Simonyan2015}~\\[.5em] 
A deep network with 16 convolutional layers. We run 50 epochs with batch size 128.

\vspace{1ex}
\begin{tabular}{@{}ll@{}}
Model & \small \url{https://github.com/chengyangfu/pytorch-vgg-cifar10/blob/master/vgg.py}
\end{tabular}

\item[\texttt{ViT} for \texttt{CIFAR10}] \hfill \citep{Dosovitskiy2021}~\\[.5em] 
A small vision transformer, based on the hyperparameter setting proposed in \url{github.com/kentaroy47/vision-transformers-cifar10}. In particular, we set the patch size to four. We run 200 epochs with batch size 512.

\vspace{1ex}
\begin{tabular}{@{}ll@{}}
Model & \small \url{https://github.com/lucidrains/vit-pytorch}
\end{tabular}

\item[\texttt{ResNet18} for \texttt{Imagenet32}] \hfill \citep{He2016}~\\[.5em] 
\texttt{Imagenet32} is a downsampled version of \texttt{Imagenet-1k} to images of $32\times32$ pixels. We adapt the last layer output size to $1000$. We run 45 epochs with batch size 128.

\vspace{1ex}
\begin{tabular}{@{}ll@{}}
Model & \small \url{https://github.com/kuangliu/pytorch-cifar/blob/master/models/resnet.py}
\end{tabular}

\item[\texttt{ResNet18} for \texttt{Imagenet-1k}] \hfill \citep{He2016}~\\[.5em] 
We use both a constant learning rate and a schedule that decays the learning rate by 0.1 every 30 epochs. We run 90 epochs. Note that for \SGDM{} the decaying schedule with initial learning rate of $0.1$ is considered state-of-the-art. As we set $\texttt{dampening}=0.9$, and this is equivalent to $\texttt{dampening}=0$ and a ten times smaller learning rate (see \cref{sec:appendix-info-stability}), in our plots the best score is displayed for initial learning rate of $1$ accordingly.

\vspace{1ex}
\begin{tabular}{@{}ll@{}}
Model & \small \url{pytorch.org/vision/main/models/generated/torchvision.models.resnet18.html}
\end{tabular}

\item[\texttt{DLRM} for \texttt{Criteo}] \hfill \citep{criteo-display-ad-challenge}~\\[.5em] 
\texttt{DLRM} is an industry-scale model with over 300 million parameters. the \texttt{Criteo} dataset contains approximately $46$ million training samples. We run 300k iterations with batch size 128.

\vspace{1ex}
\begin{tabular}{@{}ll@{}}
Dataset & \small \url{https://kaggle.com/c/criteo-display-ad-challenge}\\
Model & \small \url{https://github.com/facebookresearch/dlrm}
\end{tabular}

\item[\texttt{IWSLT14}] \hfill \citep{Ott2019}~\\[.5em] 
We use a transformer with six encoder and decoder blocks from \texttt{fairseq}. The training loss is the cross-entropy loss with label smoothing of $0.1$. We use weight decay of $\lambda=10^{-4}$ (although we noticed that weight decay does not influence the performance of \MomoAdam{}), momentum parameters $(\beta_1,\beta_2)=(0.9,0.98)$. We train for $60$ epochs.

\vspace{1ex}
\begin{tabular}{@{}ll@{}}
Model & \small \url{https://github.com/facebookresearch/fairseq}
\end{tabular}

	\item[\texttt{UNet} for \texttt{Smithsonian Butterflies}] \hfill \cite{Ronneberger2015}~\\[.5em] 
	The \texttt{Smithsonian Butterflies} dataset contains 1,000 images of butterflies. We train a diffusion model, using a \texttt{UNet2D} architecture from the \texttt{Huggingface} library. For both constant learning-rate schedule as well as cosine decay we use a warmup period, where the learning rate is increased linearly over 500 steps from zero to the final value. For \MomoAdam{} we use no weight decay. For \Adam{} we tried both $\lambda \in \{0,0.01\}$ but did not observe major differences; we display the results for $\lambda=0.01$ (the default value).
	We train for $50$ epochs with batch size 16. The training script is adapted from \url{https://colab.research.google.com/github/huggingface/notebooks/blob/main/diffusers_doc/en/pytorch/basic_training.ipynb}.
	
	\vspace{1ex}
	\begin{tabular}{@{}ll@{}}
		Dataset & \small \url{https://huggingface.co/datasets/huggan/smithsonian_butterflies_subset}\\
		Model & \small \url{https://huggingface.co/docs/diffusers/main/en/api/models/unet2d}
	\end{tabular}

	\item[\texttt{ViT} for \texttt{Imagenet-1k}] \hfill \cite{Dosovitskiy2021}~\\[.5em] 
	We train a vision transformer for image classification on the full \texttt{Imagenet-1k} dataset. We use the \texttt{timm} library for training and select the \texttt{vit\_tiny\_patch16\_224} model. 
	We use the settings reported in \cite{Defazio2023}; the only exception is that when increasing the learning rate $\alpha$, we decrease the weight decay $\lambda$ by the same factor, such that $\alpha\cdot\lambda = 10^{-4}$ for all runs.
	By standard practice, we use a warmup period of five epochs (starting at $10^{-5}$ with epoch-wise steps) to a base learning rate $\alpha_{\text{base}}$, followed by a cosine decay.
	We train for $200$ epochs with batch size 512.
	The loss function is the binary cross entropy loss with label smoothing of $0.1$ (also used in \cite{Defazio2023}).
	
	\vspace{1ex}
	\begin{tabular}{@{}ll@{}}
		Model & \small \href{https://github.com/huggingface/pytorch-image-models/blob/ef72c3cd470dd67836eebf95ec567199c890a6a2/timm/models/vision\_transformer.py#L1738}{\texttt{timm/models/vision\_transformer.py}}
	\end{tabular}
\end{description}

For each experiments, we list how long one training run approximately takes on the hardware we use. Unless specified otherwise, we train on a single NVIDIA A100 GPU. \texttt{ResNet110} for \texttt{CIFAR100} $90$ min, \texttt{ResNet20} for \texttt{CIFAR10} $30$ min, \texttt{VGG16} for \texttt{CIFAR10} $30$ min, \texttt{MLP} for \texttt{MNIST} $3$ min, 
\texttt{ViT} for \texttt{Imagenet-1k} $10$ h (on four NVIDIA A100),
\texttt{UNet} for \texttt{Smithsonian Butterflies} $30$ min,
\texttt{ResNet18} for \texttt{Imagenet32} 20 hours (on NVIDIA V100), Transformer for \texttt{IWSLT14} $3$ hours.
%
\subsection{Additional Experiments}

We present additional comparisons in \cref{fig:extendend_results}, including \texttt{AdaBelief} \cite{Zhuang2020}, \texttt{AdaBound} \cite{Luo2019}, and \texttt{Lion}~\citep{Chen2023}. For \texttt{ResNet110} on \texttt{CIFAR100}, we also tried \SGDM{} with a learning-rate schedule that decays by a factor of $0.7$ every 30 epochs. This apparently does not yield improvements.

\begin{figure}[h]
	\centering
	\begin{subfigure}[b]{0.325\textwidth}  
		\includegraphics[width=0.99\textwidth]{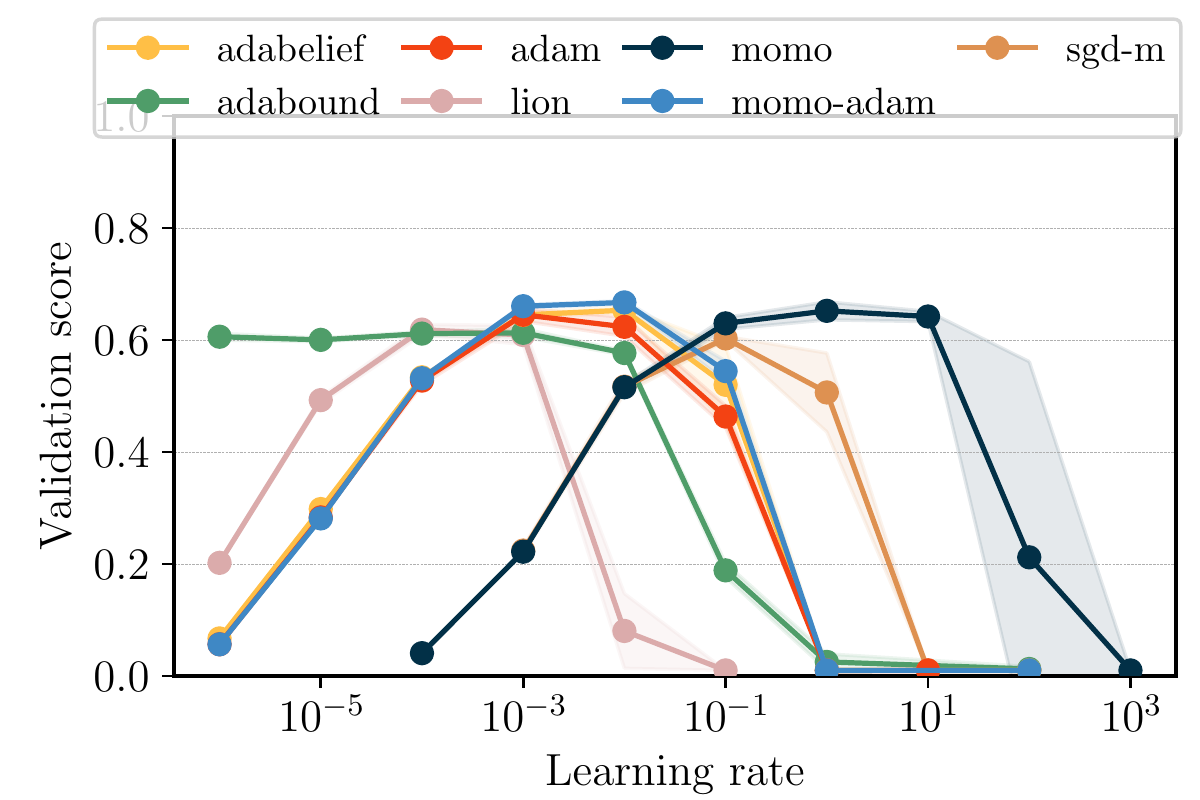}
		\caption{\texttt{ResNet110} for \texttt{CIFAR100}}
		\end{subfigure}
	\begin{subfigure}[b]{0.325\textwidth}  
		\includegraphics[width=0.99\textwidth]{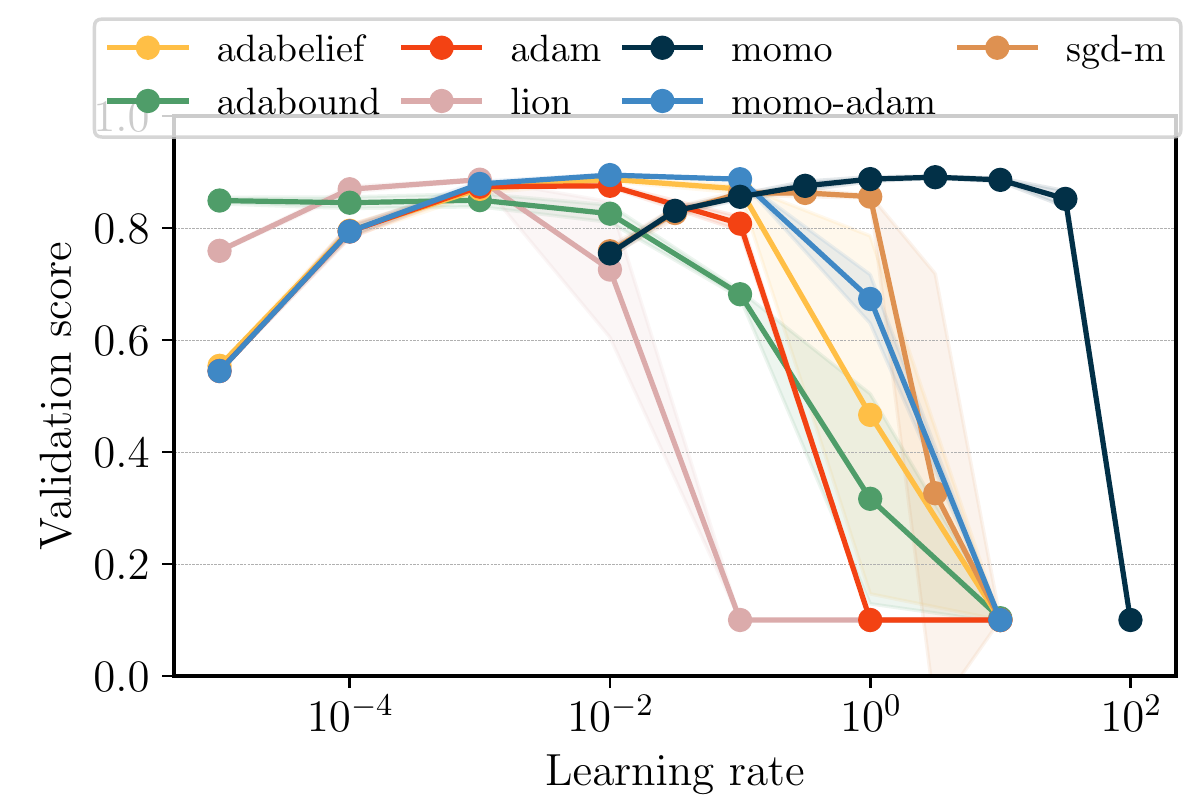}
		\caption{\texttt{ResNet20} for \texttt{CIFAR10}}
	\end{subfigure}
	\begin{subfigure}[b]{0.325\textwidth}  
		\includegraphics[width=0.99\textwidth]{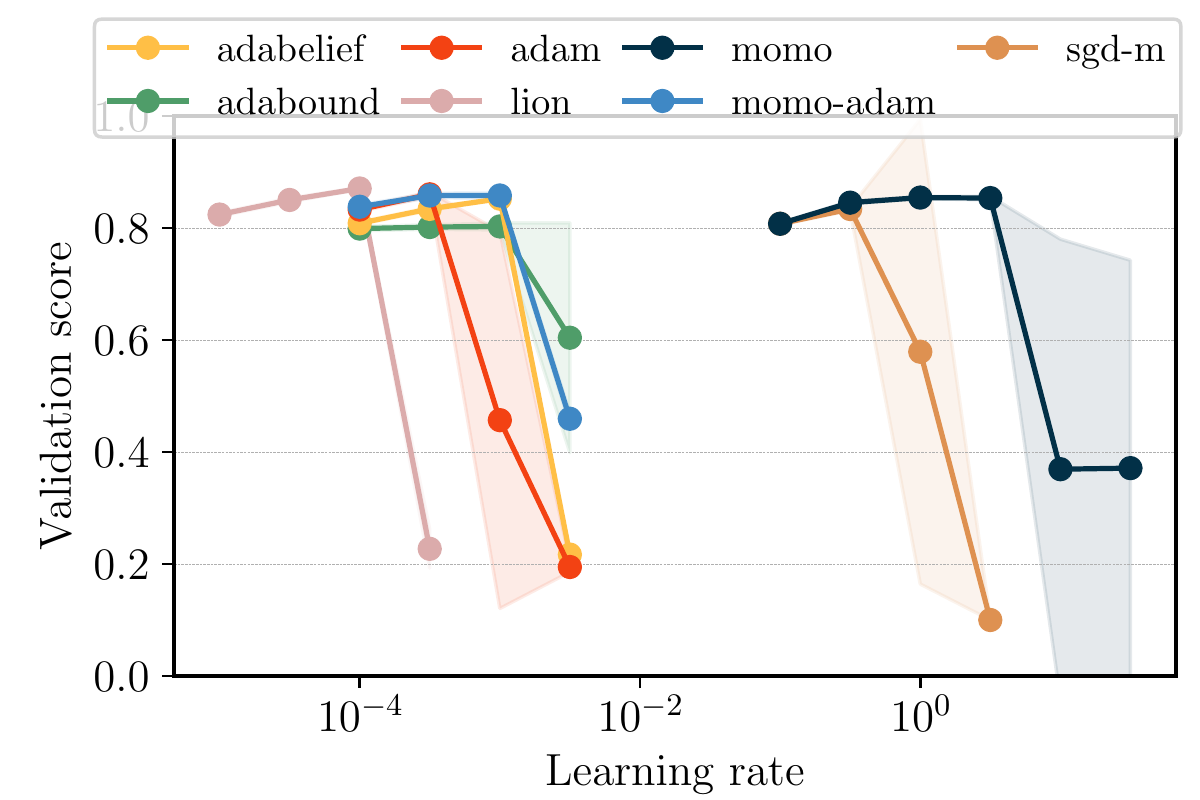}
		\caption{\texttt{ViT} for \texttt{CIFAR10}}
	\end{subfigure}
	\caption{Additional comparisons to \texttt{AdaBelief} and \texttt{AdaBound}. For \texttt{CIFAR100} (left plot), the dashed line of \SGDM{} displays result of using a learning-rate schedule that decays by a factor of $0.7$ every 30 epochs.}
	\label{fig:extendend_results}
\end{figure}

We present the results for the \texttt{UNet} for \texttt{Smithsonian Butterflies} experiment in \cref{fig:unet_smithsonian,fig:unet_smithsonian_samples}. We try a constant learning rate, as well as a cosine-decay schedule (both schedules with warmup). The cosine decay works better in general.
\begin{figure}[h]
	\centering
	\begin{subfigure}[b]{0.5\textwidth} 
		\centering
		\includegraphics[width=0.99\textwidth]{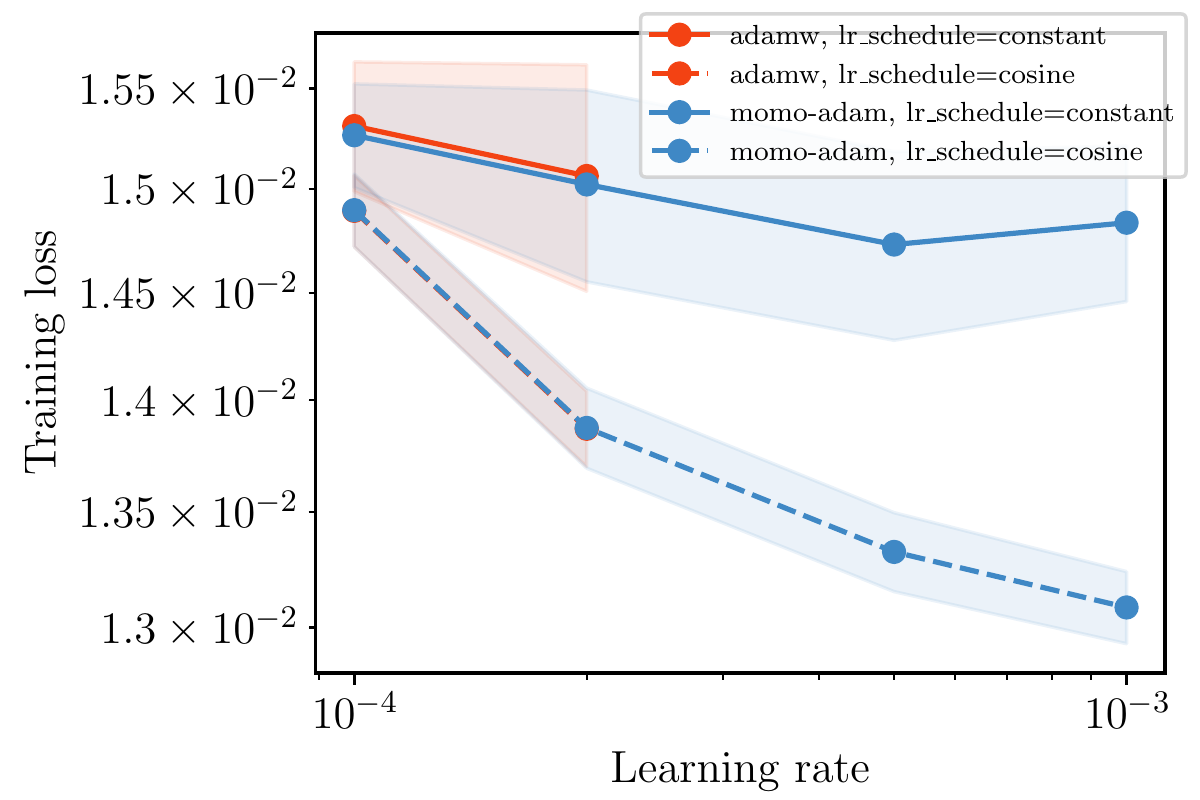}
		\label{momo:fig:unet_smithsonian_stability}
	\end{subfigure}
	\begin{subfigure}[b]{0.4\textwidth} 
		\centering
		\includegraphics[width=0.99\textwidth]{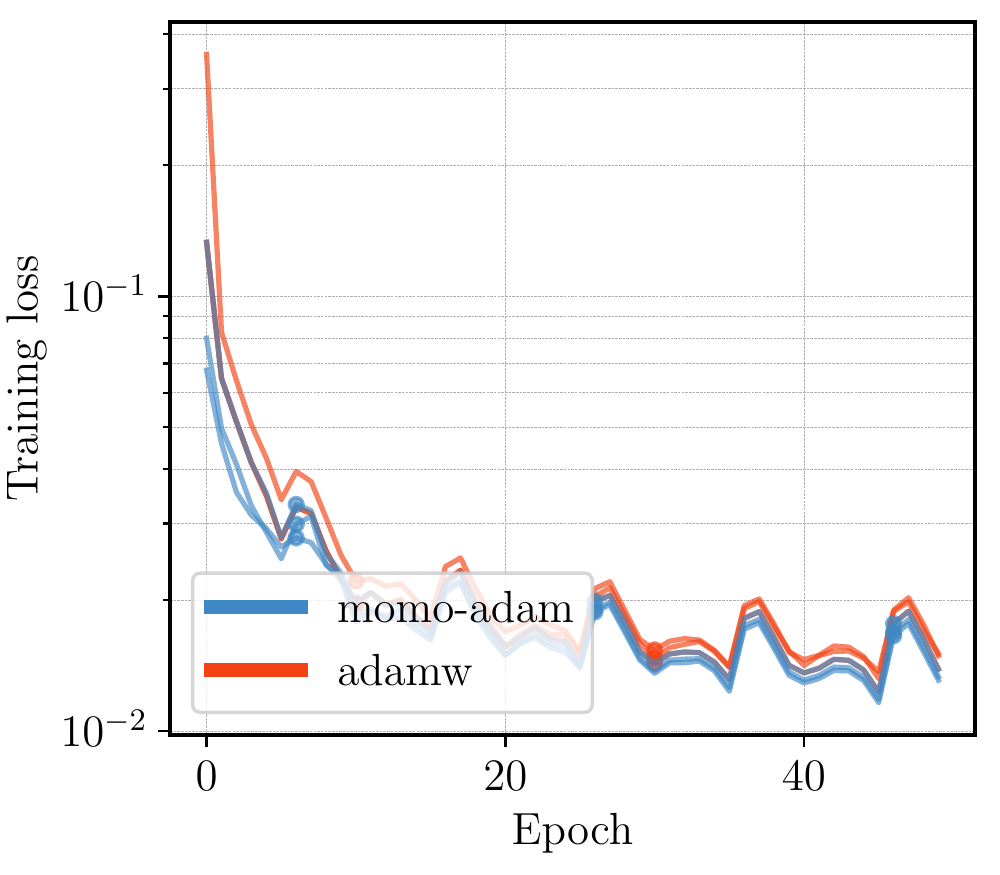}
		\label{fig:unet_smithsonian_training}
	\end{subfigure}
	\caption{\texttt{UNet} for \texttt{Smithsonian Butterflies} Left: Stability over learning rate, for constant and cosine decay schedule (more details on schedule settings in \cref{sec:appendix-models-datasets}). Values are averaged over three repetitions with different seeds, shaded area depicts one standard deviation.
	Missing points for \Adam{} mean that at least one of the three repetitions results in \texttt{NaN} loss. Right: Training loss curve for the best three settings (across all base learning rates and both schedules) for each method.}
	\label{fig:unet_smithsonian}
\end{figure}
\begin{figure}
	\centering
	\begin{subfigure}[b]{0.3\textwidth} 
		\centering
		\includegraphics[width=0.99\textwidth]{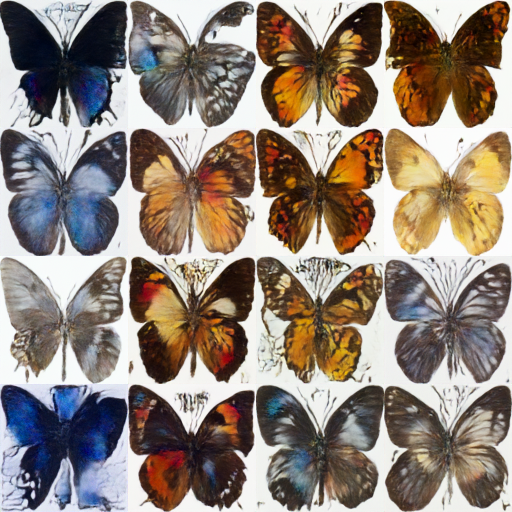}
		\caption{$\alpha_\text{base} = 0.0001$}
	\end{subfigure}
	\begin{subfigure}[b]{0.3\textwidth} 
		\centering
		\includegraphics[width=0.99\textwidth]{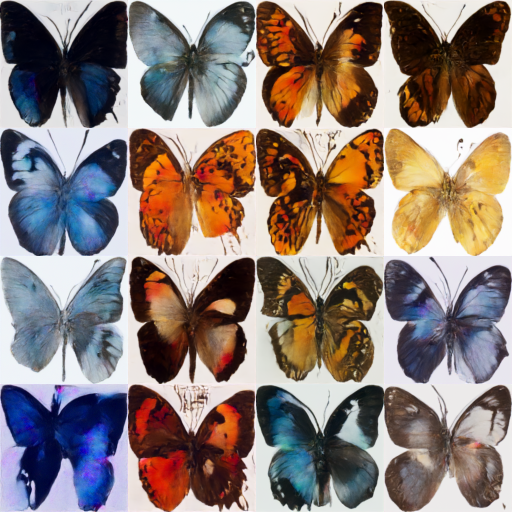}
		\caption{$\alpha_\text{base} = 0.0005$}
	\end{subfigure}
	\begin{subfigure}[b]{0.3\textwidth} 
		\centering
		\includegraphics[width=0.99\textwidth]{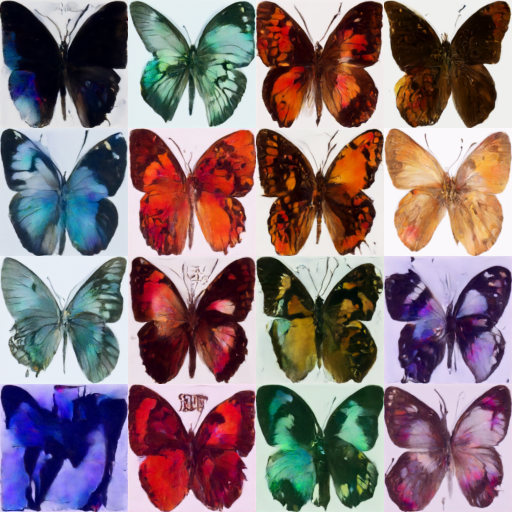}
		\caption{$\alpha_\text{base} = 0.001$}
	\end{subfigure}
	\caption{\texttt{UNet} for \texttt{Smithsonian Butterflies}: Generated images of the \texttt{UNet} diffusion model at the end of training with \MomoAdam{}. We display three different base learning rates for the cosine decay schedule (i.e.\ the displayed value of $\alpha_\text{base}$ corresponds to the $x$-axis in the left plot of \cref{fig:unet_smithsonian}). Note that when training with \Adam{}, the images in (a) look very similar, but for (b) and (c),  \Adam{} diverges and thus the model generates no useful images.}
	\label{fig:unet_smithsonian_samples}
\end{figure}
\clearpage
\subsection{Illustrative Example of Online Lower Bound Estimation}
\label{sec:appendix-simple-fstar-example}
We show how our online estimation of $\lb{k}$, derived in \cref{sec:lower} and \cref{lem:fstar}, work for a simple example. Consider a regression problem, with synthetic matrix $A\in\R^{200\times 10}$ and $b\in \R^{200}$. 
We solve the problem $\min_{x\in\R^{10}} \sum_{i=1}^{200} \frac12 \|a_i^\top x-b_i\|^2$, where $a_i$ are the rows of $A$.
The data is generated in a way such that there exists $\hat x$ with $b=A\hat x$ and hence the optimal value is $f^*=0$.

We now run \Momo{}(-\Adam{}) with lower bound estimate $\lb{k}=-10$ in all iterations, and  \Momo{}(-\Adam{})$^*$ with initialization $\lb{1}=-10$. Clearly, this is not a tight estimate of the optimal value $f^*$. From \cref{fig:fstar-synthetic-example-stability}, we see that online estimation of $\lb{k}$, used in \Momo{}(-\Adam{})$^*$, improves stability of the training compared to plain \Momo{}(-\Adam{}) where a constant value $\lb{k}=-10$ is used. From \cref{fig:fstar-synthetic-example-fstar}, we also see that the online values of $\lb{k}$ converge to $f^*=0$. 


%
\begin{figure}
    \centering
    \begin{subfigure}[b]{0.55\textwidth}  
        \centering 
        \includegraphics[width=0.99\textwidth]{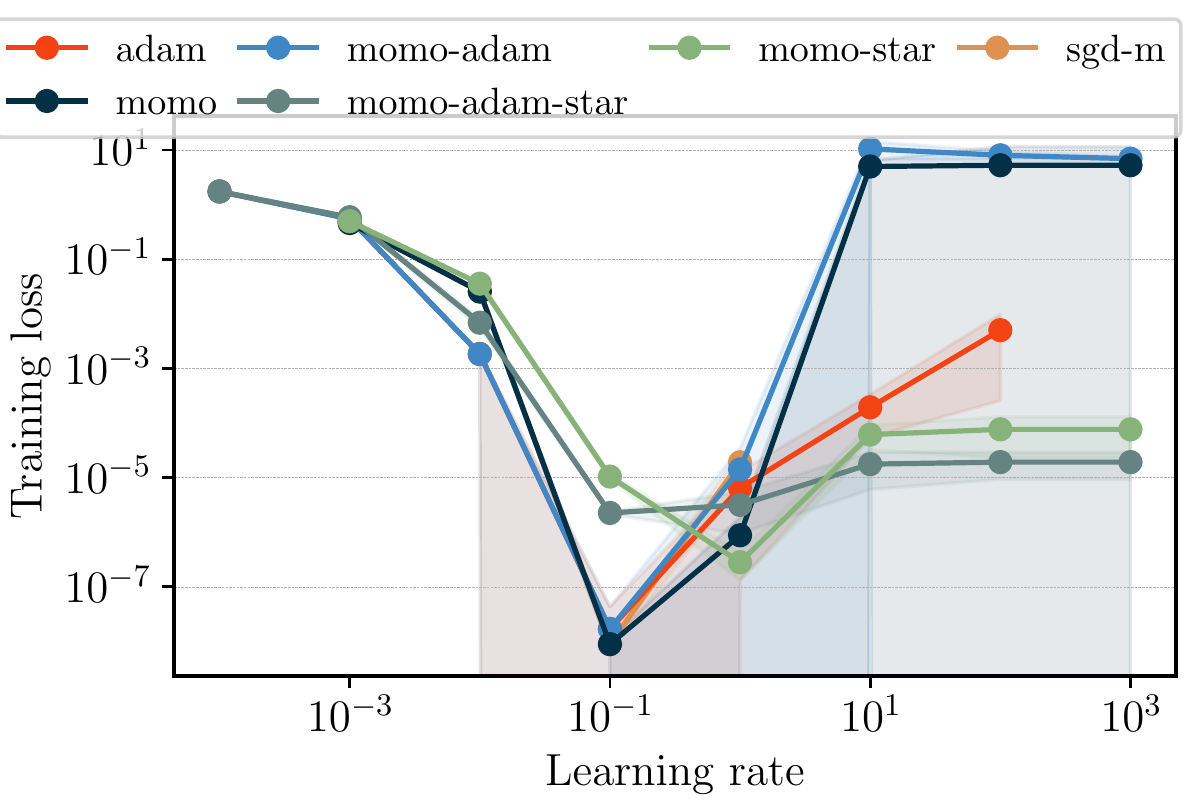}
        \caption{Training Loss}
        \label{fig:fstar-synthetic-example-stability}
    \end{subfigure}
    \begin{subfigure}[b]{0.44\textwidth}  
        \centering 
        \includegraphics[width=0.99\textwidth]{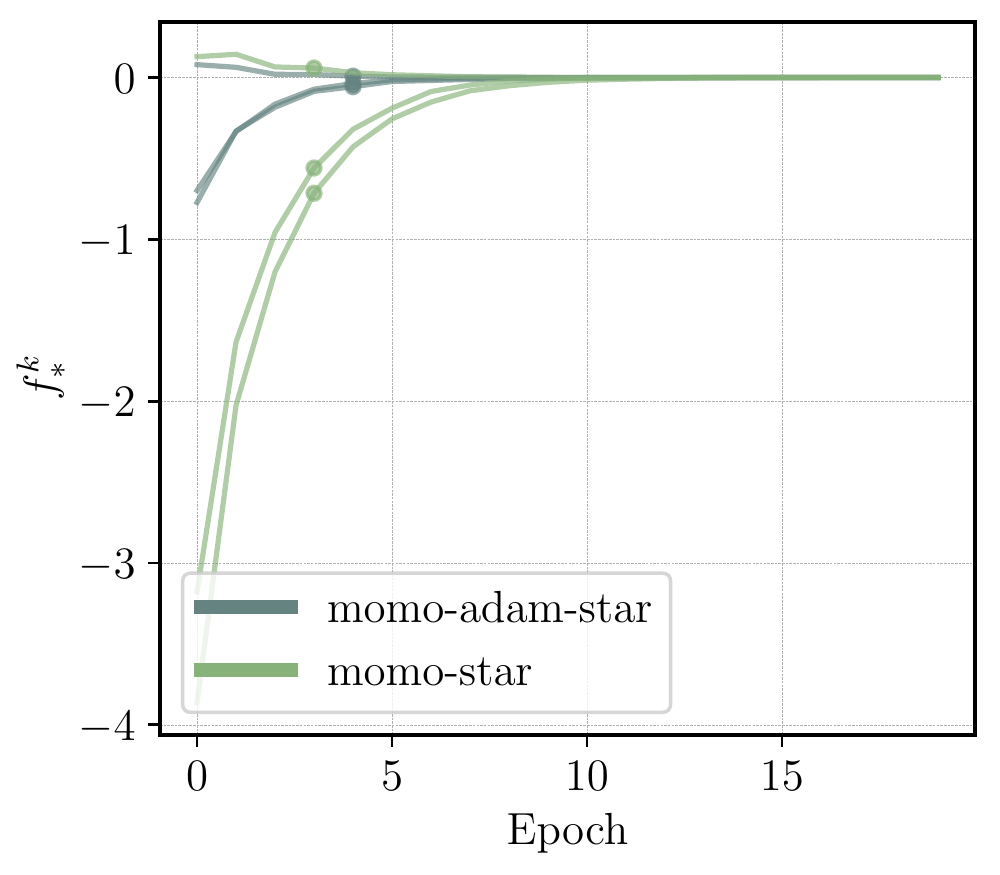}
        \caption{Online estimation of $f^*=0$}
        \label{fig:fstar-synthetic-example-fstar}
    \end{subfigure}
\caption{Illustrative example of online lower bound estimation. For all \Momo{} methods, we initialize $\lb{1}=-10$. Left: Training loss for varying (constant) learning rate $\alpha_0$. Right: Value of $\lb{k}$ over training, one line corresponds to one choice of $\alpha_0$. We plot per method the four  values of $\alpha_0$ that lead to smallest training loss.}
\label{fig:fstar-synthetic-example}
\end{figure}
